%% file: neurips_2026.tex
\documentclass{article}

\usepackage[main, final]{neurips_2026}
\usepackage{paralist}
\usepackage{enumitem}

\usepackage[utf8]{inputenc} 
\usepackage[T1]{fontenc}    
\usepackage{hyperref}       
\usepackage{url}            
\usepackage{booktabs}       
\usepackage{amsfonts}       
\usepackage{nicefrac}       
\usepackage{microtype}      
\usepackage{xcolor}         

\usepackage{amsmath,amssymb,amsthm}
\usepackage{mathtools}
\usepackage{bbm}
\usepackage{enumitem}
\usepackage{geometry}
\usepackage{hyperref}
\usepackage{xcolor}
\usepackage{soul}
\usepackage{booktabs}

\newcommand{\Zbb}{\mathbb{Z}}
\newcommand{\Rbb}{\mathbb{R}}

\newcommand{\xhdr}[1]{\noindent\textbf{#1}}

\theoremstyle{definition}

\newtheorem{theorem}{Theorem}[section]
\newtheorem*{theoremm}{Theorem}
\newtheorem{proposition}[theorem]{Proposition}
\newtheorem{definition}[theorem]{Definition}
\newtheorem{remark}[theorem]{Remark}

\newtheorem{lemma}[theorem]{Lemma}

\newtheorem{example}[theorem]{Example}
\newtheorem{corollary}[theorem]{Corollary}

\usepackage[ruled,vlined]{algorithm2e}

\usepackage{quiver}
\usepackage{amsmath,amssymb,amsfonts}

\usepackage{multirow}

\usepackage{colortbl}
\usepackage{tikz}
\usepackage{tcolorbox}
\usepackage{tabularx}
\newcolumntype{P}[1]{>{\centering\arraybackslash}p{#1}}
\usetikzlibrary{positioning, quotes}

\definecolor{lb}{RGB}{31,119,180}

\newtcolorbox{mybox}[1]{colback=lb!5!white,colframe=lb!70!black,fonttitle=\bfseries,title=#1}
\usepackage{pifont}

\usepackage{empheq}
\newtcbox{\mymath}[1][]{%
    nobeforeafter, tcbox raise base, colframe=teal!60!black,
    colback=teal!10, boxrule=1pt,
    #1}

\usepackage{subfig}

\tcbuselibrary{skins}
\tcbset{tab2/.style={enhanced,fonttitle=\bfseries,
colback=white,colframe=gray,colbacktitle=lb,
coltitle=black,center title}}

\tcbset{tab2/.style={enhanced,fonttitle=\bfseries,
colback=white,colframe=gray,colbacktitle=white,
coltitle=black,center title}}

\newtcolorbox{resbox}{standard jigsaw,opacityback=0,boxrule=1pt,top=0.9ex,bottom=0.9ex,colbacktitle=teal!10!white,colframe=teal!70!white}

\usepackage{adjustbox}
\usepackage{wrapfig}

\def\ztitle{Exact Topological Compliance:\\ Generating Persistence-Equivalent Graphs}
\title{\ztitle}

\author{
  Mattie Ji$^*$ \\
  University of Pennsylvania \\
  \texttt{mji13@sas.upenn.edu} \\
   \And
  Indradyumna Roy\thanks{Equal Contributions.} \\
  Aalto University\\
  \texttt{indradyumna.roy@aalto.fi}\\
  \And
  Vikas Garg \\
YaiYai Ltd $\&$ Aalto University  \\
\texttt{vgarg@csail.mit.edu}
}

\begin{document}

\maketitle

\begin{abstract}
Persistent homology summarizes the topological evolution of a graph as
persistence diagrams (PDs). Generating graphs whose topology matches a
target descriptor is increasingly important, yet existing topology-aware
generative models rely on soft regularization, producing graphs whose PDs may only be approximately similar.
We initiate a principled investigation of the generation problem: produce graphs that realize a target PD exactly. We propose two approaches guided by the implicit local and global constraints a target PD imposes. The local approach builds degree-based PD-equivalent graphs iteratively, and we prove that every output realizes the target PD and that every degree-based PD-equivalent graph is reachable. However, as shown thereafter,  local generation necessarily trades off between sample diversity and reaching infeasible states. We therefore recast generation as a global constraint satisfaction problem with a complete encoding solvable by modern constraint programming solvers. This is applicable to any vertex-based permutation equivariant filtration scheme and is minimal under the  degree filtration. Experiments on four standard graph generation benchmarks confirm that both methods realize the target PD exactly. Together, our methods provide a principled treatment of strict topology-compliant graph generation. Code is available at \url{https://github.com/Aalto-QuML/ExactCompliance}.
\end{abstract}

\section{Introduction}

Topological structure is a primary driver of function in many real-world graphs. The presence of components, cycles, and higher-order voids govern stability and catalytic activity in molecules~\cite{xia2015persistent, townsend2020representation, wang2025inverse}, binding affinity and virality~\cite{wang2021analysis, wei2022topological}, information flow in brain networks~\cite{sizemore2018cliques}, certain patterns in medical imaging \citep{topocl2026}, and structural risk in transportation, trade, and infrastructure networks~\cite{woldemariam2025topological, li2021hybrid, villines2026topological}.
Persistent homology~\cite{carlsson2009topology, edelsbrunner2012persistent} is the standard tool for extracting this information. 
Given a graph $G$ and a filtration scheme, persistent homology tracks the births and deaths of homological features along the induced subgraph sequence. The result is a multiset of pairs $(b,d)$ called the persistence diagram (PD). PDs are stable under perturbation
\citep{Cohen-Steiner_Edelsbrunner_Harer_2006}, computable with efficient algorithms
\citep{Otter_2017}, and have become the workhorse of topological data analysis (TDA) in machine learning,
underpinning almost every topology-aware approach to graph and higher-order neural networks~\citep{pmlr-v119-hofer20b, immonen2023going, ballester2024expressivitypersistenthomologygraph, horn2021topological, ji2025graphpersistencegoesspectral, ji2026contraction, yan2022neural, pmlr-v108-carriere20a, russold2024graphcode, verma2024topological, tdl_position} and providing topological signals that standard message passing alone cannot recover (e.g., cycle detection).

In parallel, graph generation has recently emerged as a central tool across these same
domains, with autoregressive \citep{you2018graphrnn}, one-shot
\citep{simonovsky2018graphvae}, and diffusion-based
\citep{niu2020permutation, jo2022score, vignac2022digress}
models. Recent
work conditions generation on domain-specific structural properties:
spectral profile \citep{martinkus2022spectre}, molecular validity
\citep{de2018molgan, jin2020hierarchical}, and structural
constraints like planarity and acyclicity \citep{madeira2024generative}. 
Conditioning on \textit{topology} is a natural next step, but largely unexplored. A recent attempt~\cite{parktopology}, augments discrete diffusion with a
$1$-Wasserstein PD-matching loss.
This soft regularizer nudges the learned distribution toward topological plausibility, but offers no guarantees that individually generated graphs attain the desired topological structure (e.g., the target number of cycles and components).


\begin{table*}[t!]
\caption{Overview of the core contributions in our paper.}
    \vspace{-8pt}
    \input{summary}
    \label{fig:overview}
    \vspace{-8pt}
\end{table*}

\xhdr{PD-conditioned graph generation from first principles.}
A PD is a finite summary of an underlying combinatorial structure and therefore collapses many graphs to the same image. 
The most widely studied filtration scheme in the graph learning literature is the vertex-degree function~\cite{ballester2024expressivitypersistenthomologygraph, horn2021topological, parktopology,yan2022neural, pmlr-v119-hofer20b, 10.5555/3294771.3294927}: it is permutation-equivariant and aligns naturally with the local message-passing inductive bias of GNNs. We adopt it as the primary filtration scheme throughout this paper and discuss extensions to other vertex-based, permutation-equivariant schemes (centrality \cite{liang2025centralityguided}, square clustering coefficients \cite{ji2025on}, triangle count \cite{10.5555/3666122.3668874}, etc.) where they apply.

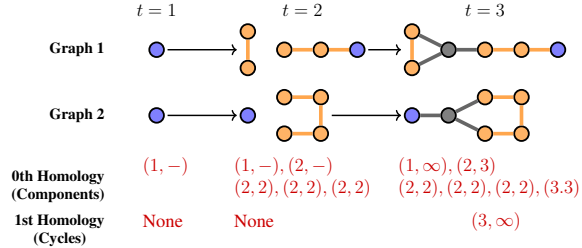
\begin{wrapfigure}{r}{0.55\textwidth}
  \centering
  \vspace{-10 pt}
  \begin{adjustbox}{width=0.55\textwidth}
  \input{Figures/intro_example}
  \end{adjustbox}
  \caption{Two non-isomorphic graphs with identical degree-based persistence. 
  Vertex color encodes degree; 
  edges inherit the color of their max-degree endpoint.}
  \label{fig:intro-ex}
  \vspace{-10pt}
\end{wrapfigure}
Even under this canonical choice, a PD is far from a complete descriptor. Figure~\ref{fig:intro-ex} shows two non-isomorphic graphs with identical degree-based PDs; in general, graphs with differing  structures and connectivity patterns, can share the same PD under a given filtration scheme. 
Each PD therefore
identifies  a \emph{class of compatible graphs}, and conditioning on a PD amounts to sampling from this class.  This raises the natural
question: \emph{given an input PD and filtration scheme $\Phi$, can we generate strictly PD-compatible graphs under $\Phi$?}

This motivates a precise criterion for any PD-conditioned generator. 
Given a filtration scheme $\Phi$ and a graph $G$, we let $\mathcal{P}(G^{\Phi}_{\bullet})$ denote the PDs associated to the induced subgraph sequence of $G$ under $\Phi$. 
We call a generator \textbf{\emph{strictly PD-compliant}} if it satisfies
\begin{enumerate}[leftmargin=*, topsep=0pt, itemsep=0pt]
\item \textbf{Soundness.} Every sample $G$ from the generator
satisfies $\mathcal{P}(G^{\Phi}_{\bullet}) = D$ exactly.
\item \textbf{Completeness.} Every PD-compatible graph can be created by the generator.
\end{enumerate}

Soft-conditioned models violate soundness by construction. Aiming for strict
PD-compliance is the natural starting point: it makes the conditioning
problem precise, and exposes what an input PD does and does not pin down. Once strict compliance is attained, it can be relaxed in applications.

\xhdr{Why exact compliance?}
Exact generation produces families of graphs with identical persistence diagrams but different combinatorial properties. Such families provide controlled datasets and benchmarks for training and evaluation in topological learning, and one can test whether a model truly exploits the prescribed topology or leans on unrelated statistics. They also give clean training and evaluation data for graph generation and reasoning from topological descriptors, with no label noise from approximate PD matching.

Exact generation can also help find counterexamples to conjectures about what properties are determined by persistence, and approximate generation cannot serve this purpose at all. Exactness is also important when applications critically depend on the exact number of topological features, such as ring counts in molecular generation~\citep{ritchie2009impact} and coverage holes in sensor networks~\citep{de2007coverage}.

\xhdr{Contributions.} We initiate a principled investigation of strict PD-compliant graph generation; Table~\ref{fig:overview} summarizes the formal statements. Our contributions are threefold.

\textbf{(1) Local iterative sampler (Section~\ref{sec::local}).} Local decisions are a prerequisite for any globally correct generation and rely only on the neighborhood-level information standard GNN pipelines already provide, making them the natural starting point. We give a sampler that is sound and complete on general, vertex-based, and degree filtrations, but the sampler can run into infeasible states. We subsequently prove an impossibility result: any local approach must trade off between locally PD-compliant choices and reaching globally infeasible states.

\textbf{(2) Global generation via constraint programming (Section~\ref{sec:global}).} The impossibility motivates us to characterize the global constraints. We recast generation as a constraint satisfaction problem applicable to any vertex-based permutation-equivariant filtration scheme, and give concrete encodings for degree, triangle and clique counts, clustering, closeness, Katz, and PageRank. We prove our method is strictly PD-compliant, and is a minimal encoding for degree filtrations.

\textbf{(3) Empirical validation (Section~\ref{sec::experiments}).} On \textsc{Comm-20}, \textsc{Planar}, \textsc{SBM}, and \textsc{ENZYMES}, both methods realise the target degree-PD exactly, while neural diffusion baselines with PD-conditioning regularisers leave a structural approximation gap that does not close with more training.

\textbf{Related Works.} Persistence-based reconstructions have been studied for embedded graphs, geometric shapes, and metric-measure spaces \citep{chazal2008towards, Chazal2011, 7789622, BELTON2020101658, DBLP:conf/cccg/SchenfischFMMW22,https://doi.org/10.1111/sapm.12526,hacquard2026statistical,ophelders2025sweepingorderssimplicialcomplex}. These approaches exploit extrinsic information such as coordinates, directions, distances, or sampling measures to recover or analyze a geometric realization. In contrast, we consider abstract graphs without prescribed embeddings and characterize the set of graphs realizing specified persistence diagrams under an intrinsic, permutation-equivariant filtration. Our framework therefore complements geometric reconstruction methods by studying graphs sharing a common intrinsic PH signature. We only consider graph filtrations, as this is quite common for the usage of persistent homology in the GNN literature \citep{pmlr-v119-hofer20b, horn2021topological, immonen2023going, pmlr-v108-zhao20d, 10.5555/3666122.3668874}.

To compare the two, the persistent homology transform (PHT) \cite{turner2014fmht} is injective \cite{Ghrist2018,CurryMukherjeeTurner2022} on embedded graphs, but it does not resolve the graph isomorphism problem, as two different embeddings of the same abstract graph can have different PHTs. Our setting contains no coordinates, and we want to reconstruct abstract graphs with equivalent PDs, with the goal of benchmarking for graph isomorphisms. Another deviation from typical TDA considerations is that the GNN literature typically requires permutation equivariance for their filtrations. This imposes a strong rigidity on filtrations that typically do not show up for embedded graphs. For example, any permutation equivariant function on a cycle graph has to be constant, but certainly we can look at the height filtration on embedded cycles.

\section{Background and Preliminaries}

We assume a graph $G$ with vertex set $V(G)$ and edge set $E(G)$ is finite, simple, and undirected.
\begin{definition}
Let $G$ be a graph, a (general) \textbf{filtration} of $G$ is a sequence of subgraphs of $G$ as
\[G_{\bullet}: \emptyset = G_{-1} \subseteq G_{0} \subseteq ... \subseteq G_n = G.\]
Note here we do allow the $G_i = G_{i+1}$ for some $i$'s. We say a filtration of $G$ is \textbf{vertex-based} if each $G_i$ is an induced subgraph of $G$ on its vertex set $V(G_i) \subseteq V(G)$. We say a vertex based filtration of $G$ is \textbf{degree-based} if $V(G_i) = \{v \in V(G)\ |\ \deg(v) \leq i\}$ and $n$ is the maximum degree in $G$. We write the degree-based filtration as $G^{\operatorname{deg}}_{\bullet}$.
\end{definition}

Given any filtration $G_{\bullet}$ of $G$, there are lists of persistence pairs tracks the evolution of the births and deaths of its topological features on its (0-dimensional) vertex-features and (1-dimensional) cycle-features. The lists are constructed as follows.
\begin{enumerate}[noitemsep,topsep=0pt, leftmargin=*]
    \item For each $b \geq 0$, we can view $G_{b-1} \subseteq G_{b}$ as a sequence of adding vertices first and then adding edges to $G_{b}$. Any vertex $v$ added is said to be \textbf{born at time} $b$ and is thought of as a vertex representing the connected component that is $\{v\}$.
    \item Whenever an edge is added between two components with vertex representatives $v$ and $w$ respectively in the stage $G_{d-1} \subseteq G_d$. We kill off whichever vertex in $v$ and $w$ that is born later,  assign its \textbf{death time} to $d$, and assign the other vertex as the representative of the new merged component. If $v$ and $w$ were born at the same time, we simply choose one of them to kill off. A vertex representing a component alive at $G_n = G$ is assigned to die at time $+\infty$.
    \item Whenever an edge is added from the same component to itself in the stage $G_{b-1} \subseteq G_b$, this is an \textbf{independent cycle}, and we assign a tuple of the form $(b, \infty)$ to it.
\end{enumerate}

\begin{definition}\label{def::persistence_pairs}
Let $G_{\bullet}$ be a filtration of $G$, its \textbf{persistence diagram} $\mathcal{P}(G_{\bullet})$ is a pair of two \textbf{unlabeled multisets} $\mathcal{P}_0(G_{\bullet})$ and $\mathcal{P}_1(G_{\bullet})$. Here, $\mathcal{P}_0(G_{\bullet})$, the 0-dimensional persistence diagram, is the list of birth and death time $(b, d)$ for each vertex $v \in V(G)$. $\mathcal{P}_1(G_{\bullet})$, the 1-dimensional persistence diagram, is the list of cycle birth times $(b, \infty)$ for each independent cycle in $G$.
\end{definition}

Note the lists $\mathcal{P}_0(G_{\bullet}), \mathcal{P}_1(G_{\bullet})$ are independent of the choices made in the construction, and the final output does not identify persistence pairs with simplices of the graph. Definition~\ref{def::persistence_pairs} is a special case of taking persistence diagrams on simplicial complexes (see \cite{Dey_Wang_2022}). Here we restricted to the case of filtrations on graphs where the only non-empty PDs are in degree $0$ and $1$.

\begin{remark}\label{rmk::filtration} Some sources \cite{ ballester2024expressivitypersistenthomologygraph, immonen2023going} alternatively define a filtration of $G$ to come from a function $f: V(G) \cup E(G) \to \Rbb$ such that for any edge $e = (v, w)$, $f(e) \geq \max(f(v), f(w))$. The persistence pairs of any $f$ is equivalent to that of Definition~\ref{def::persistence_pairs} by renormalizing $f$ to land in $\Zbb_{\geq 0}$. This is possible as $f$ can only take finitely many values, and the expressive power of persistence pairs only depend on the relative ordering of its values on the simplicies. Being vertex-based is equivalent to having $f(e) = \max(f(v), f(w))$. Note the degree function is vertex-based and already normalized.
\end{remark}
In general the persistence pairs depend on the filtration. When such filtration is canonically defined (e.g., degree-based), it is a graph isomorphism invariant. Here, ``canonically defined" refers to any permutation equivariant filtration scheme $\Phi$ (see Section 3 of \cite{ballester2024expressivitypersistenthomologygraph}). For our purposes, we define:

\begin{definition}\label{def::filtration_scheme}
A \textbf{filtration scheme} $\Phi$ assigns every graph $G$ a filtration function $\Phi_{G}: V(G) \cup E(G) \to \Rbb$. $\Phi$ is \textbf{permutation equivariant} if for any graph isomorphism $\psi: G \to H$, $\phi_{G} = \phi_{H} \circ \psi$. For any such $\Phi$ valued in $\Zbb_{\geq 0}$, we write $H^{\Phi}_{\bullet}$ to denote the sequence of subgraphs associated to $\Phi$.
\end{definition}

\begin{proposition}[Proposition 1 of \cite{ballester2024expressivitypersistenthomologygraph}]
A permutation equivariant filtration scheme gives the same lists of persistence diagrams on isomorphic graphs.
\end{proposition}

In this work, we are interested in generating graph filtrations whose persistence pairs match a given PH signature. For convenience, we will use the following two terminologies.
\begin{definition}
We say two filtrations $G_{\bullet}$ and $H_{\bullet}$ are \textbf{PD-equivalent} (persistence-diagram-equivalent) if $\mathcal{P}(G_{\bullet}) = \mathcal{P}(H_{\bullet})$. We say $G$ and $H$ are \textbf{degree-PD-equivalent} if $\mathcal{P}(G_{\bullet}^{\operatorname{deg}}) = \mathcal{P}(H_{\bullet}^{\operatorname{deg}})$. The \textbf{$k$-th truncation} $\mathcal{P}^k(G_{\bullet})$ of a filtration $G_{\bullet}$ is $\mathcal{P}(G_{\bullet}^k \coloneqq G_{-1} \subseteq G_{0} \subseteq ... \subseteq G_k)$.
\end{definition}

\section{Local Generation}\label{sec::local}

In this section, we describe a local generation process for producing PD-equivalent random filtrations. We then modify the process to adapt to vertex-based and degree-based filtrations. The methods here take a \textbf{local approach} to generation by iteratively generating a sequence of subgraphs $H_{\bullet}$ such that their $k$-th truncations (1) $\mathcal{P}^k(H_{\bullet}) = \mathcal{P}^k(G_{\bullet})$ agree for each $k \in \{-1, 0, ..., n\}$ and (2) the construction of $H_k$ should only depend on $H_{k-1}$ and PDs present in $G^{k}_{\bullet}$. The algorithms can generate all possible PD-equivalent filtrations if they terminate, but they can run into infeasible states that will be explained later in this section.

\subsection{Producing PD-Equivalent General Filtrations}\label{sec::general_fil}
Here, we describe a local generation process to iteratively create a general filtration $H_{\bullet}$ such that $\mathcal{P}(H_{\bullet}) = \mathcal{P}(G_{\bullet})$. The exact details are outlined in Algorithm~\ref{alg:pd_to_graph}  (see Appendix~\ref{app:algorithms}). Since we aim to make $H_{\bullet}$ such that its truncations agrees with that of the hidden input $G_{\bullet}$, we describe a procedure to inductively build $H_{k}$ from a given $H_{k-1}$ whose previous truncations agree with that of $G_{\bullet}$.
\begin{enumerate}[noitemsep,topsep=0pt, leftmargin=*]
    \item (\textbf{Vertex Addition}): For each $(k, d) \in \mathcal{P}_0(G_{\bullet})$, add a new vertex to $H_k$ with birth time $k$.
    \item (\textbf{Component Merging}): We traverse through the list of tuples in $\mathcal{P}_0(G)$ that die at time $k$ in \textbf{random order}. For each tuple $(b, k)$, we pick a \textbf{random component} $X$ with birth time $b$ and a \textbf{random component} $Y \neq X$ with birth time $\leq b$ and add an edge between a random pair of nodes in $X \times Y$. We then update the component birth times appropriately.
    \item (\textbf{Cycle Addition}): For each $(k, \infty) \in \mathcal{P}_1(G_{\bullet})$, add an edge between a pair $(p, q) \in X \times X$ for any component $X$ such that $p, q$ does not already share an edge.
\end{enumerate}
Note here it is important to traverse in random order during the component merging step. Otherwise, there are PD-equivalent filtrations to $G_{\bullet}$ that Algorithm~\ref{alg:pd_to_graph} could not make.

\begin{resbox}
\begin{theorem}\label{thm::general_filtration}
   Given $(\mathcal{P}_0, \mathcal{P}_1)$ the persistence pairs coming from a hidden filtration $G_{\bullet}$. If Algorithm~\ref{alg:pd_to_graph} terminates, the filtration $H_{\bullet}$ is PD-equivalent to $G_{\bullet}$. Furthermore, any filtration $H_{\bullet}'$ that is PD-equivalent to $G_{\bullet}$ can be created by Algorithm~\ref{alg:pd_to_graph}.
\end{theorem}
\end{resbox}

\begin{wrapfigure}{r}{0.45\textwidth}
\vspace{-20pt}
  \centering
  \begin{adjustbox}{width=0.4\textwidth}
  \hspace{-5pt}\input{Figures/random_cycle_capacity}
  \end{adjustbox}
  \caption{An example of how Algorithm~\ref{alg:pd_to_graph} can run out of cycles in the cycle adding stage.}\label{fig::cycle_capacity}
\vspace{-15pt}
\end{wrapfigure}

The statement of Theorem~\ref{thm::general_filtration} indicates that Algorithm~\ref{alg:pd_to_graph} may not be able to terminate sometimes. It turns out Algorithm~\ref{alg:pd_to_graph} may run out of cycles to add in the cycle addition stage.

\begin{example}\label{exp::no_more_cycle}
Let $G_{\bullet}$ be the hidden filtration given in Figure~\ref{fig::cycle_capacity}. It is perfectly possible for Algorithm~\ref{alg:pd_to_graph} to create $H_0 \subseteq H_1$, as $\mathcal{P}(\emptyset \subseteq H_0) = \mathcal{P}^0(G_{\bullet})$ and $\mathcal{P}(\emptyset \subseteq H_0 \subseteq H_1) = \mathcal{P}^1(G_{\bullet})$. However, when going from $H_1$ to $H_2$, one cannot add a single edge to run out of cycles.
\end{example}

We remark that cycle addition would not be an obstruction if we allow multi-edges between nodes. For convenience, we derive a simple criterion for when cycle addition would succeed. For a graph $B$, we define its \textbf{cycle capacity} as $\operatorname{cap}^{\operatorname{cyc}}(B) \coloneqq \sum_{A, \text{ component of }B} \frac{|V(A)|(|V(A)|-1)}{2} - |E(A)|$.

\begin{proposition}\label{prop::cycle_capacity}
    Suppose $H_{k-1}$ has been iteratively built. Algorithm~\ref{alg:pd_to_graph} can always succeed in the vertex addition and component merging step. Furthermore, if the number of the tuples $(k, \infty)$ is bounded by $\operatorname{cap}^{\operatorname{cyc}}(H_{k-1})$, then Algorithm~\ref{alg:pd_to_graph} can succeed in the cycle addition step. 
\end{proposition}

In particular, $\operatorname{cap}^{\operatorname{cyc}}(B)$ is always non-negative, so we have as a corollary that Algorithm~\ref{alg:pd_to_graph} always achieves the desired generation on acyclic graphs.
\begin{corollary}\label{cor::general_trees}
 Algorithm~\ref{alg:pd_to_graph} always terminates if $\mathcal{P}_1(G_{\bullet})$ is empty (i.e., $G$ is acyclic).
\end{corollary}

In practice, we could run Algorithm~\ref{alg:pd_to_graph} repeatedly and sample from successful outputs (see Section~\ref{sec::exp_local_vs_global}).

\begin{remark}
Although we focused on vertex-based filtration, there are also \textbf{edge-based filtrations}. A filtration is edge-based if and only if all tuples in $\mathcal{P}_0$ have birth time $0$. Thus, Algorithm~\ref{alg:pd_to_graph} (Appendix~\ref{app:algorithms}) can generate all possible edge-based filtrations matching the same PD.
\end{remark}

\subsection{Producing PD-Equivalent Vertex-Based Filtrations}\label{sec::vertex_based}

Given the PH signatures $\mathcal{P}(G_{\bullet})$ where $G_{\bullet}$ is a hidden vertex-based filtration, we describe a local generation process to iteratively create a PD-equivalent vertex-based filtration $H_{\bullet}$. Algorithm~\ref{alg:vertex_based_gen} (see Appendix~\ref{app:algorithms} for more details) is a modification of Algorithm~\ref{alg:pd_to_graph} where:
\begin{enumerate}[noitemsep,topsep=0pt,leftmargin=*]
    \item (\textbf{Vertex Addition}): For each new vertex added, we record them in a list $N_k$.
    \item (\textbf{Component Merging}): The pair of nodes $(x, y)$ selected has $\{x, y\} \cap N_k \neq \emptyset$.
    \item (\textbf{Cycle Addition}): The pair $(p, q)$ selected to make the cycle has $\{p, q\} \cap N_k \neq \emptyset$.
\end{enumerate}
Modifications 1-3 above would cause obstructions during component merging. To remedy this, we introduce an \textbf{additional change} to the \textbf{component merging step}, motivated by the following.
\begin{proposition}\label{prop::trivial_death_vertex_based}
In a vertex-based filtration, any vertex death at time $k$ forces a trivial pair $(k, k)$.
\end{proposition}
(\textbf{Minimum Vertex Pivot}): If there are no deaths, then ignore this step. Otherwise, let $b_{\min} \coloneqq \operatorname{inf}_{\text{component deaths of the form}(b, k)} b$. If the number of component whose birth time is $b_{\min}$ is equal to the number of tuples of the form $(b_{\min}, k)$, we draw an edge from a random $v \in N_{k}$ to a random vertex in a random component whose birth time is $< b_{\min}$. Otherwise, we draw an edge from a random $v \in N_k$ to a random vertex in a random component whose birth time is $\leq b_{\min}$.

There are still obstructions for cycle-addition, but most of preceding discussions in Section~\ref{sec::general_fil} apply.
\begin{theorem}\label{thm::vertex_filtration}
Given $(\mathcal{P}_0, \mathcal{P}_1)$ from a hidden vertex-based filtration $G_{\bullet}$. If Algorithm~\ref{alg:vertex_based_gen} terminates, the filtration $H_{\bullet}$ is vertex-based and PD-equivalent to $G_{\bullet}$. Furthermore, any vertex-based filtration $H_{\bullet}'$ that is PD-equivalent to $G_{\bullet}$ can be created by Algorithm~\ref{alg:vertex_based_gen}. Furthermore, Algorithm~\ref{alg:vertex_based_gen} can always succeed in the stages for vertex addition, minimum vertex pivot, and component merging. 
\end{theorem}

\begin{corollary}\label{cor::vertex_trees}
    Algorithm~\ref{alg:vertex_based_gen} always terminates if $\mathcal{P}_1(G_{\bullet})$ is empty.
\end{corollary}

\subsection{Producing PD-Equivalent Degree Filtrations}\label{sec::degree_based}

Given the PH signature $\mathcal{P}(G_{\bullet}^{\operatorname{deg}})$, we now describe a modification from Algorithm~\ref{alg:vertex_based_gen} that generates a PD-equivalent degree-based filtration (see Algorithm~\ref{alg:degree_based_gen} in Appendix~\ref{app:algorithms} for exact details), with:
\begin{enumerate}[noitemsep, topsep=0pt,leftmargin=*]
    \item (\textbf{Vertex Addition}): For each new vertex $v$ added, we assign a capacity number $\operatorname{cap}(v) = k$.
    \item (\textbf{Minimum Vertex Pivot}): We add the additional constraint that the edge must be drawn from $v \in N_k$ to a vertex whose capacity is positive. We subtract 1 off the capacity of the two nodes.
    \item (\textbf{Component Merging}): We ensure the pair of nodes $(x, y)$ selected has $\operatorname{cap}(x), \operatorname{cap}(y) > 0$. We then subtract 1 off the capacity of the two nodes.
    \item (\textbf{Cycle Addition}): We ensure the pair $(p, q)$ selected has $\operatorname{cap}(p), \operatorname{cap}(q) > 0$. We then subtract 1 off the capacity of the two nodes.
\end{enumerate}

\begin{theorem}\label{thm::degree_filtration}
Given $(\mathcal{P}_0, \mathcal{P}_1)$ the degree-based persistence coming from a hidden graph $G_{\bullet}^{\operatorname{deg}}$. If Algorithm~\ref{alg:degree_based_gen} terminates, the filtration $H_{\bullet}$ is degree-based and has the same PDs as $G_{\bullet}^{\operatorname{deg}}$. Furthermore, any degree-based filtration $(H')^{\operatorname{deg}}_{\bullet}$ that is PD-equivalent to $G_{\bullet}^{\operatorname{deg}}$ can be created by Algorithm~\ref{alg:degree_based_gen}. 
\end{theorem}

\subsection{Fundamental Limitations to the Local Approach}\label{sec::local_limit}

One might wonder if the inability to overcome the cycle addition stage is because the local algorithms may not have been optimized enough. In our local approach, we seek to construct our graph filtration iteratively to match the $k$-th truncation of the given $\mathcal{P}(G_{\bullet})$ such that the construction up to $H_k$ should only depend on $\mathcal{P}^k(G_{\bullet})$. We show below that these obstructions are instead a fundamental limitation of any similar local approach to generation.
\begin{resbox}
\begin{theorem}\label{thm::local_impossibility}
There is no Algorithm $A$ whose input is $\mathcal{P}(G_{\bullet})$, ranging over all general (resp. vertex-based, degree-based) filtrations of all graphs, and output is a (resp. vertex-based, degree-based) filtration $H_{\bullet}$ such that
\begin{enumerate}[noitemsep,topsep=0pt, leftmargin=*]
    \item Let $\mathcal{O}(G_{\bullet})$ be the set of possible outputs Algorithm $A$ can make from a fixed $\mathcal{P}(G_{\bullet})$. $\mathcal{O}(G_{\bullet})$ is equal to the set of all general (resp. vertex-based, degree-based) $H_{\bullet}$ PD-equivalent to $G_{\bullet}$.
    \item (\textbf{Local Approach}) Let $\mathcal{O}_{k}(G_{\bullet})$ denote the set of possible $H_{-1} \subseteq ... \subseteq H_k$ that Algorithm A can make from a given $\mathcal{P}(G_{\bullet})$. $\mathcal{O}_{k}(G_{\bullet}) = \mathcal{O}(G_{-1} \subseteq ... \subseteq G_k)$.
\end{enumerate}
\end{theorem}
\end{resbox}
Conceptually, Theorem~\ref{thm::local_impossibility} is saying that trying to ensure perfect generation at each step of the filtration (i.e., proceeding locally) would miss the global structure of the PH signature. The idea is that conditions (1) and (2) forces every graph filtration PD-equivalent to a truncation of $G_\bullet$ to admit an extension to a full PD-equivalent filtration. However, we can construct examples where no such extensions exist for general, vertex-based, and degree-based filtrations.

\section{Global Generation via Constraint Programming Formulation}
\label{sec:global}

Here, we take a global approach to generating PD-equivalent degree filtrations. In fact, our model is applicable to all permutation equivariant filtration scheme $\Phi$, which for simplicity of presentation we assume is valued in $\Zbb_{\geq 0}$ (this reduction is possible due to Remark~\ref{rmk::filtration}). 

In the local approach, one constructs a solution incrementally and makes a sequence of local decisions. This introduces a strong path dependence, which Theorem~\ref{thm::local_impossibility} shows are fundamentally limited. As persistent homology is inherently a global descriptor, we seek a way to fully solve the system as a single object and identify the entire feasible set of combinatorial constraints before any sampling occurs. This approach avoids the accumulation of local choices and enforces correctness on the global level. Solving this naturally leads to the usage of \textbf{constraint programming} (CP). 

In this section, we recast strict PD-compliant generation as a constraint satisfaction problem over Boolean and integer variables. When $\Phi = \operatorname{deg}$, the encoding is \emph{minimal-sufficient}: its satisfying assignments are in bijection with
graphs $H$ satisfying $\mathcal{P}(H_\bullet^{\mathrm{deg}}) = (\mathcal{P}_0,
\mathcal{P}_1)$, and removing any single constraint family admits assignments
that are not. Since our methodology applies to any permutation equivariant filtration $\Phi$, we also discuss how to implement various filtration functions in the solver.

\subsection{The Encoding Implementation}\label{sec::encoding}

Let $\Phi$ be a vertex-based permutation equivariant filtration scheme (e.g., degree) valued in $\Zbb_{\geq 0}$. Given an input $(\mathcal{P}_{0}, \mathcal{P}_1)$ coming from the PDs of $\Phi$ on a hidden graph $G$, we seek to generate a graph $H$ whose associated PDs under $\Phi$ matches the input.

\paragraph{Pre-computed quantities.}
From $(\mathcal{P}_0, \mathcal{P}_1)$ we read off:
the vertex set $V = [n]$ where $n = |\mathcal{P}_0|$;
the filtration value $a_i$, equal to the birth time of $i \in V$ in $\mathcal{P}_0$;
the maximum filtration stage $K = \max_i a_i$;
the active vertex set at stage $k$, $V_k = \{i : a_i \leq k\}$;
the alive-component count
$cc_k = |\{(b,d) \in \mathcal{P}_0 : b \leq k < d\}|$;
the cumulative cycle count
$\mathrm{cy}_k = |\{(b,\infty) \in \mathcal{P}_1 : b \leq k\}|$;
and the per-birth survival count
$\alpha(k,b) = |\{(b,d) \in \mathcal{P}_0 : d > k\}|$.

\textbf{Note:} Since the underlying graph $G$ is hidden, we do not know which vertex is assigned to birth/death pair in $\mathcal{P}_0$. However, this does not matter in the generation of the new graph, and we assign each $i$ to the multiset of birth/death pairs in an arbitrary bijection.

\paragraph{Variables.}
We model the graph to generate as the lower triangular part of an adjacency matrix $M$ with entries $M_{ij}$. To keep track of the persistence data, at each stage $k$, we also group the active vertices $V_k$ partition into $cc_k$ slots
indexed by $[cc_k] = \{0, 1, \dots, cc_k - 1\}$. Each slot is a model for a component of the graph at stage $k$ and can be thought of as the induced subgraph on vertices in the slot. We use the following variables to be solved in constraint programming:
\begin{center}
\small
\begin{tabular}{@{}lll@{}}
\toprule
Variable & Domain & Semantics \\
\midrule
$M_{ij}$,\ $1 \leq i < j \leq n$
  & $\{0,1\}$ & edge $\{i,j\} \in E(H)$ \\
$x_{kcv}$,\ $c \in [cc_k],\ v \in V_k$
  & $\{0,1\}$ & vertex $v$ in slot $c$ at stage $k$ \\
$\pi_{kcij}$,\ $c \in [cc_k],\ i,j \in V_k,\ i \neq j$
  & $\{0,1\}$ & $i$ is parent of $j$ in spanning tree of slot $c$ at stage $k$ \\
$\ell_{kcv}$,\ $c \in [cc_k],\ v \in V_k$
  & $\{0,\dots,n\}$ & depth of $v$ from the root of slot $c$ at stage $k$ \\
$r_{kcv} \coloneqq x_{kcv} \wedge \bigwedge_{u < v} \neg x_{kcu}$ & $\{0, 1\}$ & Is $v$ the vertex with smallest index in slot $c$ at stage $k$? \\
$e_{kcb} := \bigl(\exists\, v \in \mathrm{slot}^k_c : a_v = b\bigr)$ & \multirow{2}*{$\{0, 1\}$}
 & \multirow{2}*{slot $c$ at stage $k$ was born at time $b$}\\
$\wedge
\bigl(\nexists\, v \in \mathrm{slot}^k_c : a_v < b\bigr)$  &\\
\bottomrule
\end{tabular}
\end{center}
Here the variables $\pi_{\bullet}, \ell_{\bullet}$ are introduced to optimize the constraint solver, see Remark~\ref{rmk::optimization}.

\paragraph{Constraints.}
The encoding consists of six constraint families.

\noindent\textbf{$\mathcal{C}_1$ \;($\Phi$ Alignment).}\quad Each vertex aligns to its prescribed value under $\Phi$.
\[\Phi_{M}(i) = a_i, \qquad \forall\, i \in [n]\]

In the special case when $\Phi$ is degree, the constraint becomes:\\
\noindent\textbf{$\mathcal{C}_1^{\deg}$ \;(degree).}\quad
Each vertex realizes its prescribed degree.
\[
\sum_{j \neq i} M_{ij} = a_i
\qquad \forall\, i \in [n]
\]

\noindent\textbf{$\mathcal{C}_2$ \;(partition).}\quad
At each stage, every active vertex belongs to exactly one slot.
\[
\sum_{c \in [cc_k]} x_{kcv} = 1
\qquad \forall\, k,\, v \in V_k
\]

\noindent\textbf{$\mathcal{C}_3$ \;(edge-slot consistency).}\quad
The endpoints of every edge share a slot at every stage where both are active (i.e., there is no edge between two vertices at different slots).
\[
M_{uv} \;\leq\; 1 - |\,x_{kcu} - x_{kcv}\,|
\qquad \forall\, k,\, c \in [cc_k],\, u<v \in V_k
\]

\noindent\textbf{$\mathcal{C}_4$ \;(slot connectivity).}\quad
Each slot induces a connected subgraph at every stage. We witness this by
forcing the existence of a spanning tree on the slot, encoded via parent
pointers $\pi$ rooted at the canonical root $r_{kcv}$, with depths $\ell$
ruling out cycles. The tree must use only graph edges within the slot
($\pi_{kcij} \leq M_{ij} \wedge x_{kci} \wedge x_{kcj}$); each non-root member
has exactly one parent ($\sum_{u \neq v} \pi_{kcuv} = x_{kcv} - r_{kcv}$);
and depths increase along parent edges
($\pi_{kcij} = 1 \Rightarrow \ell_{kcj} = \ell_{kci} + 1$). Note that the spanning tree here is to ensure connectivity and is not used for the final graph ($M_{ij}$'s) produced.

\noindent\textbf{$\mathcal{C}_5$ \;(Euler characteristics).}\quad
Edge counts at each stage match the prescribed cycle-birth counts via
$|E_k| - |V_k| + cc_k = \mathrm{cy}_k$.
\[
\sum_{i<j,\ i,j \in V_k} M_{ij} \;-\; |V_k| \;+\; cc_k \;=\; \mathrm{cy}_k
\qquad \forall\, k.
\]

\noindent\textbf{$\mathcal{C}_6$ \;(birth-time survival).}\quad
The number of slots whose oldest member was born at $b$ equals
$\alpha(k,b)$, the number of $\mathcal{P}_0$ pairs born at $b$ still alive at
stage $k$.
\[
\sum_{c \in [cc_k]} \mathbf{1}[\,e_{kcb}\,] = \alpha(k,b)
\qquad \forall\, k,\, b,
\]
where $e_{kcb}$ holds iff slot $c$ contains a vertex born at $b$ but no
vertex born earlier. Note that when $\Phi$ is degree, we note that
$e_{kcb} = \bigl(\exists\, v \in \mathrm{slot}^k_c : \mathrm{deg}(v) = b\bigr) \wedge
\bigl(\nexists\, v \in \mathrm{slot}^k_c : \mathrm{deg}(v) < b\bigr).$ 

\begin{remark}\label{rmk::optimization}
We used spanning trees to enforce connectivity in $\mathcal{C}_{4}$ as an optimization step for the solver. A naive approach to checking connectivity would be to ensure that any two vertices in a slot has a path between them, which is about quadratic in time. Checking a spanning tree is about linear.
\end{remark}

\begin{resbox}
\begin{theorem}[Correctness of the encoding]
\label{thm:cp_correct}
Let $\mathcal{C} = \{\mathcal{C}_1, \dots, \mathcal{C}_6\}$, $\mathcal{F}$
the set of satisfying assignments to $\mathcal{C}$, and
$\mathcal{H}(\mathcal{P}_0, \mathcal{P}_1)$ the set of graphs $H$ with
$\mathcal{P}(H_\bullet^{\Phi}) = (\mathcal{P}_0, \mathcal{P}_1)$.
Define $H(\sigma)$ from $\sigma \in \mathcal{F}$ by
$E(H(\sigma)) = \{\{i,j\} : \sigma(M_{ij}) = 1\}$.
Then:
\begin{enumerate}[noitemsep, topsep=0pt,leftmargin=*]
\item \textbf{Soundness:} $H(\sigma) \in \mathcal{H}(\mathcal{P}_0,\mathcal{P}_1)$
  for every $\sigma \in \mathcal{F}$.
\item \textbf{Completeness:} every $H \in \mathcal{H}(\mathcal{P}_0,\mathcal{P}_1)$
  admits at least one $\sigma \in \mathcal{F}$ with $H(\sigma) = H$.
\item \textbf{Minimality for Degree:} When $\Phi = \operatorname{deg}$, for each $i \in \{1, \dots, 6\}$, there exist a
  target $(\mathcal{P}_0, \mathcal{P}_1)$ and an assignment $\sigma$ that
  satisfies $\mathcal{C} \setminus \{\mathcal{C}_i\}$ but with
  $H(\sigma) \notin \mathcal{H}(\mathcal{P}_0, \mathcal{P}_1)$.
\end{enumerate}
\end{theorem}
\end{resbox}
Here, soundness follows as all constraints imposed are formal properties satisfied by PDs, and completeness follows from carefully arguing these constraints completely determine the representatives. Note that the completeness of Theorem~\ref{thm::completeness} implies for any graph $G$ the solver can generate, it can also reach every representative in the same graph isomorphism class.

The full proof of Theorem~\ref{thm:cp_correct} (1-2) is in Appendix~\ref{app:cp_proofs}. In Appendix~\ref{sec:appendix-counterexamples}, we generated 6 assignments see (Figure~\ref{fig::counter_example}) as counter-examples in the proof for the minimality of the encoding when $\Phi = \operatorname{deg}$. In Appendix~\ref{app:complexity}, we also provide a complexity and scaling analysis of the global encoder for degree.

\subsection{Implementation of Various Filtration Schemes}

To demonstrate that the framework is not limited to degree, we implement other filtration schemes. Since only $\mathcal{C}_1$ (the alignment $\Phi_M(i) = a_i$) depends on the filtration, we specify it for each below. Recall $M_{ij}$ is the $(i,j)$-th entry in the adjacency matrix and $d_i = \sum_{j \neq i} M_{ij}$.

\noindent\textbf{(a) triangle and $k$-clique count.}\quad For each vertex $v$, its triangle count is the number of 3-cycles on $G$ that includes the vertex $v$. We can encode this as
\[\textbf{Triangle Count:}
\sum_{j < l,\ j,l \neq i} M_{ij} M_{il} M_{jl} = a_i\qquad \forall\, i \in [n].\]
Similar formulas also hold for $k$-clique count for $k  \geq 3$.

\noindent\textbf{(b) clustering coefficient.} The clustering coefficient of a vertex $i$ is twice its triangle count $t_i$ divided by $d_i(d_i-1)$ when $d_i\ge 2$, and $0$ otherwise. We can encode this as the equality $$\textbf{Clustering Coefficient:}\ 
2\sum_{j<l,\ j,l\neq i} M_{ij}M_{il}M_{jl} = a_i \max(1, d_i(d_i-1)) \qquad \forall\, i \in [n].$$
\begin{remark}
As $a_i$ is rational but the solver is integer-based, we in practice scale by a positive integer $D$ to give the equation $2D\,t_i = p_i\,\max(1, d_i(d_i-1))$ with $p_i=D\,a_i \in \mathbb{Z}$. This is exact because the persistence diagram depends only on the order of the filtration values, not their scale (Remark~\ref{rmk::filtration}). We also apply this in practice to the subsequent filtrations.
\end{remark}

\noindent\textbf{(c) average neighbor degree.} The average neighbor degree of a vertex is the average of degrees of its neighboring vertices. This can be encoded as $\big(\sum_{j\neq i}M_{ij}d_j\big)/\max(1,d_i) = a_i$, or more preferably,
\[\textbf{Average Neighbor Degree:}\ \sum_{j\neq i} M_{ij} d_j = a_i \max(1, d_i)\qquad \forall\, i \in [n].\]

\noindent\textbf{(d) closeness centrality.} For a \textbf{connected graph}, the closeness centrality of a vertex $i$ is $n-1$ divided by the sum of shortest-path distances $\delta(i,j)$ from every $j$ to $i$. We can encode this as
$$\textbf{Closeness Centrality}:\ n-1 = a_i \sum_{j \neq i} \delta_{ij} \qquad \forall\, i \in [n].$$
Here $\delta_{ij}$ are integer variables forced to be $\delta(i,j)$ by imposing the constraints: for each source $i$, (i) $\delta_{ii}=0$; (ii) $\delta_{ij}\ge 1$ for $j\neq i$; (iii) $|\delta_{ij}-\delta_{ik}|\le 1$ when $M_{jk}=1$; and (iv) every $j\neq i$ has a neighbor $p_{ij}$ with $\delta_{ip_{ij}}=\delta_{ij}-1$. This forces $\delta_{ij} = \delta(i,j)$.

\noindent\textbf{(e) random walk centralities.} The same technique of pinning a value as the unique solution of a system covers both \textbf{Katz centrality} and \textbf{PageRank}. For a given $\alpha$, we can encode them as
\[\textbf{Katz centrality}:\ a_i=1+\alpha\sum_{j\neq i}M_{ij}a_j,\ \forall\, i \in [n]\ ; \textbf{PageRank}:\ a_i=\tfrac{1-\alpha}{n}+\alpha\sum_{j\neq i}\tfrac{M_{ij}}{d_j}a_j,\ \forall\, i \in [n]. \]

In Appendix~\ref{app:filtrations}, we verify soundness and completeness of all filtrations above on small graphs.

\section{Experiments}\label{sec::experiments}

\xhdr{Datasets.} We evaluate on four standard graph generation benchmarks: \textsc{Comm-20}~\cite{you2018graphrnn}, \textsc{Planar}~\cite{martinkus2022spectre}, \textsc{SBM}~\cite{martinkus2022spectre}, and \textsc{ENZYMES}~\cite{morris2020tudataset}, spanning a range of sizes, densities, and topological regimes. We adopt the SPECTRE $64{:}16{:}20$ train/validation/test split convention~\cite{martinkus2022spectre} uniformly across all four datasets. Statistics and splits are deferred to Appendix~\ref{app:datasets}, see also Remark~\ref{rmk::implementation_bmin} for a technical detail on implementing Algorithm~\ref{alg:degree_based_gen}.

\subsection{Comparing Local and Global Generation}
\label{sec::exp_local_vs_global}

In this first experiment we evaluate our two PD-exact generators against each other. For each test graph, we feed its degree-based PD $(\mathcal{P}_0, \mathcal{P}_1)$ as the target to both methods, and report the success rate, wall-clock cost, and (for the local sampler) attempts before first success.

\xhdr{Setup.} The local degree-filtration generator (Algorithm~\ref{alg:degree_based_gen}) is allowed up to $20{,}000$ rejection attempts per target before a run is marked as failed; the global CP solver (Sec.~\ref{sec:global}) is given a $1{,}800$\,s per-target time budget. Both run on a $120$-CPU-core machine with all per-target generations parallelised across workers. We report success rate over the test split, mean wall-clock per target (charging the full budget on failed runs), and, for the local sampler, mean attempts to first success on successful runs. Results under random and vertex filtrations, together with median statistics, are deferred to Appendix~\ref{app:local_extra}.

\begin{table}[!ht]
\centering
\caption{Local degree-filtration sampler vs.\ the global CP solver. Wall-clock means charge the full attempt or time budget on failed runs.}
\label{tab::local_vs_global}
\small
\setlength{\tabcolsep}{6pt}
\renewcommand{\arraystretch}{1.0}
\begin{tabular}{@{}llcccc@{}}
\toprule
& & \textsc{Comm-20} & \textsc{Planar} & \textsc{SBM} & \textsc{ENZYMES} \\
\midrule
\multirow{3}{*}{\textsc{Local} (Alg.~\ref{alg:degree_based_gen})}
 & Success rate (\%) $\uparrow$              & 100.0  & 100.0   & 15.0    & 99.1   \\
 & Wall-clock per target (s) $\downarrow$    & 0.17   & 5.88    & 1465.0  & 6.94   \\
 & Attempts to first success $\downarrow$    & 49.6   & 137.5   & 4583.5  & 431.5  \\
\midrule
\multirow{2}{*}{\textsc{Global} (CP, Sec.~\ref{sec:global})}
 & Success rate (\%) $\uparrow$              & 100.0  & 100.0   & \textbf{82.5}   & 100.0  \\
 & Wall-clock per target (s) $\downarrow$    & 0.30   & 16.35   & \textbf{654.8}  & 6.85   \\
\bottomrule
\end{tabular}
\end{table}

\xhdr{Discussion.}  Table~\ref{tab::local_vs_global} shows that  on \textsc{Comm-20}, \textsc{Planar}, and \textsc{ENZYMES} both methods succeed essentially everywhere, with the local sampler $1.7$--$2.8\times$ faster on the smaller datasets and on par on \textsc{ENZYMES}. However, in \textsc{SBM}, the local sampler succeeds on only $15\%$ of targets and exhausts the budget on the rest, while CP returns a valid graph on $82.5\%$ at less than half the wall-clock. 
Consistent with Theorem~\ref{thm::local_impossibility}, local moves fail when the PD encodes a tightly constrained global structure, which CP's global view recovers. The two methods are thus complementary. Appendix~\ref{app:prac} gives a practitioner's guide to choosing a priori among them and approximate neural generators.

\subsection{Comparison with Neural Generative Baselines}
\label{sec::exp_neural}

We now compare against neural diffusion baselines that approximate PD-conditioning.
\begin{figure}[t]
\centering
\includegraphics[width=0.80\textwidth]{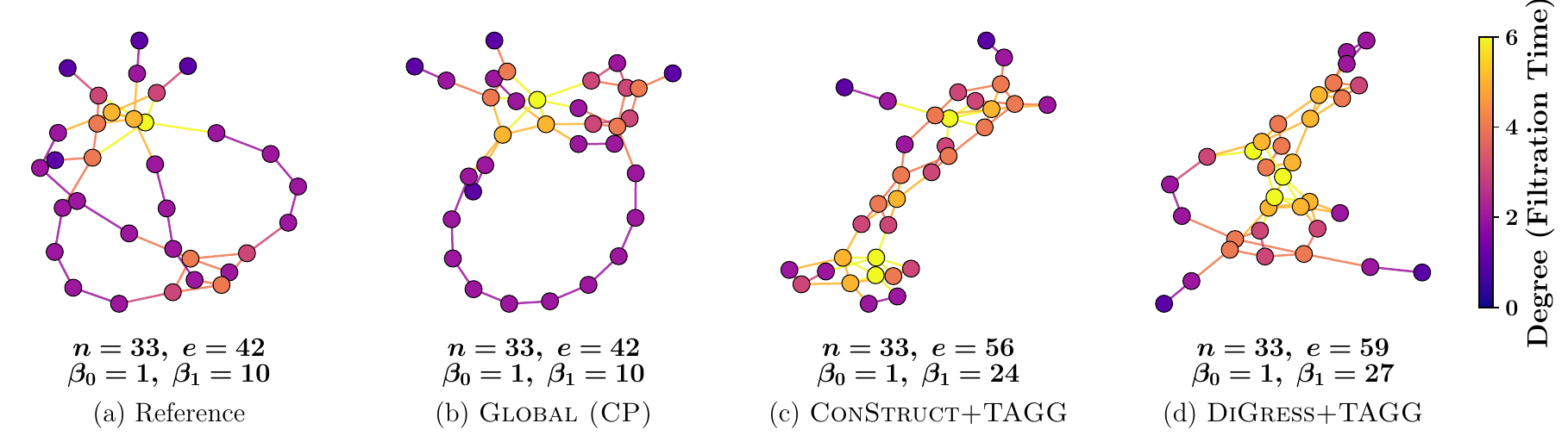}
\caption{Qualitative example on an \textsc{ENZYMES} test graph. 
Node colour encodes degree. CP matches $(n, |E|, \beta_0, \beta_1)$ exactly, while neural variants add excess cycles. 
More  in Appendix~\ref{app:qualitative}.}
\label{fig::qualitative}
\end{figure}

\xhdr{Setup.} We compare two backbones, \textsc{DiGress}~\cite{vignac2022digress} and \textsc{ConStruct}~\cite{madeira2024generative}, both augmented with \textsc{TAGG}'s~\cite{parktopology} topology-aware attention module (TAM) and $1$-Wasserstein PDM loss, yielding \textsc{DiGress+TAGG} and \textsc{ConStruct+TAGG}. Both variants are trained per-dataset with PD conditioning supplied via TAM and regularised by PDM. At inference we feed each test graph's own persistence-landscape vector $\mu_{G_{\mathrm{test}}}$ rather than the dataset average $\mu_{G'}$ used by the original \textsc{TAGG}, granting the neural variants the strongest possible per-target signal and a paired one-to-one correspondence with the reference. CP receives the same target PD as direct input. Neural variants are evaluated over $5$ random seeds; CP is deterministic. We measure topological fidelity by the mean paired $1$-Wasserstein distance under the degree filtration, $\bar W_1 = \tfrac{1}{N}\sum_i [W_1(\mathcal{P}_0(G^{(i)}_\mathrm{ref}), \mathcal{P}_0(H^{(i)})) + W_1(\mathcal{P}_1(G^{(i)}_\mathrm{ref}), \mathcal{P}_1(H^{(i)}))]$, with diagonal padding (the metric used as the PDM training loss). 
Details are deferred to Appendix~\ref{app:exp_neural}.

\begin{wraptable}{r}{0.5\textwidth}
\vspace{-12pt}
\centering
\caption{\textsc{Global} (CP) vs.\ TAGG-augmented neural baselines on degree-PD $\bar W_1$ ($0$ = exact match).}
\label{tab::cp_vs_tagg}
\footnotesize
\setlength{\tabcolsep}{3pt}
\renewcommand{\arraystretch}{1.0}
\begin{tabular}{@{}llcccc@{}}
\toprule
& & \textsc{Comm-20} & \textsc{Planar} & \textsc{SBM} & \textsc{ENZ.} \\
\midrule
\multirow{3}{*}{\textsc{CP}}
 & $\bar W_1$ $\downarrow$        & \textbf{0.00} & \textbf{0.00} & \textbf{0.00} & \textbf{0.00} \\
 & Time (s) $\downarrow$          & 0.30 & 16.35 & 654.8 & 6.85 \\
 & Train (h) $\downarrow$         & 0    & 0     & 0     & 0    \\
\midrule
\multirow{3}{*}{\shortstack[l]{\textsc{ConStr.}\\\textsc{+TAGG}}}
 & $\bar W_1$ $\downarrow$        & 9.35  & 43.23 & 434.48 & 17.35 \\
 & Time (s) $\downarrow$          & 0.88  & 4.73  & 17.11  & 4.92  \\
 & Train (h) $\downarrow$         & 9.4   & 24.3  & 59.8   & 61.6  \\
\midrule
\multirow{3}{*}{\shortstack[l]{\textsc{DiGress}\\\textsc{+TAGG}}}
 & $\bar W_1$ $\downarrow$        & 4.52  & 38.77 & 464.99 & 17.06 \\
 & Time (s) $\downarrow$          & 0.90  & 5.30  & 23.16  & 5.94  \\
 & Train (h) $\downarrow$         & 8.9   & 32.5  & 47.6   & 49.2  \\
\bottomrule
\end{tabular}
\vspace{-15pt}
\end{wraptable}

\xhdr{Discussion.} Table~\ref{tab::cp_vs_tagg} shows that CP achieves $\bar W_1 = 0$ on every dataset by construction, while the neural variants do not. 
(full table with std.\ dev.\ in Appendix~\ref{app:exp_neural}, Table~\ref{tab::cp_vs_tagg_full})
Figure~\ref{fig::qualitative} makes the same point per-graph: the neural variants merge two components into one and introduce excess cycles, while CP returns a different graph from the reference (many graphs share the same PD; see Appendix~\ref{app:diversity}) but matches every degree-PD invariant exactly. A second qualitative example, with a full per-graph breakdown of edge and cycle drift, is in Appendix~\ref{app:qualitative}.
In terms of runtime cost, neural variants are $3$--$38\times$ faster than CP at inference on larger datasets but require $9$--$62$ hours of per-dataset training; CP needs none. The two answer different questions: CP when exact PD-correctness is required and the constraint problem is tractable; neural variants when approximate correctness suffices and the training cost amortizes over many generations.

\section{Conclusion}

We initiated a principled study of strict PD-compliant graph generation: producing graphs whose persistence diagram under a given filtration matches a target exactly. Focusing on the degree filtration, we developed two complementary generators:  a local rejection sampler with soundness and completeness guarantees, and a global constraint-programming formulation with a minimal-sufficient encoding. Furthermore, we extended the global formulation to any vertex-based permutation equivariant filtration scheme. We showed that the local view, while cheap on easy instances, is fundamentally limited (Theorem~\ref{thm::local_impossibility}) when the PD encodes tight global structure. Empirically, both methods produce graphs realizing the target PD exactly on standard benchmarks, while neural diffusion baselines with PD-conditioning regularizers leave a structural gap that more training does not close. We see this work as a first step: natural extensions include richer filtration schemes, controlled relaxations of strict compliance (e.g.\ bottleneck-ball targets), and using the constraint encoding as a building block inside learned generators.

\newpage
\begin{ack}
This project was conceived during discussions at NeurIPS 2025, with early ideas dating back to when MJ and IR were participating during 2025 Aalto Science Institute international summer research programme at Aalto University. MJ is partially supported by the National Science Foundation Graduate Research Fellowship (DGE-2236662). IR and VG acknowledge support from Saab-WASP (grant 411025). VG also acknowledges the Jane and Aatos Erkko foundation (grant 7001703) for their support. We are grateful to the anonymous program chair, area chairs and the reviewers for their constructive feedback and service. MJ would especially like to thank Thomas Brazelton who inspired MJ to think about the exact topological compliance for graphs. MJ would also like to thank Erin Chambers for helpful conversations as well as Yuqin Kewang, Guangyu Meng, and Shankha Mukherjee for support during rebuttal. 
\end{ack}

{
\small
\bibliographystyle{abbrv}
\bibliography{references, references_camera_ready}
}


 \appendix

\newpage

\begin{center}
{\Large\bfseries\ztitle} \\
{\Large\bfseries(Appendix)}
\end{center}

\section{Limitations and Future Directions}\label{sec::limitations}

There are interesting graph properties that are not preserved under PD-equivalent graphs. For example, planarity and information about clustering are not necessarily invariants under PD-equivalence. The generation methods discussed in this paper are not guaranteed to maintain them either. As a future direction, we could aim to generate graphs that preserves both topological properties and other non-topological properties.  In Appendix~\ref{app:diversity}, we tested how diverse graphs in a single degree-based PD-equivalence class can be. Indeed, within one PD-equivalence class, clustering, triangle counts, spectra, and planarity all vary.

There are also a number of future directions that would be interesting to explore, for which our work did not focus on.

\xhdr{Higher-order PDs.} We restricted to filtrations on graphs, where the PDs of dimension $\geq 2$ are empty. One may be interested to extend this to filtrations on simplicial complexes. As there are higher simplices, tracking their effects would require matrix-reduction procedures, which increases both computational cost and the chance of infeasible states. To handle higher-order PDs on simplicial complexes, we would need to incorporate additional constraints in Section~\ref{sec:global}. We anticipate it would be challenging to characterize simple combinatorial constraints for the higher-order PDs, and the CP solver would certainly take longer to solve.

\xhdr{Scalability.} Exact global generation can be computationally expensive, and the CP formulation is not expected to scale uniformly to large or difficult instances. In Appendix~\ref{app:complexity}, we noted that the encoding size is polynomial; solving is NP-hard in general (it subsumes SAT), so no closed-form runtime bound applies. In regimes where both exact methods fail (harder SBM instances are the concrete example), neural sampling is the only tractable choice, at the cost of exact PD compliance.

\xhdr{General filtrations.} The soundness and completeness of the encoding in Theorem~\ref{thm:cp_correct} is independent of the given $\Phi$. However, solve cost depends on the encoding complexity: degree is linear, whereas triangle and clustering add $O(n^3)$ auxiliaries and are costlier (see Appendix~\ref{app:filtrations}). One could imagine the costs would vary across general filtrations.

\xhdr{Neural interface.} Exact persistence-diagram agreement does not imply that the generated graphs preserve other relevant properties. Coupling our methods with a neural generator to preserve properties topological summaries miss, such as planarity and spectral data, is a promising direction for future work, as is a downstream end-to-end demonstration. Since one motivation for exact compliance was to find counter-examples, one may ask if neural interface (along the themes of~\cite{wagner2021constructions}) can help to find such examples.

\section{Broader Impact Statement}\label{app:impact}

This paper presents work whose goal is to advance the problem of strictly Persistence-Diagram compliant graph generations. As highlighted in the introduction, topological descriptors have a wide variety of usage in biology, infrastructure networks, and more. There are many potential societal consequences of our work, none of which we feel must be specifically highlighted here.

\section{Declaration of LLM Usage}\label{app::llm}

We used large language models (LLMs) for help in the coding and visualization part of the paper. We also consulted LLMs in the writing of the paper, where we double checked the written statements. The proofs in Appendix~\ref{app::loc_proofs} and Appendix~\ref{app:cp_proofs} were written by human.

\newpage
\section{Proofs in Section~\ref{sec::local}}\label{app::loc_proofs}

\subsection{Proofs in Section~\ref{sec::general_fil}}

\begin{proof}[Proof of Theorem~\ref{thm::general_filtration}]
For the first statement, suppose the algorithm terminates at the end, then we have created a sequence of subgraphs at each stage as:
\[H_{\bullet}: H_{-1} = \emptyset \subseteq H_0 \subseteq ... \subseteq H.\]
Here we verify that $\mathcal{P}(H_{\bullet}) = (\mathcal{P}_0, \mathcal{P}_1)$. Indeed, at each stage $k$, each vertex addition step adds in un-filled tuples of the form $(k, -)$ for each $(k, d) \in \mathcal{P}_0$. For each $(b, k) \in \mathcal{P}_0$ processed at the component merging step, the edge added in the component merging step goes from a component born at time $b$ to one born at time $\leq b$, which has the effect of killing off the prior component. If they both happened to have the same birth time, we can without loss still kill off the first component. Thus, we see that the component merging step perfectly aligns with the $(b, k)$'s processed in $\mathcal{P}_0$. The same discussion applies to the cycles $(k, \infty)$ processed in the cycle addition step.\\

Thus, we see that as long as the algorithm terminates, then by design the result filtration $H_{\bullet}$ has persistence pairs strictly equal to $(\mathcal{P}_0, \mathcal{P}_1)$.\\

For the second statement, by induction the proof reduces to the following case - suppose we have produced a sequence of subgraphs from the algorithm
\[(G'_k)_{\bullet}: \emptyset = G'_{-1} \subseteq G'_{0} \subseteq ... \subseteq G'_{k},\quad \mathcal{P}^i((G'_{k})_{\bullet}) =  \mathcal{P}^i(G_{\bullet}) \text{ for } -1 \leq i \leq k,\]
Suppose we have $G'_{k+1} \supseteq G'_k$ such that $\mathcal{P}^{k+1}((G'_{k+1})_{\bullet}) = \mathcal{P}^{k+1}(G_{\bullet})$, we seek to show that $G'_k$ can be extended to $G'_{k+1}$ \textbf{using the moves described in the algorithm}. Note here the base case $k = -1$ is vacuously true.\\

Indeed, let $H$ be $G'_{k+1} - G'_{k}$ (which could have open edges missing a vertex) and suppose $H$ has $m$ edges.\\

Although in the step $G'_k \subseteq G'_{k+1}$, the entire $H$ is added at once, we can algorithmically think of this as adding individual vertices and edges one at a time. The theory of persistent homology ensures that the final persistence pairs produced at the end are well-defined independent of the order of simplices added, as long as whenever an edge $e = (u, v)$ is added, the vertices $u$ and $v$ are added first. Thus we can think of the step $G'_{k} \subseteq G'_{k+1}$ as
\[G'_{k} \subseteq G^{'}_{k, 0} \subseteq G'_{k, 1} \subseteq ... \subseteq G'_{k, m} = G'_{k+1}\]
where $G'_{k} \subseteq G'_{k, 0}$ adds all the vertices of $H$ in, and for $i \geq 0$, $G'_{k, i} \subseteq G'_{k, i+1}$ adds in exactly one edge from $H$. The step $G'_k \subseteq G'_{k, 0}$ is exactly the first two steps of the algorithm.\\

Now from the step $G'_{k,0} \to G'_{k, 1}$, an edge either kills a tuple in $\mathcal{P}_0$ or makes a tuple in $\mathcal{P}_1$. If the edge $e = (v, w)$ adds a cycle, then $v, w$ belongs to the same component at $G'_{k,0}$. Observe that in general if two vertices $v, w$ are in the same component of graph $A$, and $A \subseteq B$, then $v, w$ are in the same component of $B$. Thus, it is possible that an edge created in the \textbf{cycle creation} step of the algorithm links $v$ and $w$. If the edge fills a tuple $(b, -)$ to $(b, k+1)$, then it is drawn from an undead component containing $v$ whose birth time is at $b$ to a component containing $w$ that is born earlier. Now observe that in general if the component of $v \in A$ has birth time $b$, and $A \subseteq B$, then $v \in B$ also has birth time $\leq b$ (i.e. the birth time can only decrease). Thus, this can happen in the first tuple encountered in the \textbf{componet merging} step.

We can inductively align each step $G'_{k, i} \subset G'_{k, i+1}$ such that the edge $e$ added, if contributing to the death time of a component, is the $j$-th step in the for loop of the \textbf{component merging} step, and $e$ is the $j$-th component-killing-edge added. Indeed, by the second observation in the last paragraph, it suffices to note that the cycle-creating-edges does not link different components. Thus, we conclude the order of component-killing-edges added in $G'_{k, 0} \to G'_{k, 1}$ can be realized as the algorithmic steps occurring in the \textbf{component merging} step. As argued, above, the cycle-creating edges can all be deferred to the \textbf{cycle creating} step and are hence accounted as well. This concludes the proof.
\end{proof}

\begin{proof}[Proof of Proposition~\ref{prop::cycle_capacity}]
Suppose by induction we have iteratively built a graph $H_{k-1}$. If we just view $H_{k-1}$ as the final graph, then by the first part of Theorem~\ref{thm::general_filtration}, we know that $\mathcal{P}(H_{-1} \subseteq H_0 \subseteq ... H_{k-1}) = \mathcal{P}^k(G_{\bullet})$ ($\dagger$).\\

Now, clearly the vertex addition step can always be carried out, as we are just adding disjoint vertices to the graph. The \textbf{component merging} step can always be carried out. Indeed, by $(\dagger)$, we know the components in $H_{k-1}$ and the components in $G_{k-1}$ have the same multi-set of birth times. Thus, we can always find the desired edge between two components in the component merging step.\\

Finally, for the \textbf{cycle addition} step: We first note that, in general, for non-empty connected graphs $A$ and $B$ where $A$ has more than $1$ vertex, let $C$ denote linking $A$ and $B$ by an edge. Observe that $C$ always has room for cycles since we can draw an edge between any vertex in $A$ not linked to $B$ to vertices in $B$.\\

Now we observe that for a graph $A$ with isolated vertices, write $A = A' \sqcup A''$ where $A'$ has no isolated vertices and $A''$ is a discrete graph (no edges). Then clearly $\operatorname{cap}^{\operatorname{cyc}}(A) = \operatorname{cap}^{\operatorname{cyc}}(A')$ as isolated vertices cannot support cycles.\\    

Now let $H'_{k} \supseteq H_{k-1}$ be the new graph extending $H_{k-1}$ that is created after the \textbf{component merging} step. Observe that all the edges $e$ added in the \textbf{component merging step} fall into one of the following two scenarios (1) $e$ is drawn between two isolated vertices or (2) $e$ is drawn between two connected graphs $A$ and $B$ where one of them has more than $1$ vertex. Regardless, we see that the cycle capacity of the graph can only increase, and $\operatorname{cap}^{\operatorname{cyc}}(H'_k) \geq \operatorname{cap}^{\operatorname{cyc}}(H_k)$. Finally, the \textbf{cycle addition} step can succeed if and only if $\operatorname{cap}^{\operatorname{cyc}}(H'_k)$ is $\geq$ the number of tuples of the form $(k, \infty)$ in $\mathcal{P}_1$. This holds since it holds for $\operatorname{cap}^{\operatorname{cyc}}(H_k)$.
\end{proof}

\begin{proof}[Proof of Corollary~\ref{cor::general_trees}]
    By Proposition~\ref{prop::cycle_capacity}, we know that we can successfully build $H_{k+1}$ from a given $H_k$ if $\operatorname{cap}^{\operatorname{cyc}}(H_k)$ is greater than or equal to the number of tuples of the form $(k+1, \infty)$ in $\mathcal{P}_1$. If $\mathcal{P}_1(G_{\bullet})$ is the empty list, this number is in particular zero. Since the cycle capacity of a graph is non-negative, it is above greater than or equal to $0$. Thus, we can always successfully build $H_{k+1}$ from a given $H_k$ here. The statement then follows from induction.
\end{proof}

\subsection{Proofs in Section~\ref{sec::vertex_based}}

\begin{proof}[Proof of Proposition~\ref{prop::trivial_death_vertex_based}]
It suffices to show that any non-trivial vertex death at a time must also be accompanied with trivial deaths. We suppose for contradiction that this statement is false.

Let $G_{\bullet}$  be a vertex-based filtration of such counter-example at stage $G_k \subseteq G_{k+1}$. Now suppose the non-trivial death occurs by an edge $e = (v, w)$ in the steps building $G_k$ to $G_{k+1}$.  Without loss, we can also take $v \in V(G_k)$ and $w \in V(G_{k+1} - G_k)$, because $f$ is a vertex-based filtration. If $w$ has not been linked to any edge yet, this can only constitute a trivial death as $w$ is born later than $v$. Thus, it must be the case that $w$ had an edge connected to it already. In the first time $w$ is connected to some other vertex $v_0$, either $w$ dies because $v_0$ is younger (giving a trivial pair) or $w$ and $v_0$ are born at the same time (still giving a trivial pair). Thus, we see we would get a trivial pair regardless.
\end{proof}
\begin{proof}[Proof of Theorem~\ref{thm::vertex_filtration}]
By design, Algorithm~\ref{alg:vertex_based_gen} realizes a subset of possible steps taken in Algorithm~\ref{alg:pd_to_graph}. Thus Theorem~\ref{thm::general_filtration} implies that the final $H_{\bullet}$, if created, has $\mathcal{P}(H_{\bullet}) = (\mathcal{P}_0, \mathcal{P}_1)$. Now we wish to ensure that $H_{\bullet}$ is a vertex-based filtration. Indeed, in the \textbf{component merging} step and the \textbf{cycle creation} step, every edge that is added is conditioned have at least one of its vertices being supported in the newly created vertices at that stage. Thus, the filtration is vertex-based.\\

For any vertex-based filtration $H'_{\bullet}$ that is PD-equivalent to $G_{\bullet}$, we now argue how Algorithm~\ref{alg:vertex_based_gen} can create this. Indeed, if the \textbf{minimum vertex pivot} step is not imposed, then clearly Theorem~\ref{thm::general_filtration} implies that the algorithm can recreate $H'_{\bullet}$, as we had only imposed the minimum constraints to make them vertex-based. Thus, it suffices to show that this minimum vertex pivot can always be realized as the first step in $H'_{k} \subseteq H'_{k+1}$ for any vertex-based filtration (provided a vertex death has happened).\\

Indeed if vertex deaths occurred, then the component merging step has to come with the addition of new edges $\{e_1, ..., e_r\}$. Now look at all components that the vertices of $e_i$'s belong in for all $i$ and find the one with smallest birth time. Since the list of edges is non-empty, we can find a possibly non-unique component.\\

Now we claim we can perform the minimum vertex pivot there. Indeed, let $B$ be the birth time of this component. By construction we know that $B \leq b_{min}$, as $b_{min}$ takes the minimum over a subset of the edges. If $B < b_{min}$, then this can always be realize in the minimum vertex generation step of the algorithm. If $B = b_{min}$, we need to check that not all components born at $b_{min}$ are set to die at $k+1$. Indeed, suppose for contradiction they are, then to kill off the last one, you need to make an edge from a component born earlier than $b_{min}$, which implies $B < b_{min}$. Thus, we see the step for minimum vertex pivot is always possible and can go first.\\

Finally for the last sentence: clearly we can always do vertex addition, and the ability to do minimum vertex pivot follows from Proposition~\ref{prop::trivial_death_vertex_based}. The ability to succeed in component merging follows similarly as in Proposition~\ref{prop::cycle_capacity} from induction and the fact that the newly created vertices from stage $k-1$ to stage $k$ are the same in the hidden filtration and the filtration we are creating.
\end{proof}

\begin{proof}[Proof of Corollary~\ref{cor::vertex_trees}]
This again follows similarly as to the proof of Corollary~\ref{cor::vertex_trees}. Since we do not need to add cycles if $\mathcal{P}_1$ is empty, Theorem~\ref{thm::vertex_filtration} implies the algorithm always terminates.
\end{proof}

\subsection{Proofs in Section~\ref{sec::degree_based}}

\begin{proof}[Proof of Theorem~\ref{thm::degree_filtration}]
By design, Algorithm~\ref{alg:degree_based_gen} realizes a subset of possible steps taken in Algorithm~\ref{alg:vertex_based_gen}. Thus Theorem~\ref{thm::vertex_filtration} implies that the final $H_{\bullet}$, if created, is vertex-based and has $\mathcal{P}(H_{\bullet}) = (\mathcal{P}_0, \mathcal{P}_1)$. This in particular implies $H$ and $G$ have the same number of vertices born at each stage, and the same number of edges added at each stage ($\dagger$).\\

For the first statement, it suffices to show that the filtration is degree-based. Since this is already vertex based, it suffices to show that the degree of each vertex added at stage $k$ is $k$. This is exactly what the capacity assigned to each vertex means. If the algorithm can terminate at the end, then the capacity of each node has to be $0$ by ($\dagger$) and pigeonhole's principle. Thus, a node born at time $k$ will obtain $k$ many edges incident to it in the final graph. This gives the desired degree-based filtration.\\

For the second sentence: A degree-based filtration is in particular a vertex-based filtration. By Theorem~\ref{thm::vertex_filtration}, it can be realized by Algorithm~\ref{alg:vertex_based_gen}. On the other hand, any realization of this in Algorithm~\ref{alg:vertex_based_gen} automatically passes the capacity check, as the original filtration is a degree-based filtration. Thus, it can be realized by Algorithm~\ref{alg:degree_based_gen}.
\end{proof}

\subsection{Proofs in Section~\ref{sec::local_limit}}
\label{sec::local_limit_proofs}

\begin{proof}[Proof of Theorem~\ref{thm::local_impossibility}]
We split the proof into the following cases.

\textbf{For General Filtrations:} Suppose for contradiction that such Algorithm A exists. Consider $G_{\bullet} = (G_{-1} \subseteq G_0 \subseteq G_1 \subseteq G_2)$ as given in Figure~\ref{fig::cycle_capacity}. By (1) and (2), we have that $(H_{-1} \subseteq H_0 \subseteq H_1) \in \mathcal{O}_{1}(G_{\bullet})$. By (1), we know that there must exist some $H_2$ extending $H_1$ such that $\mathcal{P}(H_{-1} \subseteq H_0 \subseteq H_1 \subseteq H_{2}) = \mathcal{P}(G_{\bullet})$. However, as indicated by the Example~\ref{exp::no_more_cycle}, no such example can exist.

\textbf{For Vertex-Based Filtrations}: Suppose for contradiction that such Algorithm A exists. Consider $G_{\bullet}$ in the following diagram:
\[
\input{Figures/vertex_based_capacity}
\]
We have that $\mathcal{P}(G_{-1} \subseteq G_0 \subseteq G_1) = \mathcal{P}(H_{-1} \subseteq H_0 \subseteq H_1)$. Thus by (1) and (2), there should exist $H_2$ extending $H_{-1} \subseteq H_0 \subseteq H_1$ such that $\mathcal{P}(H_{-1} \subseteq H_0 \subseteq H_1 \subseteq H_2) = \mathcal{P}(G_{\bullet})$. We see from the persistence pair that this requires the addition of 4 edges (one can count the number of edges added based on how many vertex deaths and cycle births happened). We see from $\mathcal{P}(G_{\bullet})$ that 3 cycles are made and 1 trivial vertex death occurred in going from $G_1 \subseteq G_2$. However, by pigeonhole's principle, there is no way to adjoin a new vertex $v$ to $H_1$ along with 4 edges from $H_1$ to $v$ such that this happens.

\textbf{For Degree-Based Filtrations:} Suppose for contradiction that such Algorithm A exists. Observe that (1) and (2) implies a sub-filtration of a degree-based filtration should be a degree-based filtration. However, this is not true. For example, consider the degree based filtration of either Graph 1 or Graph 2 in Figure~\ref{fig:intro-ex} and truncate up to $t = 2$. There are degree 1 vertices that are spawned at $t = 2$, which is a contradiction.
\end{proof}

\subsection{Additional Remark for Section~\ref{sec::local}}

\begin{remark}
Instead of asking Algorithm~\ref{alg:pd_to_graph} to generate PD-equivalent graph filtrations, we could ask Algorithm~\ref{alg:pd_to_graph} to defer the cycle addition to a later stage if it ran out of cycles to add. This would typically make a graph filtration that is no longer PD-equivalent, but we would expect it to be a closer approximation of the original graph filtration than many other models in the literature.
\end{remark}

\newpage
\section{Pseudocode of the Local Generators}\label{app:algorithms}

Here, we provide the full pseudocode for the three local generators of Section~\ref{sec::local}. Algorithm~\ref{alg:pd_to_graph} is the general-filtration generator. In Algorithm~\ref{alg:vertex_based_gen}, the changes with respect to Algorithm~\ref{alg:pd_to_graph} are shown in \textcolor{violet}{violet}. Later, in Algorithm~\ref{alg:degree_based_gen}, the further changes with respect to Algorithm~\ref{alg:vertex_based_gen} are shown in \textcolor{teal}{teal}.

\begin{algorithm}[H]
\small
\caption{\textbf{General filtration} generator. PD-equivalent under any general filtration.}
\label{alg:pd_to_graph}
\KwIn{$(\mathcal{P}_0,\mathcal{P}_1)$ persistence pairs from a hidden filtration}
\KwOut{$H, f_V, f_E$ specifying the filtration of the graph $H$, or \textnormal{\textsc{Fail}}}
\SetKwFor{For}{for}{}{}
\SetKwFor{ForEach}{foreach}{}{}
\textbf{Pre-process:} $B_0(k)\!=\!\{(k,d)\!\in\!\mathcal{P}_0\}$, $D_0(k)\!=\!\{(b,k)\!\in\!\mathcal{P}_0\}$, $B_1(k)\!=\!|\{(k,\infty)\!\in\!\mathcal{P}_1\}|$\;
\textbf{Initialize:} $H\!\gets\!\emptyset$; union-find $\mathcal{U}$; per-node arrays $\mathbf{e}$ (exists), $\mathbf{c}$ (comp-id), $\mathbf{b}$ (comp-birth); adjacency $A$\;
\For{$k\!=\!0,\dots,n$\quad\normalfont\textbf{(filtration steps)}}{
    \ForEach{$(k,d)\!\in\!B_0(k)$\quad\normalfont\textbf{(Vertex Add.)}}{
        add $v$; $f_V(v)\!\gets\!k$; $\mathbf{e}[v]\!\gets\!1$; $\mathbf{c}[v]\!\gets\!v$; $\mathbf{b}[v]\!\gets\!k$\;
    }
    \ForEach{$(b,k)\!\in\!\textit{shuffle}(D_0(k))$\ \normalfont\textbf{(Comp.\ Merge)}}{
        $I_{\mathrm{m}}[i,j]\!\gets\!\mathbf{e}[i]\!\wedge\!\mathbf{e}[j]\!\wedge\!(\mathbf{b}[i]\!=\!b)\!\wedge\!(\mathbf{b}[j]\!\le\!b)\!\wedge\!(\mathbf{c}[i]\!\neq\!\mathbf{c}[j])\!\wedge\!\neg A[i,j]$\;
        $(u,v)\!\sim\!\mathrm{Unif}(I_{\mathrm{m}})$;\, $A[u,v]\!\gets\!1$;\, $f_E(u,v)\!\gets\!k$;\, $\mathcal{U}.\mathrm{union}(u,v)$;\, $\mathbf{b}[\mathcal{U}.\mathrm{find}(u)]\!\gets\!\min(\mathbf{b}[u],\mathbf{b}[v])$\;
    }
    \For{$j\!=\!1,\dots,B_1(k)$\quad\normalfont\textbf{(Cycle Add.)}}{
        $I_{\mathrm{c}}[i,j]\!\gets\!\mathbf{e}[i]\!\wedge\!\mathbf{e}[j]\!\wedge\!(\mathbf{c}[i]\!=\!\mathbf{c}[j])\!\wedge\!\neg A[i,j]$\;
        \lIf{$I_{\mathrm{c}}\!=\!\mathbf{0}$}{\Return{\textnormal{\textsc{Fail}}}}
        $(u,v)\!\sim\!\mathrm{Unif}(I_{\mathrm{c}})$;\, $A[u,v]\!\gets\!1$;\, $f_E(u,v)\!\gets\!k$\;
    }
}
\Return $H, f_V, f_E$\;
\end{algorithm}

\begin{algorithm}[H]
\small
\caption{\textbf{Vertex-based filtration} generator. Changes with respect to Alg.~\ref{alg:pd_to_graph} are in \textcolor{violet}{violet}.}
\label{alg:vertex_based_gen}
\KwIn{$(\mathcal{P}_0,\mathcal{P}_1)$ persistence pairs from a hidden \textcolor{violet}{vertex-based} filtration}
\KwOut{$H, f_V, f_E$ specifying the \textcolor{violet}{vertex-based} filtration of the graph $H$, or \textnormal{\textsc{Fail}}}
\SetKwFor{For}{for}{}{}
\SetKwFor{ForEach}{foreach}{}{}
\textbf{Pre-process:} $B_0(k)\!=\!\{(k,d)\!\in\!\mathcal{P}_0\}$, $D_0(k)\!=\!\{(b,k)\!\in\!\mathcal{P}_0\}$, $B_1(k)\!=\!|\{(k,\infty)\!\in\!\mathcal{P}_1\}|$\;
\textbf{Initialize:} $H\!\gets\!\emptyset$; union-find $\mathcal{U}$; per-node arrays $\mathbf{e}$ (exists), $\mathbf{c}$ (comp-id), $\mathbf{b}$ (comp-birth)\textcolor{violet}{, $\mathbf{n}$ (newness)}; adjacency $A$\;
\For{$k\!=\!0,\dots,n$\quad\normalfont\textbf{(filtration steps)}}{
    \textcolor{violet}{$\mathbf{n}\!\gets\!\mathbf{0}$}\;
    \ForEach{$(k,d)\!\in\!B_0(k)$\quad\normalfont\textbf{(Vertex Add.)}}{
        add $v$; $f_V(v)\!\gets\!k$; $\mathbf{e}[v]\!\gets\!1$; $\mathbf{c}[v]\!\gets\!v$; $\mathbf{b}[v]\!\gets\!k$\textcolor{violet}{; $\mathbf{n}[v]\!\gets\!1$}\;
    }
    \If{\textcolor{violet}{$D_0(k)\!\neq\!\emptyset$}\quad\normalfont\textcolor{violet}{\textbf{(Min.\ Vertex Pivot)}}}{
        \textcolor{violet}{$b_{\min}\!\gets\!\min\{b\!:\!(b,k)\!\in\!D_0(k)\}$}\;
        \textcolor{violet}{$\beta\!=\!b_{\min}\!-\!1$ if all components with birth $b_{\min}$ are dying at $k$, else $\beta\!=\!b_{\min}$}\;
        \textcolor{violet}{$I_{\mathrm{p}}[i,j]\!\gets\!(\mathbf{b}[i]\!\le\!\beta)\!\wedge\!\neg\mathbf{n}[i]\!\wedge\!\mathbf{n}[j]\!\wedge\!\neg A[i,j]$}\;
        \textcolor{violet}{$(u,v)\!\sim\!\mathrm{Unif}(I_{\mathrm{p}})$;\, $A[u,v]\!\gets\!1$;\, $f_E(u,v)\!\gets\!k$;\, $\mathcal{U}.\mathrm{union}(u,v)$;\, $\mathbf{b}[\mathcal{U}.\mathrm{find}(u)]\!\gets\!\min(\mathbf{b}[u],\mathbf{b}[v])$}\;
    }
    \ForEach{$(b,k)\!\in\!\textit{shuffle}(D_0(k))$\ \normalfont\textbf{(Comp.\ Merge)}}{
        $I_{\mathrm{m}}[i,j]\!\gets\!\mathbf{e}[i]\!\wedge\!\mathbf{e}[j]\!\wedge\!(\mathbf{b}[i]\!=\!b)\!\wedge\!(\mathbf{b}[j]\!\le\!b)\!\wedge\!(\mathbf{c}[i]\!\neq\!\mathbf{c}[j])\!\wedge\!\neg A[i,j]\textcolor{violet}{\!\wedge\!(\mathbf{n}[i]\!\vee\!\mathbf{n}[j])}$\;
        $(u,v)\!\sim\!\mathrm{Unif}(I_{\mathrm{m}})$;\, $A[u,v]\!\gets\!1$;\, $f_E(u,v)\!\gets\!k$;\, $\mathcal{U}.\mathrm{union}(u,v)$;\, $\mathbf{b}[\mathcal{U}.\mathrm{find}(u)]\!\gets\!\min(\mathbf{b}[u],\mathbf{b}[v])$\;
    }
    \For{$j\!=\!1,\dots,B_1(k)$\quad\normalfont\textbf{(Cycle Add.)}}{
        $I_{\mathrm{c}}[i,j]\!\gets\!\mathbf{e}[i]\!\wedge\!\mathbf{e}[j]\!\wedge\!(\mathbf{c}[i]\!=\!\mathbf{c}[j])\!\wedge\!\neg A[i,j]\textcolor{violet}{\!\wedge\!(\mathbf{n}[i]\!\vee\!\mathbf{n}[j])}$\;
        \lIf{$I_{\mathrm{c}}\!=\!\mathbf{0}$}{\Return{\textnormal{\textsc{Fail}}}}
        $(u,v)\!\sim\!\mathrm{Unif}(I_{\mathrm{c}})$;\, $A[u,v]\!\gets\!1$;\, $f_E(u,v)\!\gets\!k$\;
    }
}
\Return $H, f_V, f_E$\;
\end{algorithm}

\newpage
\begin{algorithm}[H]
\small
\caption{\textbf{Degree filtration} generator. Changes with respect to Alg.~\ref{alg:pd_to_graph} are in \textcolor{violet}{violet}; further changes with respect to Alg.~\ref{alg:vertex_based_gen} are in \textcolor{teal}{teal}.}
\label{alg:degree_based_gen}
\KwIn{$(\mathcal{P}_0,\mathcal{P}_1)$ persistence pairs from a hidden \textcolor{teal}{degree-based} filtration}
\KwOut{$H, f_V, f_E$ specifying the \textcolor{teal}{degree-based} filtration of the graph $H$, or \textnormal{\textsc{Fail}}}
\SetKwFor{For}{for}{}{}
\SetKwFor{ForEach}{foreach}{}{}
\textbf{Pre-process:} $B_0(k)\!=\!\{(k,d)\!\in\!\mathcal{P}_0\}$, $D_0(k)\!=\!\{(b,k)\!\in\!\mathcal{P}_0\}$, $B_1(k)\!=\!|\{(k,\infty)\!\in\!\mathcal{P}_1\}|$\;
\textbf{Initialize:} $H\!\gets\!\emptyset$; union-find $\mathcal{U}$; per-node arrays $\mathbf{e}$ (exists), $\mathbf{c}$ (comp-id), $\mathbf{b}$ (comp-birth)\textcolor{violet}{, $\mathbf{n}$ (newness)}\textcolor{teal}{, $\boldsymbol{\kappa}$ (capacity)}; adjacency $A$\;
\For{$k\!=\!0,\dots,n$\quad\normalfont\textbf{(filtration steps)}}{
    \textcolor{violet}{$\mathbf{n}\!\gets\!\mathbf{0}$}\;
    \ForEach{$(k,d)\!\in\!B_0(k)$\quad\normalfont\textbf{(Vertex Add.)}}{
        add $v$; $f_V(v)\!\gets\!k$; $\mathbf{e}[v]\!\gets\!1$; $\mathbf{c}[v]\!\gets\!v$; $\mathbf{b}[v]\!\gets\!k$\textcolor{violet}{; $\mathbf{n}[v]\!\gets\!1$}\textcolor{teal}{; $\boldsymbol{\kappa}[v]\!\gets\!k$}\;
    }
    \If{\textcolor{violet}{$D_0(k)\!\neq\!\emptyset$}\quad\normalfont\textcolor{violet}{\textbf{(Min.\ Vertex Pivot)}}}{
        \textcolor{violet}{$b_{\min}\!\gets\!\min\{b\!:\!(b,k)\!\in\!D_0(k)\}$}\;
        \textcolor{violet}{$\beta\!=\!b_{\min}\!-\!1$ if all components with birth $b_{\min}$ are dying at $k$, else $\beta\!=\!b_{\min}$}\;
        \textcolor{violet}{$I_{\mathrm{p}}[i,j]\!\gets\!(\mathbf{b}[i]\!\le\!\beta)\!\wedge\!\neg\mathbf{n}[i]\!\wedge\!\mathbf{n}[j]\!\wedge\!\neg A[i,j]$}$\textcolor{teal}{\!\wedge\!(\boldsymbol{\kappa}[i]\!>\!0\!\wedge\!\boldsymbol{\kappa}[j]\!>\!0)}$\;
        \textcolor{teal}{\lIf{$I_{\mathrm{p}}\!=\!\mathbf{0}$}{\Return{\textnormal{\textsc{Fail}}}}}
        \textcolor{violet}{$(u,v)\!\sim\!\mathrm{Unif}(I_{\mathrm{p}})$;\, $A[u,v]\!\gets\!1$;\, $f_E(u,v)\!\gets\!k$;\, $\mathcal{U}.\mathrm{union}(u,v)$;\, $\mathbf{b}[\mathcal{U}.\mathrm{find}(u)]\!\gets\!\min(\mathbf{b}[u],\mathbf{b}[v])$}\textcolor{teal}{;\, $\boldsymbol{\kappa}[u]\!\mathrel{-}\!=\!1$;\, $\boldsymbol{\kappa}[v]\!\mathrel{-}\!=\!1$}\;
    }
    \ForEach{$(b,k)\!\in\!\textit{shuffle}(D_0(k))$\ \normalfont\textbf{(Comp.\ Merge)}}{
        $I_{\mathrm{m}}[i,j]\!\gets\!\mathbf{e}[i]\!\wedge\!\mathbf{e}[j]\!\wedge\!(\mathbf{b}[i]\!=\!b)\!\wedge\!(\mathbf{b}[j]\!\le\!b)\!\wedge\!(\mathbf{c}[i]\!\neq\!\mathbf{c}[j])\!\wedge\!\neg A[i,j]$\newline\hspace*{2em}$\textcolor{violet}{\wedge\!(\mathbf{n}[i]\!\vee\!\mathbf{n}[j])}\textcolor{teal}{\!\wedge\!(\boldsymbol{\kappa}[i]\!>\!0\!\wedge\!\boldsymbol{\kappa}[j]\!>\!0)}$\;
        \textcolor{teal}{\lIf{$I_{\mathrm{m}}\!=\!\mathbf{0}$}{\Return{\textnormal{\textsc{Fail}}}}}
        $(u,v)\!\sim\!\mathrm{Unif}(I_{\mathrm{m}})$;\, $A[u,v]\!\gets\!1$;\, $f_E(u,v)\!\gets\!k$;\, $\mathcal{U}.\mathrm{union}(u,v)$;\, $\mathbf{b}[\mathcal{U}.\mathrm{find}(u)]\!\gets\!\min(\mathbf{b}[u],\mathbf{b}[v])$\textcolor{teal}{;\, $\boldsymbol{\kappa}[u]\!\mathrel{-}\!=\!1$;\, $\boldsymbol{\kappa}[v]\!\mathrel{-}\!=\!1$}\;
    }
    \For{$j\!=\!1,\dots,B_1(k)$\quad\normalfont\textbf{(Cycle Add.)}}{
        $I_{\mathrm{c}}[i,j]\!\gets\!\mathbf{e}[i]\!\wedge\!\mathbf{e}[j]\!\wedge\!(\mathbf{c}[i]\!=\!\mathbf{c}[j])\!\wedge\!\neg A[i,j]$\newline\hspace*{2em}$\textcolor{violet}{\wedge\!(\mathbf{n}[i]\!\vee\!\mathbf{n}[j])}\textcolor{teal}{\!\wedge\!(\boldsymbol{\kappa}[i]\!>\!0\!\wedge\!\boldsymbol{\kappa}[j]\!>\!0)}$\;
        \lIf{$I_{\mathrm{c}}\!=\!\mathbf{0}$}{\Return{\textnormal{\textsc{Fail}}}}
        $(u,v)\!\sim\!\mathrm{Unif}(I_{\mathrm{c}})$;\, $A[u,v]\!\gets\!1$;\, $f_E(u,v)\!\gets\!k$\textcolor{teal}{;\, $\boldsymbol{\kappa}[u]\!\mathrel{-}\!=\!1$;\, $\boldsymbol{\kappa}[v]\!\mathrel{-}\!=\!1$}\;
    }
}
\Return $H, f_V, f_E$\;
\end{algorithm}

\newpage
\section{Proof for Correctness of Encoding (Section~\ref{sec:global})}
\label{app:cp_proofs}

In this section, we will prove the correctness of the encoder in Section~\ref{sec:global} (Theorem~\ref{thm:cp_correct}). \textbf{We first prove the soundness of the encoder}, and we will break down the steps in a sequence of lemmas.\\

We first prove Lemma~\ref{lem::ancestry} as a preliminary step. 

\begin{lemma}\label{lem::ancestry}
Constraints $\mathcal{C}_1$ to $\mathcal{C}_6$ implies 
\noindent\textbf{$\mathcal{C}_7$ \;(ancestry)}: Each slot at stage $k{-}1$ has exactly one successor at stage $k$, and the
successor inherits all its vertices.
\end{lemma}

\begin{proof}[Proof of Lemma~\ref{lem::ancestry}]
Suppose for contradiction that $\mathcal{C}_{7}$ does not hold, then there exists two vertices $v, w$ at slot $c$ at stage $k-1$ such that $v$ and $w$ belong in two different slots $d$ and $d'$ at stage $k$ (Here we implicitly used $\mathcal{C}_{2}$).\\

By $\mathcal{C}_{4}$, we know there is a sequence of edges connecting $v$ to $w$ and write $v \to v_1 \to ... \to v_n \to w$ the path. Since $v$ and $w$ belong in two different component, there must exist a pair of adjacent vertices in the path here that belong in two different slots.\\

Thus, we without loss reduce to the case where $v < w$ share an edge, so $M_{vw} = 1$. Now by $\mathcal{C}_{3}$, we know that at $k+1$, we must have
\[M_{vw} = 1 \leq 1 - |x_{k+1,cv} - x_{k+1,cw}|\]
for all $c \in [\operatorname{cc}_{k}]$. However, since $v$ and $w$ belong in two different slots, there is a slot $d$ such that $x_{k+1, dv} = 1$ and $x_{k+1, dw} = 0$, which implies
\[M_{vw} = 1 \leq 1 - |1 - 0| = 0, \]
which is a contradiction.
\end{proof}

\begin{definition}\label{def::matrix_graph_translation}
Let $\{M_{ij}\}_{i < j}$ be the data specifying a graph $H$, whose vertex set is still $V = [n]$. We write $M(k)$ to denote the collection of $M_{ij}$'s such that $\Phi_{M}(i), \Phi_{M}(j) \leq k$. In other words, $M(k)$ specifies the induced subgraph $H_k$ of $H$ of vertices whose value is $\leq k$.\\

On the other hand, given the additional data $x_{kcv}$ produced by the constraints. For each $k$, we can define $H_k^{\operatorname{cp}}$ as the subgraph of $H$ whose vertices are the active vertices $V_k \subseteq V$, an edge $e = (v, w) \in H_k^{\operatorname{cp}}$ if $v, w$ belongs to the same slot at stage $k$ and $M_{vw} = 1$.
\end{definition}

We also need the following lemma.
\begin{lemma}\label{lem::properties_encoder}
Following the notation in Definition~\ref{def::matrix_graph_translation}, the data of $M_{ij}$ and $x_{kcv}$ produced by constraints $\mathcal{C}_{1}, ..., \mathcal{C}_{6}$ satisfies the following:
\begin{enumerate}
    \item The active vertices at stage $k$ are all vertices $v$ with $\Phi_{M}(v) = a_v \leq k$.
    \item Every slot at stage $k$ has to have at least 1 vertex.
    \item The number of slots at stage $k$ is the number of components of $H_k$.
    \item The vertices in each slot $c$ at stage $k$ correspond to the vertices in each component of $H_k$.
\end{enumerate}
\end{lemma}

\begin{proof}
    Throughout the proof, we implicitly use $\mathcal{C}_2$ to ensure that every vertex belongs in exactly 1 slot.\\
    
    For the first statement, we first note that $\Phi_{M}(v) = a_v$ is specified by $\mathcal{C}_1$, and the definition of $V_k$ (the active vertex set at stage $k$) specifies that they are examining vertices $v$ with $a_v \leq k$. In particular, we see that $V_k$ and $H_k$ shares the same vertex set.\\

    For the second statement. The number of slots in $cc_k$ is the number $N$ of pairs $(b, d)$ there are such that $b \leq k < d$ from the input PDs. Now we have the equality
    \[N = \sum_{b \leq k} \# \{(b,d)\ |\ d > k\} = \sum_{b \leq k} \sum_{c \in [cc_k]} \mathbf{1}[e_{kcb}],\]
    where the second equality follows from $\mathcal{C}_{6}$. For a fixed $c$ at stage $k$, clearly $e_{kcb}$ can only be $1$ for at most one choice of $b$. Thus, we see that
    \[\sum_{b \leq k} \sum_{c \in [cc_k]} \mathbf{1}[e_{kcb}] \leq |cc_k| = N\]
    with equality if and only if for each $c$, $e_{kcb} = 1$ for some $b$. The fact that $e_{kcb} = 1$ ensures the existence of a vertex in slot $c$ at stage $k$.\\

    For the third statement. Suppose the number of components in $H_k$ is less than $|cc_k|$. Since every slot has to have at least 1 vertex, there exists two vertices $v, w$ in a component of $H_k$ such that $v, w$ belongs to two different slots. By choosing a path if necessary as before, we can without loss pretend that $v, w$ are adjacent to each other. Since $v, w$ are in different slots, there is a slot $c$ such that $x_{kcv} = 1$ and $x_{kcw} = 0$. Then we have that by $\mathcal{C}_{3}$,
    \[M_{vw} = 1 \leq 1 - |x_{kcv} - x_{kcw}| = 1 - |1-0| = 0,\]
    which is a contradiction.\\

    Now suppose the number of components in $H_k$ is greater than $|cc_k|$. Since each component of $H_k$ has at least 1 vertex, it must be the case that two vertices in different components are assigned to the same slot, which is impossible by $\mathcal{C}_{4}$. This proves the third statement.\\

    For the fourth statement, by $\mathcal{C}_{4}$ and (2) we know the vertices in each slot $c$ at stage $k$ is connected and non-empty, so each slot must belong to some component in $H_k$. By $\mathcal{C}_3$ and similar argument as before, this assignment is an injection. By (3), we know this map is a bijection. Thus, we know that the induced subgraph on vertices of $c \in \operatorname{cc}_k$ is a subgraph of a component of $H_k$. The only ambiguity is if $H_k$ contains a vertex $v_+$ not in slot $c$. However, $\mathcal{C}_2$ requires $v_+$ to be in some slot, if it is in any slot other than $c$, the assignment is no longer an injection. This concludes the fourth statement.
\end{proof}

\textbf{Recall:} In Definition~\ref{def::filtration_scheme}, we write $H_{\bullet}^{\Phi}$ to denote the sequence of subgraphs specified by the function $\Phi_{H}$. 

Using Lemma~\ref{lem::ancestry} and Lemma~\ref{lem::properties_encoder}, we conclude that:
\begin{corollary}\label{cor::constraint_agrees_with_filtration}
$H_k = H^{\operatorname{cp}}_k$ for all $k$. Since the $H_k$'s form the subgraphs in the vertex based filtration function $\Phi_{M} = \Phi_{H}$, this shows that
\[\mathcal{P}(\emptyset \subseteq H_{0}^{\operatorname{cp}} \subseteq ... \subseteq H_n^{cp} = H) = \mathcal{P}(H_{\bullet}^{\Phi}). \] 
\end{corollary}

We now prove the soundness of the encoding in Theorem~\ref{thm:cp_correct}. (Note the notation change from $H$ to $H(\sigma)$ below)
\begin{theorem}[Soundness of Encoding]\label{thm::soundness}
Given a fixed $(\mathcal{P}_0, \mathcal{P}_1)$ from a hidden graph $G$. Let $\mathcal{C} = \{\mathcal{C}_1, ..., \mathcal{C}_6\}$ and $\mathcal{F}$ the set of assignment satisfying assignment to $\mathcal{C}$. For every $\sigma \in \mathcal{F}$, the graph $H(\sigma)$ specified has $\mathcal{P}(H(\sigma)_{\bullet}^{\Phi}) = (\mathcal{P}_0, \mathcal{P}_1)$.
\end{theorem}

\begin{proof}
By Corollary~\ref{cor::constraint_agrees_with_filtration}, we have that
\[\mathcal{P}(H(\sigma)_{\bullet}^{\Phi}) = \mathcal{P}(H(\sigma)_{\bullet}^{cp} \coloneqq \emptyset \subseteq H(\sigma)_{0}^{\operatorname{cp}} \subseteq ... \subseteq H(\sigma)_n^{cp} = H(\sigma)). \]
It suffices to prove the right hand side agrees with $(\mathcal{P}_0, \mathcal{P}_1)$.\\

For the ease of notation, we will write $\mathcal{P}(\emptyset \subseteq H(\sigma)_{0}^{\operatorname{cp}} \subseteq ... \subseteq H(\sigma)_n^{cp} = H(\sigma)) = (\mathcal{P}_0^{cn}, \mathcal{P}_1^{cn})$.\\

\textbf{For 1-dimensional PDs:} $\mathcal{C}_{5}$ ensures that the number of cycles born at each stage $k$ for $H(\sigma)_{\bullet}^{cp}$ is the same as the number of cycles born at each stage $k$ in $\mathcal{P}_1$.\\

Note that this also implies $H(\sigma)_k^{cp}$ and $G^{\Phi}_k$ has the same number of edges. As by construction they already have the same number of vertices and components.

\textbf{For 0-dimensional PDs:} It suffices to prove the following: For each fixed $(b, d)$, the number $N(b, d)$ of pairs of the form $(b, d)$ in $\mathcal{P}_0$ and $\mathcal{P}_0^{cn}$ are the same.\\

For a fixed $b < d$, $N(b, d)$ can be computed as
\[N(b, d) = \# \{\text{components born at $b$ and still alive at $d-1$}\} - \# \{\text{components born at $b$ and still alive at $d$}\}.\]
This quantity is the same for both sides by $\mathcal{C}_6$.\\

For $b = d$, $N(d, d)$ can be computed as
\[N(d, d) = \# \{\text{vertex deaths at stage $d$}\} - \sum_{b < d} N(b, d).\]
Since $H(\sigma)_{k}^{cp}$ and $G_{k}^{\Phi}$ has the same number of edges, vertices, components, and cycles for all $k$, we know that both sides have the same number $N$ of vertex deaths that occurred at stage $d$. We established above that the quantity $N(b, d)$ is the same for both sides when $b < d$. Thus, it follows that the quantity $N(d, d)$ is the same for both sides. This concludes the proof.
\end{proof}

Now we will prove completeness of the encoder.
\begin{theorem}[Completeness of Encoding]\label{thm::completeness}
Let $K$ be a graph such that $\mathcal{P}(K_{\bullet}^{\Phi}) = (\mathcal{P}_0, \mathcal{P}_1)$, then there exists an assignment $\sigma \in \mathcal{F}$ such that $H(\sigma) = K$. 
\end{theorem}

\begin{proof}
Fix an ordering of the vertices of $K$ in the list $[n]$ and an ordering of the components of $K_{k}^{\Phi}$ at each stage $k$. We give the following assignment to the variables:
\begin{enumerate}
    \item $M_{ij} = 1$ if $(i, j) \in E(K)$ and $i < j$. $M_{ij} = 0$ if neither $(i, j)$ or $(j, i)$ is in $E(K)$. 
    \item $x_{kcv} = 1$ if vertex $v$ is in component $c$ at $K_k$ and $0$ otherwise. We can fill in all variables because $K_{\bullet}^{\Phi}$ and $\mathcal{P}_0$ have the same number of components at each stage.
    \item We create $\pi_{kcij}$ and $\ell_{kcv}$ by choosing any spanning tree of each component $c$ at $K_k$ coming out of the root node with minimum index.
    \item $r_{kcv}$ and $e_{kcb}$ are dependent on the previous variables.
\end{enumerate}
For the precomputed quantity, we have the choice of making an arbitrary bijection, here we associated each $i$ with its value $\Phi_{K}(i)$. We can do this because the assignment corresponds exactly birth times arising in $\mathcal{P}_0$ by the hypothesis.\\

We now check that they satisfy the constraints $\mathcal{C}_1, ..., \mathcal{C}_6$ one by one.
\begin{enumerate}
    \item For $\mathcal{C}_1$: This is automatic from how we chose the bijection.
    \item For $\mathcal{C}_2$: A vertex $v$ in $K_k$ cannot belong to more than two components, so it follows that $x_{kcv}$ is $1$ for exactly one choice of $c$ at a fixed $(k, v)$. Thus it follows that $\sum_{c \in [cc_k]} x_{kcv} = 1$.
    \item For $\mathcal{C}_3$: If $u$ and $v$ does not share an edge, then the inequality is always true. Thus, we consider when $u, v$ shares an edge. By construction here $u, v$ has to be in the same component, so there is a $c \in [cc_k]$ such that $x_{kcu} = x_{kcv} = 1$. Otherwise, for any $c' \neq c$ in $[cc_k]$, we have that $x_{kc'u} = x_{kc'v} = 0$. In either cases we have that $1 \leq 1$.
    \item For $\mathcal{C}_4$: Clearly the spanning tree made at each component satisfies the constraints given.
    \item For $\mathcal{C}_5$: This follows from a standard fact in algebraic topology (see for instance, Theorem 2.44 of \cite{hatcher2002algebraic}) that for a graph $A$,
    \[\# V(A) - \# E(A) = \# \text{components of } A - \# \text{independent cycles of } A.\]
    This is called the invariance of the Euler characteristics. In particular, $\mathcal{C}_{5}$ is satisfied by rewriting the equation above to solve for the cycle counts.
    \item For $\mathcal{C}_6$: The vertex with minimum filtration value in a component is equal to the birth time of said component. Thus, the count holds for $K^{\Phi}_{\bullet}$ as well.
\end{enumerate}
\end{proof}

Finally, we are able to prove Theorem~\ref{thm:cp_correct}. To recall, this is the statement that:
\begin{theoremm}
Let $\mathcal{C} = \{\mathcal{C}_1, \dots, \mathcal{C}_6\}$, $\mathcal{F}$
the set of satisfying assignments to $\mathcal{C}$, and
$\mathcal{H}(\mathcal{P}_0, \mathcal{P}_1)$ the set of graphs $H$ with
$\mathcal{P}(H_\bullet^{\Phi}) = (\mathcal{P}_0, \mathcal{P}_1)$.
Define $H(\sigma)$ from $\sigma \in \mathcal{F}$ by
$E(H(\sigma)) = \{\{i,j\} : \sigma(M_{ij}) = 1\}$.
Then:
\begin{enumerate}[noitemsep, topsep=0pt,leftmargin=*]
\item \textbf{Soundness:} $H(\sigma) \in \mathcal{H}(\mathcal{P}_0,\mathcal{P}_1)$
  for every $\sigma \in \mathcal{F}$.
\item \textbf{Completeness:} every $H \in \mathcal{H}(\mathcal{P}_0,\mathcal{P}_1)$
  admits at least one $\sigma \in \mathcal{F}$ with $H(\sigma) = H$.
\item \textbf{Minimality for Degree:} When $\Phi = \operatorname{deg}$, for each $i \in \{1, \dots, 6\}$, there exist a
  target $(\mathcal{P}_0, \mathcal{P}_1)$ and an assignment $\sigma$ that
  satisfies $\mathcal{C} \setminus \{\mathcal{C}_i\}$ but with
  $H(\sigma) \notin \mathcal{H}(\mathcal{P}_0, \mathcal{P}_1)$.
\end{enumerate}
\end{theoremm}

\begin{proof}[Proof of Theorem~\ref{thm:cp_correct}]
    Soundness is proven in Theorem~\ref{thm::soundness}, and completeness is proven in Theorem~\ref{thm::completeness}. For minimality of degree, we removed each of the six constraints and generated an assignment that does not satisfy the removed constraints. The counter-examples are written in Appendix~\ref{sec:appendix-counterexamples} and explained there.
\end{proof}

\newpage
\section{Per-Constraint Counterexamples for Degree-Based PD}
\label{sec:appendix-counterexamples}

For each necessary constraint $\mathcal{C}_i$, we exhibit the a target
$G$ we found in our search where, with $\mathcal{C}_i$ disabled, the constraint solver
returns a graph $H$ with $\mathrm{PH}(H) \neq \mathrm{PH}(G)$. Throughout this section, we will always use degree-based filtration and consider degree-based persistence diagrams.

Counterexamples were found by
enumerating up to $80$ distinct CP solutions per (target, dropped-constraint)
pair across all named graphs on $|V| \leq 6$ and many random samples for
$|V| \leq 9$. Figure~\ref{fig::counter_example} displays the target graph and the counterexample produced.

In this section, we use $(b, d)^n$ to indicate the tuple $(b, d)$ has multiplicity $n$. 

\begin{remark}
    Before we proceed, we make a comment on how the $\mathcal{C}_1$ constraint is verified. Recall that in the actual generation process, the underlying graph $G$ is hidden from us, so we only have access to the persistence tuples. Therefore, $\mathcal{C}_1$ is satisfied as long as one can match the output $\{\Phi_{H}(v), v \in H\}$ to the multi-set of birth times that we could read off from $\mathcal{P}_0$. In particular, for degree, this means that the degree of the literal index $i$ in $H$ need not be the degree of the literal index $i$ in $G$. Of course, after reordering if necessary, we can make this true.
\end{remark}

\subsection{\texorpdfstring{$\mathcal{C}_1$}{C1} (degree)}
\textbf{Target.} $G = 2K_2$: $V = \{0,1,2,3\}$, $E_G = \{(0,1),(2,3)\}$.
Degrees $(1,1,1,1)$. $\mathcal{P}_0(G) = \{(1,1)^2,(1,\infty)^2\}$,
$\mathcal{P}_1(G) = \emptyset$.

\textbf{Counterexample.} $H$: $E_H = \{(0,1),(0,3)\}$.

\textbf{Witness.}
\begin{align*}
& M_{01} = M_{03} = 1, \quad M_{ij} = 0 \text{ otherwise} \\
& x_{1,0,0} = x_{1,0,1} = x_{1,0,3} = 1, \quad x_{1,1,2} = 1 \\
& \pi_{1,0,0,1} = \pi_{1,0,0,3} = 1, \quad
  \ell_{1,0,0} = 0,\ \ell_{1,0,1} = \ell_{1,0,3} = 1.
\end{align*}

\textbf{Verification.}
\begin{itemize}
    \item $\mathcal{C}_2$: each vertex in one slot.
    \item $\mathcal{C}_3$: edges $(0,1),(0,3)$ within slot 0.
    \item $\mathcal{C}_4$: slot 0 spans the star $K_{1,2}$.
    \item $\mathcal{C}_5$: $2 - 4 + 2 = 0 = \mathrm{cy}_1$.
    \item $\mathcal{C}_6$: $\sum_c \mathbf{1}[e_{1,c,1}] = 2 = \alpha(1,1)$.
\end{itemize}

\textbf{Why $\mathcal{C}_1$ fails.}
$\deg_H(0) = 2$, $\deg_H(2) = 0$.

\textbf{PH Comparison.} $\mathrm{PH}(H) = (\{(0,\infty),(1,2),(1,\infty),(2,2)\}, \emptyset)
\neq \mathrm{PH}(G)$.

\subsection{\texorpdfstring{$\mathcal{C}_2$}{C2} (partition)}
\textbf{Target.} $G = P_6$: $E_G = \{(0,1),(1,2),(2,3),(3,4),(4,5)\}$,
degrees $(1,2,2,2,2,1)$.
$\mathcal{P}_0(G) = \{(1,2),(1,\infty),(2,2)^4\}$,
$\mathcal{P}_1(G) = \emptyset$.
Decoded: $cc_1 = 2$, $cc_2 = 1$, $\alpha(1,1)=2$, $\alpha(2,1)=1$,
$\alpha(2,2)=0$.

\textbf{Counterexample.} $H = P_3 \sqcup K_3$:
$E_H = \{(0,5),(1,5),(2,3),(2,4),(3,4)\}$. This is the disjoint union of the path graph $0\text{-}5\text{-}1$ and a triangle on $\{2,3,4\}$.

\textbf{Witness (relaxed partition; vertices $\{2,3,4\}$ assigned to no slot at $k=2$).}
\begin{align*}
& M_{05}=M_{15}=M_{23}=M_{24}=M_{34}=1 \\
& x_{1,0,0}=1,\ x_{1,1,1}=1 \\
& x_{2,0,0}=x_{2,0,1}=x_{2,0,5}=1, \quad x_{2,0,2}=x_{2,0,3}=x_{2,0,4}=0 \\
& \pi_{2,0,0,5} = \pi_{2,0,5,1} = 1, \quad \ell_{2,0,0}=0,\ \ell_{2,0,5}=1,\ \ell_{2,0,1}=2.
\end{align*}

\textbf{Verification.}
\begin{itemize}
\item $\mathcal{C}_1$: the multi-set of degrees agree.
\item $\mathcal{C}_3$: edges $(0,5),(1,5)$ in slot 0; $(2,3),(2,4),(3,4)$ outside any slot.
\item $\mathcal{C}_4$: slot 0 spans path $0\text{-}5\text{-}1$.

\item $\mathcal{C}_5$: $|E_2|-|V_2|+cc_2 = 5-6+1 = 0 = \mathrm{cy}_2$.
\item $\mathcal{C}_6$: slot 0 elder-of-1 (contains $0,1$, no younger); not elder-of-2
(contains 5 but also younger). Sums match $\alpha(2,1)=1$, $\alpha(2,2)=0$.
\end{itemize}

The up-shot here is that since vertices $2, 3, 4$ are not registered in the component, the pass for $\mathcal{C}_{5}$ and $\mathcal{C}_6$ were only conducted on the path graph $0-5-1$. Thus, the assignment was satisfied.

\textbf{Why $\mathcal{C}_2$ fails.}
$\sum_c x_{2,c,2} = 0 \neq 1$, similarly for vertices 3 and 4.

\textbf{PH Comparison.} $\mathrm{PH}(H) = (\{(1,2),(1,\infty),(2,2)^3,(2,\infty)\},\{(2,\infty)\})
\neq \mathrm{PH}(G)$.

\subsection{\texorpdfstring{$\mathcal{C}_3$}{C3} (edge-slot consistency)}
\textbf{Target.} $G = 2K_3$:
$E_G = \{(0,1),(0,2),(1,2),(3,4),(3,5),(4,5)\}$, all degrees 2.
$\mathcal{P}_0(G) = \{(2,2)^4,(2,\infty)^2\}$, $\mathcal{P}_1(G) = \{(2,\infty)^2\}$.

\textbf{Counterexample.} $H \cong C_6$:
$E_H = \{(0,1),(0,2),(1,4),(2,3),(3,5),(4,5)\}$.

\textbf{Witness (partition respected, cross-slot edges allowed).}
\begin{align*}
& \text{slot 0} = \{0,1,2\}: x_{2,0,0}=x_{2,0,1}=x_{2,0,2}=1 \\
& \text{slot 1} = \{3,4,5\}: x_{2,1,3}=x_{2,1,4}=x_{2,1,5}=1 \\
& \pi_{2,0,0,1}=\pi_{2,0,0,2}=1, \quad \ell_{2,0,0}=0,\ \ell_{2,0,1}=\ell_{2,0,2}=1 \\
& \pi_{2,1,3,5}=\pi_{2,1,5,4}=1, \quad \ell_{2,1,3}=0,\ \ell_{2,1,5}=1,\ \ell_{2,1,4}=2.
\end{align*}

\textbf{Verification.}
\begin{itemize}
\item $\mathcal{C}_1$: the multi-set of degrees agree.
\item $\mathcal{C}_2$: clear.
\item $\mathcal{C}_4$: each slot uses 2 in-slot edges to span a path.
\item $\mathcal{C}_5$: $6-6+2 = 2 = \mathrm{cy}_2$.
\item $\mathcal{C}_6$: each slot elder-of-2; sum $= 2 = \alpha(2,2)$.
\end{itemize}

The up-shot here is that topologically the graph constructed has one component, but slot-wise there are still two slots going on. The calculations for $\mathcal{C}_5$ and $\mathcal{C}_6$ are conducted using slots rather than components, hence the assignment was satisfied.\\

\textbf{Why $\mathcal{C}_3$ fails.}
Edges $(1,4)$ and $(2,3)$ have endpoints in different slots.

\textbf{PH Comparison.} $\mathrm{PH}(H = C_6) = (\{(2,2)^5,(2,\infty)\},\{(2,\infty)\})
\neq \mathrm{PH}(G)$.

\subsection{\texorpdfstring{$\mathcal{C}_4$}{C4} (slot connectivity)}
\textbf{Target.} $G = C_6$:
$E_G = \{(0,1),(0,5),(1,2),(2,3),(3,4),(4,5)\}$.
$\mathcal{P}_0(G) = \{(2,2)^5,(2,\infty)\}$, $\mathcal{P}_1(G) = \{(2,\infty)\}$.

\textbf{Counterexample.} $H = 2K_3$:
$E_H = \{(0,3),(0,4),(1,2),(1,5),(2,5),(3,4)\}$.

\textbf{Witness (single slot covering disconnected graph).}
\[
M \text{ encodes } H, \quad x_{2,0,v}=1 \ \forall v.
\]
Spanning-tree variables $\pi, \ell$ are unconstrained (assign all zero).

\textbf{Verification.}
\begin{itemize}
\item $\mathcal{C}_1$: the multi-set of degrees agree.
\item $\mathcal{C}_2$: every vertex is in slot 0.
\item $\mathcal{C}_3$: all edges within slot 0.
\item $\mathcal{C}_6$: $6-6+1 = 1 = \mathrm{cy}_2$.
\item $\mathcal{C}_7$: slot 0 elder-of-2; sum $= 1 = \alpha(2,2)$.
\end{itemize}

\textbf{Why $\mathcal{C}_4$ fails.}
Slot 0's induced subgraph in $H$ is two disjoint triangles (so the subgraph is disconnected). No assignment of $\pi, \ell$ satisfies the spanning-tree depth constraint.

\textbf{PH Comparison.} $\mathrm{PH}(H = 2K_3) = (\{(2,2)^4,(2,\infty)^2\},\{(2,\infty)^2\})
\neq \mathrm{PH}(G)$.\\

Before we move on to $\mathcal{C}_{5}$ and $\mathcal{C}_6$.

\subsection{\texorpdfstring{$\mathcal{C}_5$}{C5} (Euler characteristics)}
We searched all named graphs and 200+ Erdős–Rényi (ER) random graph samples on $|V| \leq 7$ without
finding a graph-level counterexample. The smallest in our search appears at
$|V| = 8$.

\textbf{Target.} $G = \mathrm{ER}(8, 0.4)$ with seed $5$:
\[
E_G = \{(0,7),(1,6),(2,3),(2,6),(2,7),(3,4),(3,7),(4,6),(5,6),(5,7)\}.
\]
Degrees are $(1,1,4,3,2,2,4,4)$.
$\mathcal{P}_0(G) = \{(1,4),(1,\infty),(2,4)^2,(3,3)^2,(4,4)^2\}$,
$\mathcal{P}_1(G) = \{(4,\infty)^3\}$.

\textbf{Counterexample.} $H$:
\[
E_H = \{(0,6),(1,6),(2,4),(2,5),(3,6),(3,7),(4,5),(4,7),(5,7),(6,7)\}.
\]

\textbf{Witness.} The graph $H$ and its associated degree-based filtration specifies the $M_{ij}$'s and $x_{kcv}$'s. The spanning tree variables are done at each component. The following table records the progression of relevant data in $G$ and $H$ as the stages $k$ progress.
\begin{center}
\begin{tabular}{c|cccccc}
$k$ & $|V_k|$ & $|E_k(G)|$ & $|E_k(H)|$ & $cc_k$ & $\mathrm{cy}_k(G)$ & $\mathrm{cy}_k(H)$ \\
\hline
1 & 2 & 0 & 0 & 2 & 0 & 0 \\
2 & 4 & 0 & 0 & 4 & 0 & 0 \\
\textbf{3} & \textbf{6} & \textbf{2} & \textbf{3} & \textbf{4} & \textbf{0} & \textbf{1} \\
4 & 8 & 10 & 10 & 1 & 3 & 3 \\
\end{tabular}
\end{center}

\textbf{Why $\mathcal{C}_5$ Failed.}  Both graphs share the same per-level component structure and birth-time survival. A witness assigning each vertex to its component's slot at every level satisfies $\mathcal{C}_1$-$\mathcal{C}_4,\mathcal{C}_6$.
Only $\mathcal{C}_5$'s equation
$$|E_k| - |V_k| + cc_k = \mathrm{cy}_k$$
distinguishes the two graphs at $k=3$:

\textbf{PH Comparison.} Here we have that $\mathcal{P}_0(H) = \mathcal{P}_0(G)$, but
$\mathcal{P}_1(H) = \{(3,\infty),(4,\infty)^2\} \neq \mathcal{P}_1(G) = \{(4, \infty)^3\}$.

\subsection{\texorpdfstring{$\mathcal{C}_6$}{C6} (birth-time survival)}
\textbf{Target} $G = \mathrm{ER}(7, 0.3)$ with seed $28$:
\[
E_G = \{(0,1),(0,2),(0,4),(0,5),(1,3),(1,6),(2,3),(2,4),(2,5),(2,6),(4,5)\}.
\]
Degrees are $(4,3,5,2,3,3,2)$.
$\mathcal{P}_0(G) = \{(2,3),(2,\infty),(3,3)^2,(3,4),(4,4),(5,5)\}$,
$\mathcal{P}_1(G) = \{(4,\infty),(5,\infty)^4\}$.

\textbf{Counterexample.} $H$:
\[
E_H = \{(0,2),(0,6),(1,5),(1,6),(2,3),(2,6),(3,4),(3,5),(4,5),(4,6),(5,6)\}.
\]

\textbf{Witness.} 
The graph $H$ and its associated degree-based filtration specifies the $M_{ij}$'s and $x_{kcv}$'s. The spanning tree variables are done at each component. At every level, $|V_k|$, $|E_k|$, $cc_k$, and $\mathrm{cy}_k$ are
\emph{identical} between $G$ and $H$:
\begin{center}
\begin{tabular}{c|cccc}
$k$ & $|V_k|$ & $|E_k|$ & $cc_k$ & $\mathrm{cy}_k$ \\
\hline
2 & 2 & 0 & 2 & 0 \\
\textbf{3} & \textbf{5} & \textbf{3} & \textbf{2} & \textbf{0} \\
4 & 6 & 6 & 1 & 1 \\
5 & 7 & 11 & 1 & 5 \\
\end{tabular}
\end{center}

\textbf{Why $\mathcal{C}_6$ Failed.} The difference occurs in $\alpha(k, b)$ at $k = 3$. Recall $\alpha(3, 2)$ refers to the number of components born at $2$ that is still alive at stage $3$.
\begin{center}
\begin{tabular}{c|cccc}
 & $\alpha(3,2)$  \\
\hline
$G$ & 1  \\
$H$ & 2  \\
\end{tabular}
\end{center}

\textbf{PH Comparison.} $\mathcal{P}_0(H) = \{(2,4),(2,\infty),(3,3)^3,(4,4),(5,5)\} \neq \mathcal{P}_0(G)$, and
$\mathcal{P}_1(H) = \mathcal{P}_1(G)$.

\subsection{\texorpdfstring{$\mathcal{C}_7$}{C7} (ancestry)}
\label{sec:ancestry-redundant}

For completeness, we also empirically tested adding vs. not adding the redundant constraint $\mathcal{C}_7$. Across all $95$ test cases (graphs up to $|V| = 30$), disabling $\mathcal{C}_5$
yields no graph-level counterexample. Moreover, the solver runs $1.42\times$
faster without $\mathcal{C}_5$ ($23.0$\,s vs.~$16.2$\,s aggregate).

\newpage
\begin{center}
  \[\includegraphics[width=0.55\textwidth]{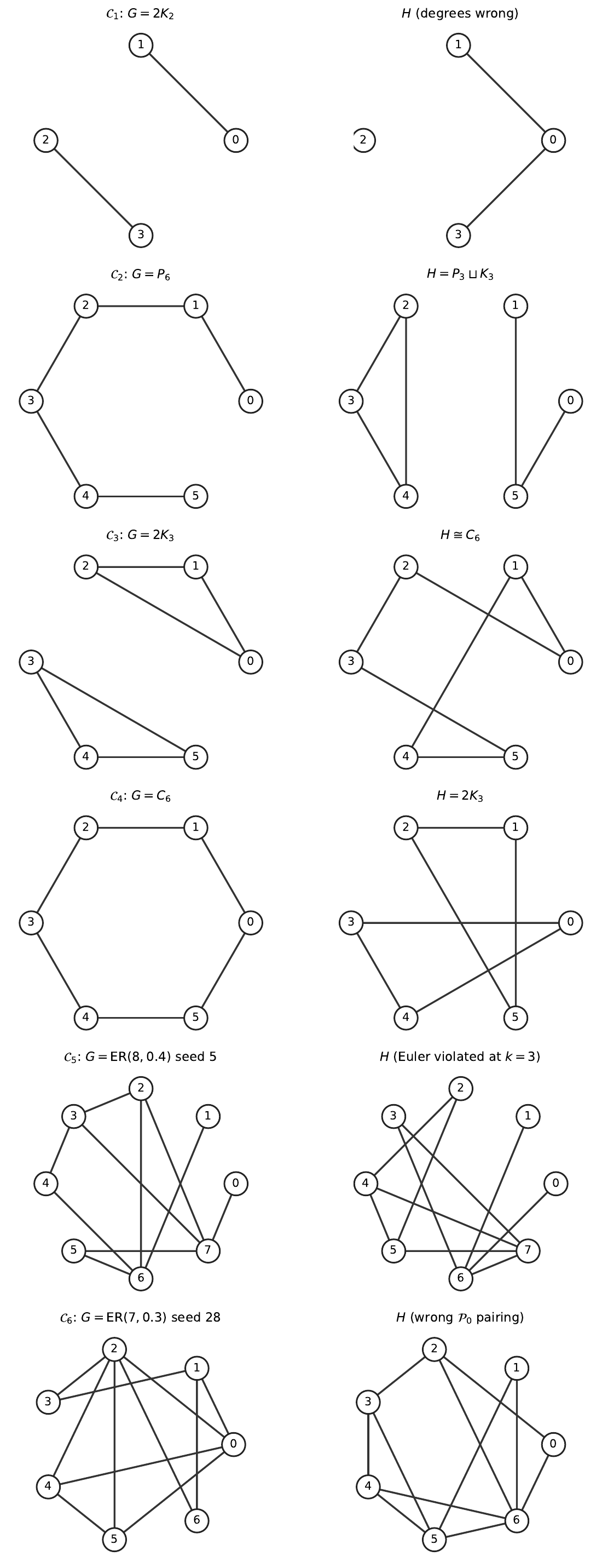}\]
  \captionsetup{hypcap=false}\captionof{figure}{Counterexamples in Appendix~\ref{sec:appendix-counterexamples}. The left displays the target graph and the right displays a graph generated with constraint $\mathcal{C}_i$ removed that has a different degree PD.}\label{fig::counter_example}
\end{center}

\newpage

\section{Experimental Details}
\label{app:setup}

\subsection{Datasets}
\label{app:datasets}

We evaluate on four standard graph generation benchmarks: \textsc{Comm-20}
(small dense two-community SBMs), \textsc{Planar} (fixed-size random planar
graphs), \textsc{SBM} (large stochastic block model graphs), and
\textsc{ENZYMES} (real-world enzyme protein contact maps from the
TUDataset~\cite{morris2020tudataset}). The first three are synthetic and
follow the conventions of prior graph generation
work~\cite{martinkus2022spectre, vignac2022digress, parktopology,
madeira2024generative}; \textsc{ENZYMES} is a  real-world dataset also previously used in prior work. We
adopt the uniform $64/16/20\%$ train/validation/test split convention from
SPECTRE~\cite{martinkus2022spectre}.

Aggregate statistics are listed in Table~\ref{tab::dataset_stats}. The four
datasets cover a range of regimes that stress different aspects of our
methods. \textsc{Comm-20} is the smallest and densest, with the lowest
max degree; the easiest setting in which to demonstrate strict PD compliance.
\textsc{Planar} has fixed $|V| = 64$ and  every graph is guaranteed to be connected and planar by construction.
\textsc{SBM} drives the difficulty: largest graphs
($|V|$ up to $187$), highest edge counts ($|E|$ up to $1129$), and highest
max degree ($\bar{\Delta} = 17.1$), and is where constraint-programming
scalability is most stressed. \textsc{ENZYMES} tests generalisation beyond
synthetic benchmarks.

\begin{table}[!ht]
\centering
\caption{Dataset statistics (all splits combined). Splits follow the uniform
$64/16/20\%$ convention.}
\label{tab::dataset_stats}
\small
\setlength{\tabcolsep}{4pt}
\renewcommand{\arraystretch}{1.05}
\begin{tabular}{@{}lcccc@{}}
\toprule
& \textsc{Comm-20} & \textsc{Planar} & \textsc{SBM} & \textsc{ENZYMES} \\
\midrule
Total graphs                 & 100 & 200 & 200 & 587 \\
Splits (train / val / test)  & 64 / 16 / 20 & 128 / 32 / 40 & 128 / 32 / 40 & 376 / 94 / 117 \\
Nodes (min / mean / max)     & 12 / 15.7 / 20 & 64 / 64 / 64 & 44 / 104.0 / 187 & 10 / 33.0 / 125 \\
Edges (min / mean / max)     & 21 / 37.0 / 61 & 173 / 177.8 / 181 & 129 / 500.4 / 1129 & 16 / 63.3 / 149 \\
Max.\ degree (mean)          & 6.5 & 9.0 & 17.1 & 6.1 \\
\bottomrule
\end{tabular}
\end{table}

\subsection{Baseline Methods}
\label{app:baselines}

We compare against four neural baselines along two axes: the diffusion
backbone (\textsc{DiGress} or \textsc{ConStruct}) and whether the
topology-aware components of \textsc{TAGG} (TAM module and PDM loss) are
added. This yields \textsc{DiGress}, \textsc{ConStruct},
\textsc{DiGress+TAGG}, and \textsc{ConStruct+TAGG}.

\subsubsection{DiGress}
\label{app:digress}

\textsc{DiGress}~\cite{vignac2022digress} is a discrete denoising diffusion
model and the foundation for both \textsc{ConStruct} and \textsc{TAGG}. The
forward process independently corrupts each node and edge category through
transition matrices $Q^t_X, Q^t_E$. The reverse process is a graph
transformer $\phi_\theta$ that maps a noisy graph $G^t$ to predicted
distributions $(\hat{p}^X, \hat{p}^E)$ over clean-graph categories, trained
with cross-entropy
\[
\mathcal{L}_{\mathrm{CE}}(G_0, \hat{p}^G) =
\sum_i \mathrm{CE}(x_i, \hat{p}^X_i) + \lambda \sum_{i,j} \mathrm{CE}(e_{ij}, \hat{p}^E_{ij}).
\]
We use the marginal-noise variant, which sets the limit distribution to the
empirical node and edge marginals of the training set, and augment the
denoising network with auxiliary cycle counts and Laplacian spectral features
to mitigate the expressivity limits of message
passing~\cite{gin}. \textsc{DiGress} is unconstrained and serves as the
vanilla diffusion baseline.

\subsubsection{ConStruct}
\label{app:construct}

\textsc{ConStruct}~\cite{madeira2024generative} extends \textsc{DiGress} with
hard constraints for \emph{edge-deletion-invariant} properties (planarity,
acyclicity, max degree, etc.). It introduces two modifications.

\xhdr{Edge-absorbing noise model.} The forward edge transition sends each
edge either to itself or to an absorbing no-edge state,
\[
Q^t_E = \alpha^t_{\mathrm{abs}} I + (1 - \alpha^t_{\mathrm{abs}})\, \mathbf{1}_c\, e^{\prime}_E,
\]
making the forward sequence a chain of subgraphs $G^0 \supseteq G^1 \supseteq
\cdots \supseteq G^T = \emptyset$ and the reverse process pure edge insertion.

\xhdr{Connection to graph filtrations.} The forward chain $G^0 \supseteq G^1
\supseteq \cdots \supseteq G^T = \emptyset$ (and its reverse) is structurally
a general filtration of $G^0$ in the sense of Sec.~\ref{sec::general_fil},
with the edge-insertion order set by the diffusion schedule rather than a
canonical filter function. \textsc{ConStruct} is not topology-aware and  the
schedule has no relation to persistent homology. Still, the subgraph-chain skeleton
aligns more naturally with the filtration setting of persistent homology
than \textsc{DiGress}'s independent-edge corruption does, making
\textsc{ConStruct} a more filtration-compatible baseline even before any
topological signal is added on top.

\xhdr{Projector.} At each reverse step, given the candidate $\hat G^{t-1}$
and previous iterate $G^t$, the projector iteratively inserts candidate edges
$E(\hat G^{t-1}) \setminus E(G^t)$ in random order, discarding any whose
insertion violates the target property. Theorem 1
of~\cite{madeira2024generative} shows this recovers an optimal solution to
$\min_{G \in \mathcal{C}} \mathrm{GED}(\hat G^{t-1}, G)$ for $\mathcal{C} =
\{G : \mathcal{P}(G), G \supseteq G^t\}$.

We apply \textsc{ConStruct} with planarity on \textsc{Planar}; for the other
three datasets, which lack a uniform edge-deletion-invariant property, we run
the same architecture with the projector disabled.

\subsubsection{TAGG: Topology-Aware Attention and PDM Loss}
\label{app:tagg}

\textsc{TAGG}~\cite{parktopology} uses the degree-based filtration as a
conditioning signal and proposes two topology-aware components built on top
of it, both of which we re-implement and apply to both \textsc{ConStruct}
and \textsc{DiGress} backbones.

\xhdr{PDM loss.} The persistence diagrams $D_{G_0}, D_{\hat p^{G_0}}$ are
computed under the degree-based filtration (matching our setting up to
normalisation). The loss is the $1$-Wasserstein distance with diagonal
padding for cardinality mismatch and per-dimension matching to prevent
$\mathcal{P}_0$/$\mathcal{P}_1$ confusion:
\[
\mathcal{L}_{\mathrm{PDM}}(G_0, \hat p^{G_0})
= W_1(D_{G_0}, D_{\hat p^{G_0}}).
\]
The full objective is $\mathcal{L} = \mathcal{L}_{\mathrm{CE}} + \alpha\,
\mathcal{L}_{\mathrm{PDM}}$.

\xhdr{Topology-aware attention (TAM).} A homological feature vector
$\mu_{G_0} \in \mathbb{R}^{L \cdot S}$, derived from the persistence
landscape of $G_0$, is injected as a global bias in attention scores at every
denoising layer:
\[
A^m_t = \mathrm{softmax}\!\Bigl( \tfrac{Q^m_t (K^m_t)^\top}{\sqrt{d_k}} \cdot Z^{m}_{E_t} + H_{G_0} \Bigr),
\]
with $H_{G_0} = \mathrm{Lin}_\mu(\mu_{G_0})$. At inference, $G_0$ is unknown,
and the original \textsc{TAGG} substitutes the training-set average
$\mu_{G'}$ as conditioning, making it a distribution-level method rather
than a per-target one. To align it with our setting --- where the target PD
is given per sample --- we instead feed the homological feature vector
$\mu_{G_{\mathrm{test}}}$ derived from each test graph as the conditioning
signal at inference. This grants \textsc{TAGG} the strongest possible
per-target topology signal, putting it on equal footing with our methods for
the strict PD-compliance evaluation.

\subsubsection{Augmenting DiGress and ConStruct with TAGG}
\label{app:tagg_augmentation}
\textsc{DiGress+TAGG} adds both TAM and the PDM loss to \textsc{DiGress},
recovering the original \textsc{TAGG} configuration.
\textsc{ConStruct+TAGG} adds the same components on top of
\textsc{ConStruct}'s edge-absorbing forward process and projector-augmented
reverse process; the projector remains active and filters edges proposed by
the TAGG-augmented denoiser. In both cases, inference uses the per-target
$\mu_{G_{\mathrm{test}}}$ derived from each test graph, as described above.

\xhdr{Hyperparameter selection.} For each baseline we use the default
hyperparameters proposed in the original paper for the datasets considered.

\subsection{Hardware} The constraint-programming solver and the local rejection sampler ran on a 120-CPU-core machine; no GPU was used. Neural diffusion baselines were trained on $4\times$ NVIDIA A100 (40\,GB) GPUs and run at inference on a single A100 or RTX~6000 Ada (48\,GB).

\subsection{Local vs.\ Global: Full Comparison Across Filtrations}
\label{app:local_extra}

We expand Sec.~\ref{sec::exp_local_vs_global} with the full set of local generators across the three filtration regimes considered in Sec.~\ref{sec::local}, together with median statistics and per-method commentary.

\xhdr{Methods compared.} Recall from Sec.~\ref{sec::local} that the local approach yields three rejection samplers, each tied to a different filtration scheme:
\begin{itemize}[leftmargin=1.3em, topsep=2pt, itemsep=2pt]
\item \textsc{Random} (Algorithm~\ref{alg:pd_to_graph}): general filtration. The vertex and edge filter values $f_V, f_E$ are not canonical and are sampled along with the graph.
\item \textsc{Vertex} (Algorithm~\ref{alg:vertex_based_gen}): vertex-based filtration. Edges are forced to take time $\max(f_V(u), f_V(v))$, but $f_V$ itself is not canonical and is sampled.
\item \textsc{Degree} (Algorithm~\ref{alg:degree_based_gen}): degree-based filtration. Both vertex and edge filter values are uniquely determined by the graph, so the algorithm only needs to return the graph itself.
\end{itemize}
For \textsc{Random} and \textsc{Vertex}, the algorithm therefore returns the pair $(H, f)$ specifying both the graph and a filtration function realising the target PD; only \textsc{Degree} (and the global CP solver) work directly with the canonical degree filter $f_V(v) = \deg(v)$. In this comparison the CP solver targets the degree filtration, so that it is compared with \textsc{Degree} on the same canonical filtration. Other filtrations are verified in Appendix~\ref{app:filtrations}.

\begin{remark}\label{rmk::implementation_bmin}
    In the implementation of Algorithm~\ref{alg:vertex_based_gen} and Algorithm~\ref{alg:degree_based_gen}, we took an alternative approach to the \textbf{minimum vertex pivot} step.

    In any stage of vertex based filtrations, let $b_{\min} = \operatorname{inf}_{\text{component deaths of the form }(b, k+1)} b$. In the minimum vertex pivot step, we instead do the following procedure:
\begin{enumerate}[leftmargin=*, topsep=0pt, itemsep=0pt]
    \item If there is only one component born at $b_{\min}$, randomly pick a component born at $t < b_{\min}$ and connect to a new vertex.
    \item Otherwise, randomly pick a component born at $t \leq b_{min}$ and connect to a new vertex.
\end{enumerate}
If the algorithm terminates, this is still guaranteed to give a PD-equivalent output and could realize all possible PD-equivalent vertex-based filtrations. This choice involved less variables to check, but a caveat is that this could potentially fail the component merging step.
\end{remark}

\xhdr{Setup.} We follow the protocol of Sec.~\ref{sec::exp_local_vs_global}: each rejection sampler is allowed up to $20{,}000$ proposals per target with $4$ parallel evaluator threads, and CP is given a $1{,}800$\,s budget with $120$ internal CP-SAT search workers. Wall-clocks are measured on a clean $120$-CPU-core machine. We report success rate, mean and median per-target wall-clock (charging the full budget on failed runs), and mean and median attempts to first success on successful runs.

\xhdr{Results.} Table~\ref{tab::local_vs_global_full} reports the full set of numbers.

\begin{table}[!ht]
\centering
\caption{Full local/global comparison across the three filtration regimes. \textsc{Random} and \textsc{Vertex} target their respective filtrations; \textsc{Degree} (Alg.~\ref{alg:degree_based_gen}) and \textsc{CP} both target the canonical degree filtration. Wall-clock means include the full $20{,}000$-attempt or $1{,}800$\,s budget on failed runs; attempt statistics are over successful runs only. CP performs a single solver call per target.}
\label{tab::local_vs_global_full}
\small
\setlength{\tabcolsep}{4pt}
\renewcommand{\arraystretch}{1.1}
\begin{tabular}{@{}llcccc@{}}
\toprule
& & \textsc{Comm-20} & \textsc{Planar} & \textsc{SBM} & \textsc{ENZYMES} \\
\midrule
\multirow{5}{*}{\textsc{Random} (Alg.~\ref{alg:pd_to_graph})}
 & Success (\%) $\uparrow$               & 100.0  & 100.0   & 100.0   & 100.0  \\
 & Wall-clock mean (s) $\downarrow$      & 0.10   & 2.02    & 6.04    & 0.71   \\
 & Wall-clock median (s) $\downarrow$    & 0.09   & 2.02    & 5.30    & 0.65   \\
 & Attempts mean $\downarrow$            & 1.0    & 2.8     & 1.1     & 1.7    \\
 & Attempts median $\downarrow$          & 1      & 1       & 1       & 1      \\
\midrule
\multirow{5}{*}{\textsc{Vertex} (Alg.~\ref{alg:vertex_based_gen})}
 & Success (\%) $\uparrow$               & 100.0  & 100.0   & 100.0   & 100.0  \\
 & Wall-clock mean (s) $\downarrow$      & 0.11   & 2.82    & 8.70    & 0.86   \\
 & Wall-clock median (s) $\downarrow$    & 0.10   & 2.82    & 6.31    & 0.76   \\
 & Attempts mean $\downarrow$            & 1.1    & 258.6   & 476.6   & 81.0   \\
 & Attempts median $\downarrow$          & 1      & 185.5   & 231.0   & 41.0   \\
\midrule
\multirow{5}{*}{\textsc{Degree} (Alg.~\ref{alg:degree_based_gen})}
 & Success (\%) $\uparrow$               & 100.0  & 100.0   & \textbf{15.0}    & 99.1   \\
 & Wall-clock mean (s) $\downarrow$      & 0.17   & 5.88    & 1465.0  & 6.94   \\
 & Wall-clock median (s) $\downarrow$    & 0.13   & 4.62    & ${\sim}3000^{\dagger}$  & 1.07   \\
 & Attempts mean $\downarrow$            & 49.6   & 137.5   & 4583.5  & 431.5  \\
 & Attempts median $\downarrow$          & 6      & 70.0    & 3552.5  & 15     \\
\midrule
\multirow{2}{*}{\textsc{CP} (Sec.~\ref{sec:global})}
 & Success (\%) $\uparrow$               & 100.0  & 100.0   & \textbf{82.5}    & 100.0  \\
 & Wall-clock mean (s) $\downarrow$      & 0.30   & 16.35   & 654.8   & 6.85   \\
 & Wall-clock median (s) $\downarrow$    & 0.16   & 16.21   & 142.5   & 1.78   \\
\bottomrule
\end{tabular}
\\[2pt]
{\footnotesize $^{\dagger}$ \textsc{Degree}/\textsc{SBM}: only $6/40$ targets succeed; the median is dominated by the failure budget. Successful runs alone have a median wall-clock of $162$\,s.}
\end{table}

\xhdr{Discussion.} The three local samplers form a clean ordering from least to most constrained. \textsc{Random}, working under a general filtration where $f_V, f_E$ are jointly sampled, accepts almost every proposal (mean attempts $1$--$3$ across all four datasets) and finishes in seconds. \textsc{Vertex}, with edges deterministically pinned to $\max(f_V(u), f_V(v))$, sees attempt counts grow to a few hundred on the larger benchmarks but still completes within $9$\,s on every dataset. \textsc{Degree}, where the filtration is fully canonical and the PD effectively determines the degree sequence, is where rejection becomes expensive: the per-attempt success probability collapses on large dense instances, and on \textsc{SBM} only $15\%$ of test targets are solved within the $20{,}000$-attempt budget. The CP solver targets the same degree filtration but, by enforcing the global constraints up front, raises the success rate to $82.5\%$ on \textsc{SBM} and cuts the mean wall-clock by more than a factor of two. On \textsc{Comm-20}, \textsc{Planar}, and \textsc{ENZYMES} the two degree-filtration methods solve essentially all targets, with the local sampler $1.7$--$2.8\times$ faster on the smaller datasets and on par on \textsc{ENZYMES}.

\xhdr{Tail behaviour.} The mean--median gap on \textsc{SBM} and \textsc{ENZYMES} reflects a heavy right tail in attempt counts. For \textsc{Degree}/\textsc{ENZYMES} (where success is $99.1\%$), $50\%$ of targets are solved in $\leq 15$ attempts, but the top $5\%$ require $\geq 441$ and the worst $17{,}331$. The mean ($431.5$) is dominated by this tail and overstates the typical cost. This is the natural pattern of rejection sampling against PD constraints derived from heterogeneous graphs: most reference PDs are typical and easy to match by chance, but a few have rare structure that the proposal distribution rarely hits.

\xhdr{Practical recipe.} The pair (success rate, mean attempts) characterises the difficulty of each (filtration, dataset) cell. Mean attempts $\leq 100$ marks an easy cell where the local sampler is the right tool; $\geq 1{,}000$ a borderline cell; $\geq 10{,}000$ (or hitting \texttt{max\_attempts}) a hard cell where CP becomes necessary. A simple hybrid --- run the local sampler with a modest cap (e.g.\ $1{,}000$ attempts) and fall back to CP on rejection --- achieves the best of both: free generation on the easy majority of targets, principled fallback on the hard tail. We use the local-only and CP-only configurations in the main paper for clarity of presentation.

\subsection{Neural Baseline Comparison: Full Details}
\label{app:exp_neural}

This appendix expands Sec.~\ref{sec::exp_neural} with full descriptions of the training and inference protocols, the evaluation metric, the full results table with standard deviations (Table~\ref{tab::cp_vs_tagg_full}, the unabridged version of Table~\ref{tab::cp_vs_tagg}), and discussion of each (method, dataset) cell.

\xhdr{Inference protocol: paired vs.\ averaged conditioning.} The original \textsc{TAGG}~\cite{parktopology} substitutes the dataset-averaged landscape vector $\mu_{G'} = \frac{1}{|D_\mathrm{train}|} \sum_{G \in D_\mathrm{train}} \mu_G$ as the conditioning signal at every test sample, since the test graph $G_0$ is unknown at inference. This makes \textsc{TAGG} a distribution-level method: it generates graphs whose topology resembles the training-set average, not a specific target. To align it with our setting --- where the target PD is given per sample --- we instead feed each test graph's own $\mu_{G_{\mathrm{test}}}$ at inference. Concretely, in the test loop we iterate the test data loader and pass each batch's $\mu$ vector and $n$-nodes mask into the reverse-diffusion sampling routine, producing one generated graph per test reference. This change has two consequences: (i) it grants the TAGG variants the strongest possible per-target topology signal, putting them on equal footing with our CP solver for the strict PD-compliance evaluation; (ii) it yields a 1-to-1 correspondence between generated and reference graphs, enabling the paired $\bar W_1$ metric. Within our paper, all three methods (CP, \textsc{ConStruct+TAGG}, \textsc{DiGress+TAGG}) are evaluated under this paired protocol so they are pairwise comparable; absolute $\bar W_1$ numbers are not directly comparable to the original \textsc{TAGG} paper's reported numbers, which use the averaged $\mu_{G'}$ protocol.

\xhdr{Evaluation metric details.} For each test reference $G^{(i)}_\mathrm{ref}$ paired with its generated graph $H^{(i)}$, we compute persistence diagrams under the degree filtration using \texttt{gudhi} and pair them via the standard $1$-Wasserstein distance with diagonal padding for unequal-cardinality multisets. The two homology dimensions are summed:
\[
\bar W_1^{(i)} \;=\; W_1\!\big(\mathcal{P}_0(G^{(i)}_\mathrm{ref}), \mathcal{P}_0(H^{(i)})\big) + W_1\!\big(\mathcal{P}_1(G^{(i)}_\mathrm{ref}), \mathcal{P}_1(H^{(i)})\big),
\]
and we aggregate as $\bar W_1 = \frac{1}{N_\mathrm{test}} \sum_i \bar W_1^{(i)}$. The CP solver's $\bar W_1 = 0.000$ entries serve as a sanity check: they confirm both that CP's output PD matches the target exactly and that our $W_1$ measurement code returns zero on identical multisets (no float-precision dust). The numerical scale of $\bar W_1$ varies across datasets because larger graphs have larger PDs --- absolute magnitudes between datasets are not directly comparable, only relative orderings within a column.

\xhdr{Hardware and budgets.} CP is run on a $120$-core CPU slice with \texttt{cp\_sat\_workers=120} and a $1{,}800$\,s per-target time limit. Neural inference runs on a single GPU (RTX~6000 Ada or A6000) per cell, single-process. Training is multi-GPU when available (typically $4$ GPUs with DDP). The trained checkpoint per cell is selected as the lowest-validation-MMD-mean checkpoint among the top-$3$ lowest \texttt{val/epoch\_NLL} checkpoints, automated in the test branch of each \texttt{train\_*.py} script.

\xhdr{Full results with standard deviations.} Table~\ref{tab::cp_vs_tagg_full} reports the unabridged version of Table~\ref{tab::cp_vs_tagg}, with neural variants summarised as mean\,$\pm$\,std over $5$ random seeds. CP is deterministic and reported as a single number. Standard deviations are uniformly small ($1$--$7\%$ of the mean), confirming that the residual gap between neural variants and CP is structural rather than due to seed noise.

\begin{table}[!ht]
\centering
\caption{Full version of Table~\ref{tab::cp_vs_tagg}: \textsc{Global} (CP) vs.\ TAGG-augmented neural baselines, with standard deviations over $5$ random seeds. CP is deterministic.}
\label{tab::cp_vs_tagg_full}
\small
\setlength{\tabcolsep}{4pt}
\renewcommand{\arraystretch}{1.1}
\begin{tabular}{@{}llcccc@{}}
\toprule
& & \textsc{Comm-20} & \textsc{Planar} & \textsc{SBM} & \textsc{ENZYMES} \\
\midrule
\multirow{3}{*}{\textsc{Global} (CP)}
 & $\bar W_1$ $\downarrow$            & \textbf{0.000} & \textbf{0.000} & \textbf{0.000} & \textbf{0.000} \\
 & Time/graph (s) $\downarrow$        & 0.30 & 16.35 & 654.8 & 6.85 \\
 & Train (h) $\downarrow$             & 0    & 0     & 0     & 0    \\
\midrule
\multirow{3}{*}{\textsc{ConStruct+TAGG}}
 & $\bar W_1$ $\downarrow$            & 9.35\,$\pm$\,0.97 & 43.23\,$\pm$\,0.56 & 434.48\,$\pm$\,17.54 & 17.35\,$\pm$\,0.24 \\
 & Time/graph (s) $\downarrow$        & 0.88\,$\pm$\,0.01 & 4.73\,$\pm$\,0.09  & 17.11\,$\pm$\,0.09   & 4.92\,$\pm$\,0.46  \\
 & Train (h) $\downarrow$             & 9.4  & 24.3  & 59.8  & 61.6 \\
\midrule
\multirow{3}{*}{\textsc{DiGress+TAGG}}
 & $\bar W_1$ $\downarrow$            & 4.52\,$\pm$\,0.32 & 38.77\,$\pm$\,1.07 & 464.99\,$\pm$\,32.73 & 17.06\,$\pm$\,0.44 \\
 & Time/graph (s) $\downarrow$        & 0.90\,$\pm$\,0.00 & 5.30\,$\pm$\,0.28  & 23.16\,$\pm$\,0.47   & 5.94\,$\pm$\,0.05  \\
 & Train (h) $\downarrow$             & 8.9  & 32.5  & 47.6  & 49.2 \\
\bottomrule
\end{tabular}
\end{table}

\begin{table}[!ht]
\centering
\caption{Handling of failed CP runs on \textsc{SBM}. CP solves $33/40$ targets within the $1{,}800$\,s budget (the coverage reported in Table~\ref{tab::local_vs_global}); its $\bar W_1$ in Table~\ref{tab::cp_vs_tagg} is over these solved targets. Restricting the neural baselines to the same $33$ targets gives an equal comparison.}
\label{tab::sbm_equal}
\small
\begin{tabular}{@{}lcc@{}}
\toprule
\textsc{SBM}, $\bar W_1$ $\downarrow$ & All $40$ targets & Same $33$ targets (CP-solved) \\
\midrule
\textsc{ConStruct+TAGG} & 434.48 & 408.95 \\
\textsc{DiGress+TAGG}   & 464.99 & 454.49 \\
CP (ours)               & 0.00 ($33/40$ solved) & 0.00 \\
\bottomrule
\end{tabular}
\end{table}

\xhdr{Per-dataset observations.}
\begin{itemize}[leftmargin=1.3em, topsep=2pt, itemsep=2pt]
\item \textbf{\textsc{Comm-20}.} The smallest and densest dataset. CP solves each target in $0.3$\,s --- the constraint problem is small enough that solver overhead dominates --- and is the only cell where CP is faster than the neural variants. \textsc{DiGress+TAGG} achieves the lowest neural $\bar W_1$ ($4.52$), roughly halving \textsc{ConStruct+TAGG}'s ($9.35$); the marginal-distribution transition seems to give a more PD-faithful denoising path on small graphs.

\item \textbf{\textsc{Planar}.} The fixed-size benchmark. CP scales to $16.35$\,s/graph, $3$--$3.5\times$ slower than the neural variants, since the cycle structure under degree filtration ($\sim$$110$ cycles, max degree $8$--$12$) generates a moderately large slot/ancestry encoding. The two TAGG variants achieve similar $\bar W_1$ ($\sim$$40$).

\item \textbf{\textsc{SBM}.} The hardest cell. $7/40$ CP runs hit the $1{,}800$\,s timeout, contributing the worst-case $1{,}800$\,s to the mean ($654.8$\,s). The neural variants run in $17$--$23$\,s/graph but post the largest $\bar W_1$ ($\sim$$435$--$465$), reflecting both the size of the PDs (max degree $\sim$$21$, $100+$ $H_1$ points) and the difficulty of the topology distribution. CP's $\bar W_1=0$ is computed over its $33$ solved targets; Table~\ref{tab::sbm_equal} compares all methods on the same set of targets.

\item \textbf{\textsc{ENZYMES}.} The most heterogeneous (real-world). CP and the neural variants are roughly tied on per-graph time ($5$--$7$\,s); the neural variants achieve $\bar W_1 \approx 17$, the lowest TAGG $\bar W_1$ relative to PD scale.
\end{itemize}

\subsection{Qualitative Examples: Detailed Discussion}
\label{app:qualitative}

This appendix expands Figure~\ref{fig::qualitative} from Sec.~\ref{sec::exp_neural} with a second test graph and a per-row breakdown of how each method's output deviates from the reference. Both examples are drawn from the \textsc{ENZYMES} test split; node colour encodes degree, which under the degree filtration is also the filtration value of each vertex.

\begin{figure}[!ht]
\centering
\includegraphics[width=0.95\textwidth]{Figures/enzymes_87_filtration.pdf}\\[6pt]
\includegraphics[width=0.95\textwidth]{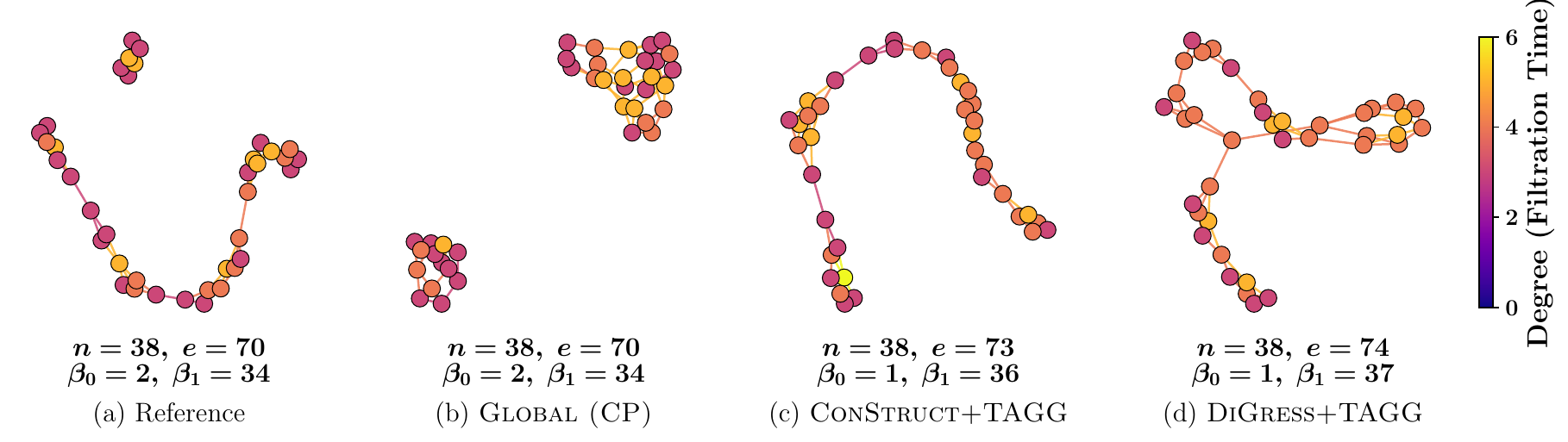}
\caption{Two qualitative examples on \textsc{ENZYMES} test graphs (\texttt{enzymes\_87}, top; \texttt{enzymes\_93}, bottom). Left to right: reference, \textsc{Global} (CP), \textsc{ConStruct+TAGG}, \textsc{DiGress+TAGG}. Node colour encodes degree. CP matches $(n, |E|, \beta_0, \beta_1)$ exactly on both graphs; neural variants drift in edge count and cycle count.}
\label{fig::qualitative_full}
\end{figure}

\xhdr{Top row (\texttt{enzymes\_87}).} The reference is a moderate-size enzyme contact graph with two connected components and a moderate cycle count. CP reproduces both invariants exactly: the same $\beta_0$, the same $\beta_1$, and the same edge count, with vertices placed into two slots whose union recovers the target PD. The two neural variants both merge the two reference components into a single connected graph ($\beta_0$: $2 \to 1$) and overshoot the target cycle count by $2$--$3$ ($\beta_1$). The merge is the more telling failure mode: it shows that even with the strongest per-target conditioning ($\mu_{G_\mathrm{test}}$ rather than the training-set average), the diffusion denoiser cannot reliably preserve component count, because component structure is a global property that no local denoising step on its own can enforce. The excess cycles compound the issue --- the PDM regulariser nudges the predicted PD toward the reference but does not commit to it.

\xhdr{Bottom row (\texttt{enzymes\_93}).} A larger and denser reference, where the failure modes are amplified. Both neural variants generate $14$--$17$ excess edges and more than double the target $\beta_1$, while $\beta_0$ stays roughly correct. This is the opposite failure pattern from the top row: rather than merging components, the denoiser over-fills the cycle structure within each component. The same structural cause is at work: the PDM loss is a soft penalty on the diagram, not a hard constraint, and when many edge configurations land at similar diffusion-loss values, the gradient-based optimiser settles in a region that is topologically off-target. CP, by contrast, returns a graph that is visibly different from the reference --- node positions and edge layout do not match --- but every degree-PD invariant is correct. This is the crucial separation our framework makes precise: a strict PD-compliant generator does not aim to recover the reference graph, only some graph in the PD-equivalence class.

\section{Complexity and Scaling Analysis}\label{app:complexity}

Here we give a complexity and scaling analysis of our local sampler and global encoder when $\Phi = \deg$. Since other filtrations only affects the constraint $\mathcal{C}_1$, the analysis outside of $\mathcal{C}_1$ still applies to them. For the global encoding, we also ran empirical experiments to test its runtimes in practice.

\textbf{Notation:} We adopt the following notations in the analysis:
\begin{itemize}
    \item $n = |\mathcal{P}_0|$ is the vertex count
    \item $m$ is the edge count
    \item $K = \max_v \deg(v)$ is the max degree (this can be read off as the max filtration stage)
    \item $cc_{\max} = \max_k |\{(b, d) \in \mathcal{P}_0 : b \le k < d\}|$ is the max concurrent alive component count.
\end{itemize}

\subsection{Local sampler} Each iteration of the sampling selects a vertex-birth or edge-birth event and updates a union-find structure, dominated by a preprocessing sort of the filtration events. This gives a single successful run the overall complexity of $O((n + m) \log(n + m))$ in time and $O(n + m)$ in space. There are no bounds in general on how many runs one needs to try to get a successful outcome (Theorem~\ref{thm::local_impossibility}), except for when we either have acyclic targets (Corollary~\ref{cor::general_trees}) or stages where cycle-birth demand stays within the running graph's cycle capacity (Proposition~\ref{prop::cycle_capacity}).

\subsection{Global CP encoding}
The global encoder is summarized by the following table.
\begin{table}[!ht]
\centering
\small
\begin{tabular}{@{}lll@{}}
\toprule
Constraint & Variables & Constraints \\
\midrule
$\mathcal{C}_1^{\deg}$ & $M_{ij}$: $O(n^2)$ & $O(n)$ \\
$\mathcal{C}_2$ & $x_{kcv}$: $O(K \cdot cc_{\max} \cdot n)$ & $O(Kn)$ \\
$\mathcal{C}_3$ & --- & $O(K \cdot cc_{\max} \cdot n^2)$ \\
$\mathcal{C}_4$ & $\pi_{kcij}$: $O(K \cdot cc_{\max} \cdot n^2)$; $\ell_{kcv}, r_{kcv}$: $O(K \cdot cc_{\max} \cdot n)$ & $O(K \cdot cc_{\max} \cdot n^2)$ \\
$\mathcal{C}_5$ & --- & $O(K)$ \\
$\mathcal{C}_6$ & $e_{kcb}$: $O(K \cdot cc_{\max})$ & $O(K \cdot cc_{\max})$ \\
\bottomrule
\end{tabular}
\end{table}

In total, the encoding uses $O(K \cdot cc_{\max} \cdot n^2)$ variables and constraints. Memory is linear in this quantity (all measured instances fit within commodity RAM). Encoding size is polynomial; solving is NP-hard in general (it subsumes SAT), so no closed-form runtime bound applies. However, modern CP solvers are highly optimized in practice.

Furthermore, we also conducted empirical study of graph structural properties on CP runtime. We ran controlled scaling experiments on synthetic targets, varying one property at a time and running each through the CP solver with a 60s budget. Here we report on some \textbf{empirical behaviors} we observed. 
\begin{enumerate}[leftmargin=*, topsep=0pt, itemsep=1pt]
\item \textbf{Vertices $n$ dominates model size} ($O(n^2)$ edge variables). We tested this on Erd\H{o}s--R\'enyi graphs at fixed $p=0.3$, while varying $n$ (the number of vertices). We obtained on average 0.16s at $n=10$, 0.87s at $n=20$, 1.42s at $n=25$. There is roughly $3$-$5\times$ per doubling.
\item \textbf{Density helps CP.} We also observed that denser graphs merge components earlier and reduces the quantity $cc_{\max}$. We tested this on Erd\H{o}s--R\'enyi graphs at fixed $n=25$, while varying $p$. We obtained on average 6.83s at $p=0.2$, 0.90s at $p=0.5$, 0.60s at $p=0.7$. Thus, we empirically conclude that density affects cost through $cc_{\max}$.
\item \textbf{$cc_{\max}$ is the dominant factor.} We tested this on disjoint unions of $k$ copies of $P_3$, isolating $cc_{\max}$ while keeping $K$ and cycle count fixed. We observed on average 0.03s at $cc_{\max}=4$, 0.41s at 12, 3.1s at 20, 12.8s at 24. The cost roughly doubles per $+2$ copies.
\item \textbf{Cycles add a modest overhead.} 
We tested this on disjoint unions of copies of $P_3$ against disjoint unions of copies of $C_3$ at matched $cc_{\max}$. At $cc_{\max}=6$, $6 \times C_3$ (6 cycles) was marginally slower than $6 \times P_3$ (0 cycles), consistent with the extra cycle-birth constraints, though the effect was small relative to the dependence on $cc_{\max}$ itself.
\end{enumerate}

\begin{remark}
The empirical behaviors above explain the behavior of the \textbf{execution time on the SBM dataset} (see Table~\ref{tab::local_vs_global}). The SBM graphs have about $100$ vertices on average, and 2-3 communities persisting across many filtration stages before merging. This drives $cc_{\max}$ up over long ranges, hitting both the $n^2$ size term in the encoding and the $cc_{\max}$-multiplier in the slot/parent variable counts simultaneously. The observed SBM timeouts (30-min budget) all have $n \in [144, 164]$ combined with high $cc_{\max}$.
\end{remark}

\subsection{A practitioner's guide to choosing the methods}
\label{app:prac}

Here we give some heuristics as to which method to choose (local sampler, global CP solver, and neural approximations) in practice.
\begin{enumerate}[leftmargin=*, topsep=0pt, itemsep=1pt]
\item \textbf{Local sampler.} This is preferred for sparse targets. Cycle counts grow slowly, so Proposition~\ref{prop::cycle_capacity} (and Corollary~\ref{cor::general_trees} for acyclic targets) guarantees single-run termination.
\item \textbf{Global CP solver.} This is preferred for denser targets where Theorem~\ref{thm::local_impossibility}'s infeasibility scenarios make local sampling unreliable, provided that the quantities $n$ and $cc_{\max}$ stay tractable. Empirically, we observe runtimes of sub-3s per graph at $n \le 40$, $cc_{\max} \le 5$; and sub-30s at $n \le 60$ or $cc_{\max} \le 10$.
\item \textbf{Neural approximation.} In regimes where both exact methods struggle (e.g., harder SBM instances, as they are not sparse enough for local, and too large in $cc_{\max}$ for CP), neural sampling is the only tractable choice, at the cost of exact PD compliance. As reported in the paper, one-off training cost is substantial, but amortized inference is fast. This pays off when many approximate graphs are needed quickly. Closing the gap in the third regime for our approach likely requires targeted algorithmic advances, or importing established techniques from the constraint-programming literature: symmetry breaking~\citep{gent1999symmetry, puget2002symmetry}, conflict-driven clause learning~\citep{marques1999grasp}, portfolio solvers~\citep{xu2008satzilla}, problem-specific decomposition~\citep{dechter1989tree, regin1994filtering}, none of which our current encoding exploits.
\end{enumerate}

\newpage
\section{Experiments with Other Filtration Schemes}\label{app:filtrations}

We verify the filtrations --- degree, triangle, $k$-clique, clustering, average neighbor degree, closeness, Katz, and PageRank --- for both soundness and completeness.
\begin{enumerate}[leftmargin=*, topsep=0pt, itemsep=1pt]
    \item For \textbf{soundness}, from each generated graph, we recompute $\Phi$ with NetworkX, recompute its PDs, and check it matches the target. Over a suite of structurally diverse graphs (paths, cycles, stars, complete graphs, random trees, Erd\H{o}s--R\'enyi, grids, barbells, wheels, Petersen) of size $\le 30$, every solution realizes the exact target, with no failures.
    \item \textbf{Completeness }is checked by exhaustive enumeration up to isomorphism at $n \le 6$: the CP solution set equals the full persistence-equivalence class.
\end{enumerate}

\begin{remark}
Overall, the solver cost depends on the encoding complexity. For degree, this is linear, whereas triangle and clustering add $O(n^3)$ auxiliaries and are costlier. On \textsc{Comm-20}, the success rate / median per-graph solve time are: degree 100\% / 0.16 s, triangle 90\% / 0.25 s, clustering 85\% / 1.31 s. Scaling these richer filtrations further on remaining datasets and other benchmarks, is a promising direction for future work.
\end{remark}

\section{Diversity Within a Degree-Based PD-Equivalence Class}\label{app:diversity}

Here we test how diverse the graphs in a (\textbf{degree-based}) PD-equivalence class are. We first note that diversity is intrinsic to our methods, both of which are sound and complete over the PD-equivalence class. The local sampler produces distinct members through its random choices. The CP-SAT solver enumerates the entire class (all solutions of the encoding), and at scale random seeds return distinct exact members. Thus, in general, we would expect many diverse graphs to occur in a given input PD.

In the experiment, for each reference target we run the CP solver $20\times$ with randomized search, and average pairwise structural distances over targets. The results are displayed below as:

\begin{table}[!ht]
\centering
\small
\begin{tabular}{@{}lcccccc@{}}
\toprule
Dataset & \# distinct & clustering & triangle & spectral & planar-rate & PD invariants ($\beta_0$, $\beta_1$, deg, density) \\
\midrule
\textsc{Comm-20} & 20 & 0.064 & 0.190 & 0.059 & 0.087 & 0.000 \\
\textsc{Planar}  & 18 & 0.020 & 0.212 & 0.020 & 0.000 & 0.000 \\
\textsc{SBM}     & 14 & 0.026 & 0.099 & 0.021 & 0.000 & 0.000 \\
\textsc{ENZYMES} & 20 & 0.072 & 0.240 & 0.060 & 0.039 & 0.000 \\
\bottomrule
\end{tabular}
\end{table}

To elaborate on the columns of the table:
\begin{itemize}[leftmargin=*, topsep=0pt, itemsep=1pt]
    \item \textbf{\# distinct} is the median number of pairwise non-isomorphic graphs among the $20$ solves.
    \item \textbf{clustering} is the mean pairwise absolute difference in average clustering coefficient (in $[0, 1]$).
    \item \textbf{triangle} is the mean pairwise relative difference in triangle count.
    \item \textbf{spectral} is the mean pairwise $\sqrt{n}$-normalized $L_2$ distance between sorted normalized-Laplacian spectra.
    \item \textbf{planar-rate} the fraction of class members that are planar.
    \item The last column (\textbf{PD invariants}) record the mean pairwise absolute difference of the PD-determined invariants ($\beta_0$, $\beta_1$, average degree, density), which is exactly $0$.
\end{itemize}
Note that the spectral column also serves as a surrogate for measuring \textbf{graph edit distances}. This is because exact GED is intractable at benchmark sizes ($n$ up to $\sim 123$), so the $\sqrt{n}$-normalized spectral distance is the standard scalable stand-in.

The table confirms that within a degree-PD class, generations diversify in non-topological structure while their PD invariants stay exactly fixed.


\end{document}

%% file: summary.tex
\begin{center}
\begin{minipage}{13.5cm}
\begin{tcolorbox}[tab2,tabularx={ll},title={},boxrule=1pt,top=0.9ex,bottom=0.9ex,colbacktitle=lb!10!white,colframe=lb!70!white]
\textbf{Local Generation of PD-Equivalent Filtrations (Section~\ref{sec::local})} \\ 
\addlinespace[0.3ex]
\quad Local Generation for General/Vertex/Degree Filtrations & Algorithm~\ref{alg:pd_to_graph},~\ref{alg:vertex_based_gen},~\ref{alg:degree_based_gen}\\ 
\addlinespace[0.3ex]
\quad Correctness and Diversity of Local Generations & Theorem~\ref{thm::general_filtration},~\ref{thm::vertex_filtration},~\ref{thm::degree_filtration} \\ \addlinespace[0.3ex]
\quad Perfect Generation on Acyclic Graphs (for General/Vertex) & Corollary~\ref{cor::general_trees},~\ref{cor::vertex_trees} \\ \addlinespace[0.3ex] 
\quad Impossibility Results for Local Generation & \autoref{thm::local_impossibility} \\ \addlinespace[0.3ex]
\midrule
\textbf{Global Generation of PD-Equivalent Graphs (Section~\ref{sec:global})} \\ \addlinespace[0.3ex]
\quad Global Generation for Any Canonical Vertex Filtration Scheme $\Phi$ & Section~\ref{sec::encoding}\\ \addlinespace[0.3ex]
\quad Soundness, Completeness, and Degree Minimality of Encoding & \autoref{thm:cp_correct} \\ \addlinespace[0.3ex]
\midrule
\textbf{Experiments (Section~\ref{sec::experiments})} \\ \addlinespace[0.3ex]
\quad Local vs.\ Global on Four Graph Generation Benchmarks & Table~\ref{tab::local_vs_global}\\ \addlinespace[0.3ex]
\quad Comparison Against PD-Conditioned Neural Diffusion Baselines & Table~\ref{tab::cp_vs_tagg}, Figure~\ref{fig::qualitative}\\ \addlinespace[0.3ex]
\end{tcolorbox}
\end{minipage}
\end{center}

%% file: Figures/intro_example.tex
\begin{tikzpicture}[
  vone/.style    ={draw, circle, fill=blue!50,   line width=0.5mm, minimum size=0.4cm, inner sep=0pt},
  vtwo/.style    ={draw, circle, fill=orange!60, line width=0.5mm, minimum size=0.4cm, inner sep=0pt},
  vthree/.style  ={draw, circle, fill=black!50,  line width=0.5mm, minimum size=0.4cm, inner sep=0pt},
  etwo/.style    ={line width=1.0mm, color=orange!70},
  ethree/.style  ={line width=1.0mm, color=black!60},
  arr/.style     ={->, line width=0.4mm},
  pp/.style      ={red!80!black, font=\footnotesize}
]

\node[font=\small] at (2.0,  5.6) {\Large $t = 1$};
\node[font=\small] at (6.0,  5.6) {\Large $t = 2$};
\node[font=\small] at (11.0, 5.6) {\Large $t = 3$};

\node[anchor=east, font=\small] at (0.7,  4.5) {\large \textbf{Graph 1}};
\node[anchor=east, font=\small] at (0.7,  2.7) {\large \textbf{Graph 2}};
\node[anchor=east, font=\small] at (0.7,  1.0) {\large \textbf{0th Homology}};
\node[anchor=east, font=\small] at (0.7,  0.5) {\large \textbf{(Components)}};
\node[anchor=east, font=\small] at (0.7, -0.2) {\large \textbf{1st Homology}};
\node[anchor=east, font=\small] at (0.2,  -0.7) {\large \textbf{(Cycles)}};


\node[vone] at (2.0, 4.5) {};

\node[vtwo] (P2a) at (4.5, 5.0) {};
\node[vtwo] (P2b) at (4.5, 4.0) {};
\node[vtwo] (P2x) at (5.5, 4.5) {};
\node[vtwo] (P2y) at (6.5, 4.5) {};
\node[vone] (P2z) at (7.5, 4.5) {};
\draw[etwo] (P2a) -- (P2b);
\draw[etwo] (P2x) -- (P2y);
\draw[etwo] (P2y) -- (P2z);

\node[vthree] (P3c) at (10.0, 4.5) {};
\node[vtwo]   (P3a) at ( 9.0, 5.0) {};
\node[vtwo]   (P3b) at ( 9.0, 4.0) {};
\node[vtwo]   (P3x) at (11.0, 4.5) {};
\node[vtwo]   (P3y) at (12.0, 4.5) {};
\node[vone]   (P3z) at (13.0, 4.5) {};
\draw[ethree] (P3c) -- (P3a);
\draw[ethree] (P3c) -- (P3b);
\draw[etwo]   (P3a) -- (P3b);
\draw[ethree] (P3c) -- (P3x);
\draw[etwo]   (P3x) -- (P3y);
\draw[etwo]   (P3y) -- (P3z);

\draw[arr] (2.3, 4.5) -- (4.2, 4.5);
\draw[arr] (7.8, 4.5) -- (8.7, 4.5);


\node[vone] at (2.0, 2.7) {};

\node[vone] (Q2l) at (4.5, 2.7) {};
\node[vtwo] (Q2p) at (5.5, 3.2) {};
\node[vtwo] (Q2q) at (5.5, 2.2) {};
\node[vtwo] (Q2A) at (6.5, 3.2) {};
\node[vtwo] (Q2B) at (6.5, 2.2) {};
\draw[etwo] (Q2p) -- (Q2A);
\draw[etwo] (Q2A) -- (Q2B);
\draw[etwo] (Q2B) -- (Q2q);

\node[vthree] (Q3c) at (10.0, 2.7) {};
\node[vone]   (Q3l) at ( 9.0, 2.7) {};
\node[vtwo]   (Q3p) at (11.0, 3.2) {};
\node[vtwo]   (Q3q) at (11.0, 2.2) {};
\node[vtwo]   (Q3A) at (12.0, 3.2) {};
\node[vtwo]   (Q3B) at (12.0, 2.2) {};
\draw[ethree] (Q3c) -- (Q3l);
\draw[ethree] (Q3c) -- (Q3p);
\draw[ethree] (Q3c) -- (Q3q);
\draw[etwo]   (Q3p) -- (Q3A);
\draw[etwo]   (Q3A) -- (Q3B);
\draw[etwo]   (Q3B) -- (Q3q);

\draw[arr] (2.3, 2.7) -- (4.2, 2.7);
\draw[arr] (6.8, 2.7) -- (8.7, 2.7);


\node[pp, anchor=west] at (8.5,  1.3) {\Large $(1,\infty), (2,3)$ };
\node[pp, anchor=west] at (8.5,  0.7) {\Large$(2,2), (2,2), (2,2), (3.3)$};

\node[pp, anchor=west] at (1.5,  1.3) {\Large $(1,-)$};
\node[pp, anchor=west] at (1.5,  -0.2) {\Large \text{None}};

\node[pp, anchor=west] at (4,  1.3) {\Large $(1,-), (2,-)$ };
\node[pp, anchor=west] at (4,  0.7) {\Large$(2,2), (2,2), (2,2)$};
\node[pp, anchor=west] at (4,  -0.2) {\Large \text{None}};


\node[pp, anchor=west] at (10.5, -0.2) { \Large $(3,\infty)$};

\end{tikzpicture}

%% file: Figures/random_cycle_capacity.tex
\begin{tikzpicture}[x=0.75pt,y=0.75pt,yscale=-1,xscale=1]

\draw  [color={rgb, 255:red, 208; green, 2; blue, 27 }  ,draw opacity=1 ][fill={rgb, 255:red, 208; green, 2; blue, 27 }  ,fill opacity=1 ] (288,83.25) .. controls (288,80.63) and (290.13,78.5) .. (292.75,78.5) .. controls (295.37,78.5) and (297.5,80.63) .. (297.5,83.25) .. controls (297.5,85.87) and (295.37,88) .. (292.75,88) .. controls (290.13,88) and (288,85.87) .. (288,83.25) -- cycle ;
\draw [color={rgb, 255:red, 208; green, 2; blue, 27 }  ,draw opacity=1 ][fill={rgb, 255:red, 208; green, 2; blue, 27 }  ,fill opacity=1 ]   (278,64.5) -- (292.75,85.25) ;
\draw  [color={rgb, 255:red, 208; green, 2; blue, 27 }  ,draw opacity=1 ][fill={rgb, 255:red, 208; green, 2; blue, 27 }  ,fill opacity=1 ] (277.13,59.83) .. controls (279.71,59.35) and (282.19,61.05) .. (282.67,63.63) .. controls (283.15,66.21) and (281.45,68.69) .. (278.87,69.17) .. controls (276.29,69.65) and (273.81,67.95) .. (273.33,65.37) .. controls (272.85,62.79) and (274.55,60.31) .. (277.13,59.83) -- cycle ;
\draw  [color={rgb, 255:red, 208; green, 2; blue, 27 }  ,draw opacity=1 ][fill={rgb, 255:red, 208; green, 2; blue, 27 }  ,fill opacity=1 ] (260,87.25) .. controls (260,84.63) and (262.13,82.5) .. (264.75,82.5) .. controls (267.37,82.5) and (269.5,84.63) .. (269.5,87.25) .. controls (269.5,89.87) and (267.37,92) .. (264.75,92) .. controls (262.13,92) and (260,89.87) .. (260,87.25) -- cycle ;
\draw [color={rgb, 255:red, 208; green, 2; blue, 27 }  ,draw opacity=1 ][fill={rgb, 255:red, 208; green, 2; blue, 27 }  ,fill opacity=1 ]   (250,68.5) -- (264.75,89.25) ;
\draw  [color={rgb, 255:red, 208; green, 2; blue, 27 }  ,draw opacity=1 ][fill={rgb, 255:red, 208; green, 2; blue, 27 }  ,fill opacity=1 ] (249.13,63.83) .. controls (251.71,63.35) and (254.19,65.05) .. (254.67,67.63) .. controls (255.15,70.21) and (253.45,72.69) .. (250.87,73.17) .. controls (248.29,73.65) and (245.81,71.95) .. (245.33,69.37) .. controls (244.85,66.79) and (246.55,64.31) .. (249.13,63.83) -- cycle ;
\draw  [color={rgb, 255:red, 74; green, 144; blue, 226 }  ,draw opacity=1 ][fill={rgb, 255:red, 74; green, 144; blue, 226 }  ,fill opacity=1 ] (78,62.25) .. controls (78,59.63) and (80.13,57.5) .. (82.75,57.5) .. controls (85.37,57.5) and (87.5,59.63) .. (87.5,62.25) .. controls (87.5,64.87) and (85.37,67) .. (82.75,67) .. controls (80.13,67) and (78,64.87) .. (78,62.25) -- cycle ;
\draw [color={rgb, 255:red, 74; green, 144; blue, 226 }  ,draw opacity=1 ][fill={rgb, 255:red, 74; green, 144; blue, 226 }  ,fill opacity=1 ]   (64,71.5) -- (82.75,62.25) ;
\draw  [color={rgb, 255:red, 0; green, 0; blue, 0 }  ,draw opacity=1 ][fill={rgb, 255:red, 0; green, 0; blue, 0 }  ,fill opacity=1 ] (98.13,69.83) .. controls (100.71,69.35) and (103.19,71.05) .. (103.67,73.63) .. controls (104.15,76.21) and (102.45,78.69) .. (99.87,79.17) .. controls (97.29,79.65) and (94.81,77.95) .. (94.33,75.37) .. controls (93.85,72.79) and (95.55,70.31) .. (98.13,69.83) -- cycle ;
\draw [color={rgb, 255:red, 74; green, 144; blue, 226 }  ,draw opacity=1 ][fill={rgb, 255:red, 74; green, 144; blue, 226 }  ,fill opacity=1 ]   (64,71.5) -- (78.75,92.25) ;
\draw  [color={rgb, 255:red, 74; green, 144; blue, 226 }  ,draw opacity=1 ][fill={rgb, 255:red, 74; green, 144; blue, 226 }  ,fill opacity=1 ] (63.13,66.83) .. controls (65.71,66.35) and (68.19,68.05) .. (68.67,70.63) .. controls (69.15,73.21) and (67.45,75.69) .. (64.87,76.17) .. controls (62.29,76.65) and (59.81,74.95) .. (59.33,72.37) .. controls (58.85,69.79) and (60.55,67.31) .. (63.13,66.83) -- cycle ;
\draw  [color={rgb, 255:red, 74; green, 144; blue, 226 }  ,draw opacity=1 ][fill={rgb, 255:red, 74; green, 144; blue, 226 }  ,fill opacity=1 ] (143,64.25) .. controls (143,61.63) and (145.13,59.5) .. (147.75,59.5) .. controls (150.37,59.5) and (152.5,61.63) .. (152.5,64.25) .. controls (152.5,66.87) and (150.37,69) .. (147.75,69) .. controls (145.13,69) and (143,66.87) .. (143,64.25) -- cycle ;
\draw [color={rgb, 255:red, 74; green, 144; blue, 226 }  ,draw opacity=1 ][fill={rgb, 255:red, 74; green, 144; blue, 226 }  ,fill opacity=1 ]   (129,73.5) -- (147.75,64.25) ;
\draw  [color={rgb, 255:red, 74; green, 144; blue, 226 }  ,draw opacity=1 ][fill={rgb, 255:red, 74; green, 144; blue, 226 }  ,fill opacity=1 ] (163.13,71.83) .. controls (165.71,71.35) and (168.19,73.05) .. (168.67,75.63) .. controls (169.15,78.21) and (167.45,80.69) .. (164.87,81.17) .. controls (162.29,81.65) and (159.81,79.95) .. (159.33,77.37) .. controls (158.85,74.79) and (160.55,72.31) .. (163.13,71.83) -- cycle ;
\draw  [color={rgb, 255:red, 74; green, 144; blue, 226 }  ,draw opacity=1 ][fill={rgb, 255:red, 74; green, 144; blue, 226 }  ,fill opacity=1 ] (139,92.25) .. controls (139,89.63) and (141.13,87.5) .. (143.75,87.5) .. controls (146.37,87.5) and (148.5,89.63) .. (148.5,92.25) .. controls (148.5,94.87) and (146.37,97) .. (143.75,97) .. controls (141.13,97) and (139,94.87) .. (139,92.25) -- cycle ;
\draw [color={rgb, 255:red, 74; green, 144; blue, 226 }  ,draw opacity=1 ][fill={rgb, 255:red, 74; green, 144; blue, 226 }  ,fill opacity=1 ]   (129,73.5) -- (143.75,94.25) ;
\draw  [color={rgb, 255:red, 74; green, 144; blue, 226 }  ,draw opacity=1 ][fill={rgb, 255:red, 74; green, 144; blue, 226 }  ,fill opacity=1 ] (128.13,68.83) .. controls (130.71,68.35) and (133.19,70.05) .. (133.67,72.63) .. controls (134.15,75.21) and (132.45,77.69) .. (129.87,78.17) .. controls (127.29,78.65) and (124.81,76.95) .. (124.33,74.37) .. controls (123.85,71.79) and (125.55,69.31) .. (128.13,68.83) -- cycle ;
\draw [color={rgb, 255:red, 74; green, 144; blue, 226 }  ,draw opacity=1 ][fill={rgb, 255:red, 74; green, 144; blue, 226 }  ,fill opacity=1 ]   (143.75,92.25) -- (147.75,64.25) ;
\draw  [dash pattern={on 4.5pt off 4.5pt}]  (181,45.5) -- (181,124.5) ;
\draw  [color={rgb, 255:red, 0; green, 0; blue, 0 }  ,draw opacity=1 ][fill={rgb, 255:red, 0; green, 0; blue, 0 }  ,fill opacity=1 ] (43.13,70.83) .. controls (45.71,70.35) and (48.19,72.05) .. (48.67,74.63) .. controls (49.15,77.21) and (47.45,79.69) .. (44.87,80.17) .. controls (42.29,80.65) and (39.81,78.95) .. (39.33,76.37) .. controls (38.85,73.79) and (40.55,71.31) .. (43.13,70.83) -- cycle ;
\draw  [color={rgb, 255:red, 0; green, 0; blue, 0 }  ,draw opacity=1 ][fill={rgb, 255:red, 0; green, 0; blue, 0 }  ,fill opacity=1 ] (19,91.25) .. controls (19,88.63) and (21.13,86.5) .. (23.75,86.5) .. controls (26.37,86.5) and (28.5,88.63) .. (28.5,91.25) .. controls (28.5,93.87) and (26.37,96) .. (23.75,96) .. controls (21.13,96) and (19,93.87) .. (19,91.25) -- cycle ;
\draw  [color={rgb, 255:red, 0; green, 0; blue, 0 }  ,draw opacity=1 ][fill={rgb, 255:red, 0; green, 0; blue, 0 }  ,fill opacity=1 ] (74,92.25) .. controls (74,89.63) and (76.13,87.5) .. (78.75,87.5) .. controls (81.37,87.5) and (83.5,89.63) .. (83.5,92.25) .. controls (83.5,94.87) and (81.37,97) .. (78.75,97) .. controls (76.13,97) and (74,94.87) .. (74,92.25) -- cycle ;
\draw  [color={rgb, 255:red, 0; green, 0; blue, 0 }  ,draw opacity=1 ][fill={rgb, 255:red, 0; green, 0; blue, 0 }  ,fill opacity=1 ] (219,79.25) .. controls (219,76.63) and (221.13,74.5) .. (223.75,74.5) .. controls (226.37,74.5) and (228.5,76.63) .. (228.5,79.25) .. controls (228.5,81.87) and (226.37,84) .. (223.75,84) .. controls (221.13,84) and (219,81.87) .. (219,79.25) -- cycle ;
\draw  [color={rgb, 255:red, 0; green, 0; blue, 0 }  ,draw opacity=1 ][fill={rgb, 255:red, 0; green, 0; blue, 0 }  ,fill opacity=1 ] (202,88.25) .. controls (202,85.63) and (204.13,83.5) .. (206.75,83.5) .. controls (209.37,83.5) and (211.5,85.63) .. (211.5,88.25) .. controls (211.5,90.87) and (209.37,93) .. (206.75,93) .. controls (204.13,93) and (202,90.87) .. (202,88.25) -- cycle ;

\draw (267,107.4) node [anchor=north west][inner sep=0.75pt]    {$\textcolor[rgb]{0.82,0.01,0.11}{H_{1}}$};
\draw (74,109.4) node [anchor=north west][inner sep=0.75pt]    {$\textcolor[rgb]{0.29,0.56,0.89}{G_{1}}$};
\draw (136,109.4) node [anchor=north west][inner sep=0.75pt]    {$\textcolor[rgb]{0.29,0.56,0.89}{G_{2}}$};
\draw (19,110.4) node [anchor=north west][inner sep=0.75pt]    {$\textcolor[rgb]{0.29,0.56,0.89}{G}\textcolor[rgb]{0.29,0.56,0.89}{_{0}}$};
\draw (209,108.4) node [anchor=north west][inner sep=0.75pt]    {$\textcolor[rgb]{0.82,0.01,0.11}{H}\textcolor[rgb]{0.82,0.01,0.11}{_{0}}$};

\end{tikzpicture}

%% file: Figures/vertex_based_capacity.tex
\begin{tikzpicture}[x=0.75pt,y=0.75pt,yscale=-1,xscale=1]

\draw  [dash pattern={on 4.5pt off 4.5pt}]  (204,65.5) -- (204,144.5) ;
\draw  [color={rgb, 255:red, 0; green, 0; blue, 0 }  ,draw opacity=1 ][fill={rgb, 255:red, 0; green, 0; blue, 0 }  ,fill opacity=1 ] (62.13,90.83) .. controls (64.71,90.35) and (67.19,92.05) .. (67.67,94.63) .. controls (68.15,97.21) and (66.45,99.69) .. (63.87,100.17) .. controls (61.29,100.65) and (58.81,98.95) .. (58.33,96.37) .. controls (57.85,93.79) and (59.55,91.31) .. (62.13,90.83) -- cycle ;
\draw  [color={rgb, 255:red, 0; green, 0; blue, 0 }  ,draw opacity=1 ][fill={rgb, 255:red, 0; green, 0; blue, 0 }  ,fill opacity=1 ] (42,111.25) .. controls (42,108.63) and (44.13,106.5) .. (46.75,106.5) .. controls (49.37,106.5) and (51.5,108.63) .. (51.5,111.25) .. controls (51.5,113.87) and (49.37,116) .. (46.75,116) .. controls (44.13,116) and (42,113.87) .. (42,111.25) -- cycle ;
\draw  [color={rgb, 255:red, 0; green, 0; blue, 0 }  ,draw opacity=1 ][fill={rgb, 255:red, 0; green, 0; blue, 0 }  ,fill opacity=1 ] (52.13,72.83) .. controls (54.71,72.35) and (57.19,74.05) .. (57.67,76.63) .. controls (58.15,79.21) and (56.45,81.69) .. (53.87,82.17) .. controls (51.29,82.65) and (48.81,80.95) .. (48.33,78.37) .. controls (47.85,75.79) and (49.55,73.31) .. (52.13,72.83) -- cycle ;
\draw  [color={rgb, 255:red, 0; green, 0; blue, 0 }  ,draw opacity=1 ][fill={rgb, 255:red, 0; green, 0; blue, 0 }  ,fill opacity=1 ] (28,93.25) .. controls (28,90.63) and (30.13,88.5) .. (32.75,88.5) .. controls (35.37,88.5) and (37.5,90.63) .. (37.5,93.25) .. controls (37.5,95.87) and (35.37,98) .. (32.75,98) .. controls (30.13,98) and (28,95.87) .. (28,93.25) -- cycle ;
\draw [color={rgb, 255:red, 208; green, 2; blue, 27 }  ,draw opacity=1 ][fill={rgb, 255:red, 208; green, 2; blue, 27 }  ,fill opacity=1 ]   (113,93.5) -- (131.75,84.25) ;
\draw  [color={rgb, 255:red, 0; green, 0; blue, 0 }  ,draw opacity=1 ][fill={rgb, 255:red, 0; green, 0; blue, 0 }  ,fill opacity=1 ] (126.21,80.45) .. controls (128.79,79.97) and (131.27,81.67) .. (131.75,84.25) .. controls (132.23,86.83) and (130.53,89.31) .. (127.95,89.79) .. controls (125.37,90.27) and (122.89,88.57) .. (122.41,85.99) .. controls (121.93,83.41) and (123.63,80.93) .. (126.21,80.45) -- cycle ;
\draw [color={rgb, 255:red, 208; green, 2; blue, 27 }  ,draw opacity=1 ][fill={rgb, 255:red, 208; green, 2; blue, 27 }  ,fill opacity=1 ]   (113,93.5) -- (101.75,112.25) ;
\draw  [color={rgb, 255:red, 208; green, 2; blue, 27 }  ,draw opacity=1 ][fill={rgb, 255:red, 208; green, 2; blue, 27 }  ,fill opacity=1 ] (112.13,88.83) .. controls (114.71,88.35) and (117.19,90.05) .. (117.67,92.63) .. controls (118.15,95.21) and (116.45,97.69) .. (113.87,98.17) .. controls (111.29,98.65) and (108.81,96.95) .. (108.33,94.37) .. controls (107.85,91.79) and (109.55,89.31) .. (112.13,88.83) -- cycle ;
\draw [color={rgb, 255:red, 74; green, 144; blue, 226 }  ,draw opacity=1 ][fill={rgb, 255:red, 74; green, 144; blue, 226 }  ,fill opacity=1 ]   (282,87.5) -- (300.75,78.25) ;
\draw  [color={rgb, 255:red, 0; green, 0; blue, 0 }  ,draw opacity=1 ][fill={rgb, 255:red, 0; green, 0; blue, 0 }  ,fill opacity=1 ] (295.21,74.45) .. controls (297.79,73.97) and (300.27,75.67) .. (300.75,78.25) .. controls (301.23,80.83) and (299.53,83.31) .. (296.95,83.79) .. controls (294.37,84.27) and (291.89,82.57) .. (291.41,79.99) .. controls (290.93,77.41) and (292.63,74.93) .. (295.21,74.45) -- cycle ;
\draw [color={rgb, 255:red, 74; green, 144; blue, 226 }  ,draw opacity=1 ][fill={rgb, 255:red, 74; green, 144; blue, 226 }  ,fill opacity=1 ]   (282,87.5) -- (296.75,108.25) ;
\draw  [color={rgb, 255:red, 74; green, 144; blue, 226 }  ,draw opacity=1 ][fill={rgb, 255:red, 74; green, 144; blue, 226 }  ,fill opacity=1 ] (281.13,82.83) .. controls (283.71,82.35) and (286.19,84.05) .. (286.67,86.63) .. controls (287.15,89.21) and (285.45,91.69) .. (282.87,92.17) .. controls (280.29,92.65) and (277.81,90.95) .. (277.33,88.37) .. controls (276.85,85.79) and (278.55,83.31) .. (281.13,82.83) -- cycle ;
\draw  [color={rgb, 255:red, 0; green, 0; blue, 0 }  ,draw opacity=1 ][fill={rgb, 255:red, 0; green, 0; blue, 0 }  ,fill opacity=1 ] (292,108.25) .. controls (292,105.63) and (294.13,103.5) .. (296.75,103.5) .. controls (299.37,103.5) and (301.5,105.63) .. (301.5,108.25) .. controls (301.5,110.87) and (299.37,113) .. (296.75,113) .. controls (294.13,113) and (292,110.87) .. (292,108.25) -- cycle ;
\draw [color={rgb, 255:red, 74; green, 144; blue, 226 }  ,draw opacity=1 ][fill={rgb, 255:red, 74; green, 144; blue, 226 }  ,fill opacity=1 ]   (308,89.5) -- (326.75,80.25) ;
\draw  [color={rgb, 255:red, 0; green, 0; blue, 0 }  ,draw opacity=1 ][fill={rgb, 255:red, 0; green, 0; blue, 0 }  ,fill opacity=1 ] (321.21,76.45) .. controls (323.79,75.97) and (326.27,77.67) .. (326.75,80.25) .. controls (327.23,82.83) and (325.53,85.31) .. (322.95,85.79) .. controls (320.37,86.27) and (317.89,84.57) .. (317.41,81.99) .. controls (316.93,79.41) and (318.63,76.93) .. (321.21,76.45) -- cycle ;
\draw [color={rgb, 255:red, 74; green, 144; blue, 226 }  ,draw opacity=1 ][fill={rgb, 255:red, 74; green, 144; blue, 226 }  ,fill opacity=1 ]   (308,89.5) -- (322.75,110.25) ;
\draw  [color={rgb, 255:red, 74; green, 144; blue, 226 }  ,draw opacity=1 ][fill={rgb, 255:red, 74; green, 144; blue, 226 }  ,fill opacity=1 ] (307.13,84.83) .. controls (309.71,84.35) and (312.19,86.05) .. (312.67,88.63) .. controls (313.15,91.21) and (311.45,93.69) .. (308.87,94.17) .. controls (306.29,94.65) and (303.81,92.95) .. (303.33,90.37) .. controls (302.85,87.79) and (304.55,85.31) .. (307.13,84.83) -- cycle ;
\draw  [color={rgb, 255:red, 0; green, 0; blue, 0 }  ,draw opacity=1 ][fill={rgb, 255:red, 0; green, 0; blue, 0 }  ,fill opacity=1 ] (318,110.25) .. controls (318,107.63) and (320.13,105.5) .. (322.75,105.5) .. controls (325.37,105.5) and (327.5,107.63) .. (327.5,110.25) .. controls (327.5,112.87) and (325.37,115) .. (322.75,115) .. controls (320.13,115) and (318,112.87) .. (318,110.25) -- cycle ;
\draw  [color={rgb, 255:red, 0; green, 0; blue, 0 }  ,draw opacity=1 ][fill={rgb, 255:red, 0; green, 0; blue, 0 }  ,fill opacity=1 ] (257.13,89.83) .. controls (259.71,89.35) and (262.19,91.05) .. (262.67,93.63) .. controls (263.15,96.21) and (261.45,98.69) .. (258.87,99.17) .. controls (256.29,99.65) and (253.81,97.95) .. (253.33,95.37) .. controls (252.85,92.79) and (254.55,90.31) .. (257.13,89.83) -- cycle ;
\draw  [color={rgb, 255:red, 0; green, 0; blue, 0 }  ,draw opacity=1 ][fill={rgb, 255:red, 0; green, 0; blue, 0 }  ,fill opacity=1 ] (237,110.25) .. controls (237,107.63) and (239.13,105.5) .. (241.75,105.5) .. controls (244.37,105.5) and (246.5,107.63) .. (246.5,110.25) .. controls (246.5,112.87) and (244.37,115) .. (241.75,115) .. controls (239.13,115) and (237,112.87) .. (237,110.25) -- cycle ;
\draw  [color={rgb, 255:red, 0; green, 0; blue, 0 }  ,draw opacity=1 ][fill={rgb, 255:red, 0; green, 0; blue, 0 }  ,fill opacity=1 ] (247.13,71.83) .. controls (249.71,71.35) and (252.19,73.05) .. (252.67,75.63) .. controls (253.15,78.21) and (251.45,80.69) .. (248.87,81.17) .. controls (246.29,81.65) and (243.81,79.95) .. (243.33,77.37) .. controls (242.85,74.79) and (244.55,72.31) .. (247.13,71.83) -- cycle ;
\draw  [color={rgb, 255:red, 0; green, 0; blue, 0 }  ,draw opacity=1 ][fill={rgb, 255:red, 0; green, 0; blue, 0 }  ,fill opacity=1 ] (223,92.25) .. controls (223,89.63) and (225.13,87.5) .. (227.75,87.5) .. controls (230.37,87.5) and (232.5,89.63) .. (232.5,92.25) .. controls (232.5,94.87) and (230.37,97) .. (227.75,97) .. controls (225.13,97) and (223,94.87) .. (223,92.25) -- cycle ;
\draw  [color={rgb, 255:red, 0; green, 0; blue, 0 }  ,draw opacity=1 ][fill={rgb, 255:red, 0; green, 0; blue, 0 }  ,fill opacity=1 ] (97,112.25) .. controls (97,109.63) and (99.13,107.5) .. (101.75,107.5) .. controls (104.37,107.5) and (106.5,109.63) .. (106.5,112.25) .. controls (106.5,114.87) and (104.37,117) .. (101.75,117) .. controls (99.13,117) and (97,114.87) .. (97,112.25) -- cycle ;
\draw [color={rgb, 255:red, 208; green, 2; blue, 27 }  ,draw opacity=1 ][fill={rgb, 255:red, 208; green, 2; blue, 27 }  ,fill opacity=1 ]   (100.08,80.12) -- (113,93.5) ;
\draw  [color={rgb, 255:red, 0; green, 0; blue, 0 }  ,draw opacity=1 ][fill={rgb, 255:red, 0; green, 0; blue, 0 }  ,fill opacity=1 ] (99.21,75.45) .. controls (101.79,74.97) and (104.27,76.67) .. (104.75,79.25) .. controls (105.23,81.83) and (103.53,84.31) .. (100.95,84.79) .. controls (98.37,85.27) and (95.89,83.57) .. (95.41,80.99) .. controls (94.93,78.41) and (96.63,75.93) .. (99.21,75.45) -- cycle ;
\draw [color={rgb, 255:red, 208; green, 2; blue, 27 }  ,draw opacity=1 ][fill={rgb, 255:red, 208; green, 2; blue, 27 }  ,fill opacity=1 ]   (168,96.5) -- (186.75,87.25) ;
\draw  [color={rgb, 255:red, 0; green, 0; blue, 0 }  ,draw opacity=1 ][fill={rgb, 255:red, 0; green, 0; blue, 0 }  ,fill opacity=1 ] (181.21,83.45) .. controls (183.79,82.97) and (186.27,84.67) .. (186.75,87.25) .. controls (187.23,89.83) and (185.53,92.31) .. (182.95,92.79) .. controls (180.37,93.27) and (177.89,91.57) .. (177.41,88.99) .. controls (176.93,86.41) and (178.63,83.93) .. (181.21,83.45) -- cycle ;
\draw [color={rgb, 255:red, 208; green, 2; blue, 27 }  ,draw opacity=1 ][fill={rgb, 255:red, 208; green, 2; blue, 27 }  ,fill opacity=1 ]   (168,96.5) -- (156.75,115.25) ;
\draw  [color={rgb, 255:red, 0; green, 0; blue, 0 }  ,draw opacity=1 ][fill={rgb, 255:red, 0; green, 0; blue, 0 }  ,fill opacity=1 ] (152,115.25) .. controls (152,112.63) and (154.13,110.5) .. (156.75,110.5) .. controls (159.37,110.5) and (161.5,112.63) .. (161.5,115.25) .. controls (161.5,117.87) and (159.37,120) .. (156.75,120) .. controls (154.13,120) and (152,117.87) .. (152,115.25) -- cycle ;
\draw [color={rgb, 255:red, 208; green, 2; blue, 27 }  ,draw opacity=1 ][fill={rgb, 255:red, 208; green, 2; blue, 27 }  ,fill opacity=1 ]   (155.08,83.12) -- (168,96.5) ;
\draw  [color={rgb, 255:red, 0; green, 0; blue, 0 }  ,draw opacity=1 ][fill={rgb, 255:red, 0; green, 0; blue, 0 }  ,fill opacity=1 ] (154.21,78.45) .. controls (156.79,77.97) and (159.27,79.67) .. (159.75,82.25) .. controls (160.23,84.83) and (158.53,87.31) .. (155.95,87.79) .. controls (153.37,88.27) and (150.89,86.57) .. (150.41,83.99) .. controls (149.93,81.41) and (151.63,78.93) .. (154.21,78.45) -- cycle ;
\draw  [color={rgb, 255:red, 0; green, 0; blue, 0 }  ,draw opacity=1 ][fill={rgb, 255:red, 0; green, 0; blue, 0 }  ,fill opacity=1 ] (161,43.25) .. controls (161,40.63) and (163.13,38.5) .. (165.75,38.5) .. controls (168.37,38.5) and (170.5,40.63) .. (170.5,43.25) .. controls (170.5,45.87) and (168.37,48) .. (165.75,48) .. controls (163.13,48) and (161,45.87) .. (161,43.25) -- cycle ;
\draw    (165.75,43.25) -- (182.08,88.12) ;
\draw    (165.75,43.25) -- (155.08,83.12) ;
\draw    (165.75,43.25) -- (168,96.5) ;
\draw    (165.75,43.25) -- (156.75,115.25) ;
\draw  [color={rgb, 255:red, 208; green, 2; blue, 27 }  ,draw opacity=1 ][fill={rgb, 255:red, 208; green, 2; blue, 27 }  ,fill opacity=1 ] (167.13,91.83) .. controls (169.71,91.35) and (172.19,93.05) .. (172.67,95.63) .. controls (173.15,98.21) and (171.45,100.69) .. (168.87,101.17) .. controls (166.29,101.65) and (163.81,99.95) .. (163.33,97.37) .. controls (162.85,94.79) and (164.55,92.31) .. (167.13,91.83) -- cycle ;
\draw  [color={rgb, 255:red, 208; green, 2; blue, 27 }  ,draw opacity=1 ][fill={rgb, 255:red, 208; green, 2; blue, 27 }  ,fill opacity=1 ] (134.13,113.83) .. controls (136.71,113.35) and (139.19,115.05) .. (139.67,117.63) .. controls (140.15,120.21) and (138.45,122.69) .. (135.87,123.17) .. controls (133.29,123.65) and (130.81,121.95) .. (130.33,119.37) .. controls (129.85,116.79) and (131.55,114.31) .. (134.13,113.83) -- cycle ;
\draw [color={rgb, 255:red, 208; green, 2; blue, 27 }  ,draw opacity=1 ][fill={rgb, 255:red, 208; green, 2; blue, 27 }  ,fill opacity=1 ]   (122.08,105.12) -- (135,118.5) ;
\draw  [color={rgb, 255:red, 0; green, 0; blue, 0 }  ,draw opacity=1 ][fill={rgb, 255:red, 0; green, 0; blue, 0 }  ,fill opacity=1 ] (119,105.25) .. controls (119,102.63) and (121.13,100.5) .. (123.75,100.5) .. controls (126.37,100.5) and (128.5,102.63) .. (128.5,105.25) .. controls (128.5,107.87) and (126.37,110) .. (123.75,110) .. controls (121.13,110) and (119,107.87) .. (119,105.25) -- cycle ;
\draw  [color={rgb, 255:red, 208; green, 2; blue, 27 }  ,draw opacity=1 ][fill={rgb, 255:red, 208; green, 2; blue, 27 }  ,fill opacity=1 ] (193.13,118.83) .. controls (195.71,118.35) and (198.19,120.05) .. (198.67,122.63) .. controls (199.15,125.21) and (197.45,127.69) .. (194.87,128.17) .. controls (192.29,128.65) and (189.81,126.95) .. (189.33,124.37) .. controls (188.85,121.79) and (190.55,119.31) .. (193.13,118.83) -- cycle ;
\draw [color={rgb, 255:red, 208; green, 2; blue, 27 }  ,draw opacity=1 ][fill={rgb, 255:red, 208; green, 2; blue, 27 }  ,fill opacity=1 ]   (180.08,108.12) -- (193,121.5) ;
\draw  [color={rgb, 255:red, 0; green, 0; blue, 0 }  ,draw opacity=1 ][fill={rgb, 255:red, 0; green, 0; blue, 0 }  ,fill opacity=1 ] (177,108.25) .. controls (177,105.63) and (179.13,103.5) .. (181.75,103.5) .. controls (184.37,103.5) and (186.5,105.63) .. (186.5,108.25) .. controls (186.5,110.87) and (184.37,113) .. (181.75,113) .. controls (179.13,113) and (177,110.87) .. (177,108.25) -- cycle ;

\draw (290,127.4) node [anchor=north west][inner sep=0.75pt]    {$\textcolor[rgb]{0.29,0.56,0.89}{H_{1}}$};
\draw (97,129.4) node [anchor=north west][inner sep=0.75pt]    {$\textcolor[rgb]{0.82,0.01,0.11}{G_{1}}$};
\draw (159,129.4) node [anchor=north west][inner sep=0.75pt]    {$\textcolor[rgb]{0.82,0.01,0.11}{G_{2}}$};
\draw (42,130.4) node [anchor=north west][inner sep=0.75pt]    {$\textcolor[rgb]{0.82,0.01,0.11}{G_{0}}$};
\draw (232,128.4) node [anchor=north west][inner sep=0.75pt]  [color={rgb, 255:red, 74; green, 144; blue, 226 }  ,opacity=1 ]  {$\textcolor[rgb]{0.29,0.56,0.89}{H_{0}}$};

\end{tikzpicture}